\documentclass[prodmode,acmtecs]{acmsmall} % Aptara syntax

\usepackage[ruled]{algorithm2e}

\SetAlFnt{\small}
\SetAlCapFnt{\small}
\SetAlCapNameFnt{\small}
\SetAlCapHSkip{0pt}
\IncMargin{-\parindent}

\acmVolume{0}
\acmNumber{0}
\acmArticle{00}
\acmYear{2000}
\acmMonth{3}

\doi{0000001.0000001}

\issn{1234-56789}

\usepackage{amssymb}
\usepackage{amsmath}
\newtheorem{thm}{Theorem}[section]
\newtheorem{dfn}{Definition}[section]
\newtheorem{ex}{Example}[section]
\newtheorem{prop}{Proposition}[section]
\newtheorem{cor}{Corollary}[section]
\newtheorem{lem}{Lemma}[section]
\usepackage{enumerate}
\usepackage{comment}

\makeatletter
\newcommand{\dotminus}{\mathbin{\text{\@dotminus}}}

\newcommand{\@dotminus}{%
  \ooalign{\hidewidth\raise1ex\hbox{.}\hidewidth\cr$\m@th-$\cr}%
}
\makeatother

\newcommand*{\qedex}{\hfill\ensuremath{\blacksquare}}

\newcommand*{\es}{\emptyset}
\newcommand*{\se}{\mathcal{SE}}
\newcommand*{\seqc}{\overline{SE(Q)}}
\newcommand*{\seqrc}{\overline{SE(Q + R)}}
\newcommand*{\lpa}{\mathcal{LP}_\mathcal{A}}

\newcommand*{\revpm}{*_\gamma}
\newcommand*{\conpm}{\dotminus_\gamma}
\newcommand*{\pq}{\mathbb{P}_Q}
\newcommand*{\gpq}{\gamma(\mathbb{P}_Q)}
\newcommand*{\pqr}{\mathbb{P}_{Q+R}}

\newcommand*{\pqm}{\mathbb{P}_Q^-}
\newcommand*{\gpqm}{\gamma(\mathbb{P}_Q^-)}
\newcommand*{\prm}{\mathbb{P}_R^-}
\newcommand*{\pqrm}{\mathbb{P}_{Q+R}^-}
\newcommand*{\cg}{\circ_\gamma}
\newcommand*{\oa}{|_a}

\newcommand*{\mpq}{\mathcal{M}(P)|_Q}

\newcommand*{\reve}{*_\preceq}
\newcommand*{\reveR}{*_{\preceq^R}}
\newcommand*{\cone}{\dotminus_{\preceq}}
\newcommand*{\coneR}{\dotminus_{\preceq^R}}

\newcommand*{\cq}{cut_\preceq(Q)}
\newcommand*{\cqm}{cut_{\preceq}^-(Q)}
\newcommand*{\cqr}{cut_\preceq(Q+R)}
\newcommand*{\cqrm}{cut_{\preceq}^-(Q+R)}

\newcommand*{\scr}{!_{\gamma_P}}
\newcommand*{\botse}{\bot_!^{SE}}
\newcommand*{\scrse}{!_{\gamma_P}^{SE}}
\newcommand*{\pqas}{\mathbb{P}_Q^{AS}}
\newcommand*{\revpmas}{*_{\gamma^1}^{AS}}
\begin{document}

% Page heads
\markboth{S. Binnewies et al.}{Syntax-Preserving Belief Change Operators for Logic Programs}

% Title portion
\title{Syntax-Preserving Belief Change Operators for Logic Programs}
\author{SEBASTIAN BINNEWIES
\affil{Griffith University}
ZHIQIANG ZHUANG
\affil{Griffith University}
KEWEN WANG
\affil{Griffith University}
BELA STANTIC
\affil{Griffith University}
}
% NOTE! Affiliations placed here should be for the institution where the
%       BULK of the research was done. If the author has gone to a new
%       institution, before publication, the (above) affiliation should NOT be changed.
%       The authors 'current' address may be given in the "Author's addresses:" block (below).
%       So for example, Mr. Abdelzaher, the bulk of the research was done at UIUC, and he is
%       currently affiliated with NASA.

\begin{abstract}
Recent methods have adapted the well-established AGM and belief base frameworks for belief change to cover belief revision in logic programs. In this study here, we present two new sets of belief change operators for logic programs. They focus on preserving the explicit relationships expressed in the rules of a program, a feature that is missing in purely semantic approaches that consider programs only in their entirety. In particular, operators of the latter class fail to satisfy preservation and support, two important properties for belief change in logic programs required to ensure intuitive results. 

We address this shortcoming of existing approaches by introducing partial meet and ensconcement constructions for logic program belief change, which allow us to define syntax-preserving operators that satisfy preservation and support. Our work is novel in that our constructions not only preserve more information from a logic program during a change operation than existing ones, but they also facilitate natural definitions of contraction operators, the first in the field to the best of our knowledge.

In order to evaluate the rationality of our operators, we translate the revision and contraction postulates from the 
AGM and belief base frameworks to the logic programming setting. We show that our operators fully comply with the belief base framework and formally state the interdefinability between our operators. We further propose an algorithm that is based on modularising a logic program to reduce partial meet and ensconcement revisions or contractions to performing the operation only on the relevant modules of that program. Finally, we compare our approach to two state-of-the-art logic program revision methods and demonstrate that our operators address the shortcomings of one and generalise the other method.

\end{abstract}

\begin{CCSXML}
<ccs2012>
<concept>
<concept_id>10010147.10010178.10010187.10010196</concept_id>
<concept_desc>Computing methodologies~Logic programming and answer set programming</concept_desc>
<concept_significance>500</concept_significance>
</concept>
<concept>
<concept_id>10010147.10010178.10010187.10010189</concept_id>
<concept_desc>Computing methodologies~Nonmonotonic, default reasoning and belief revision</concept_desc>
<concept_significance>300</concept_significance>
</concept>
<concept>
<concept_id>10003752.10003790.10003795</concept_id>
<concept_desc>Theory of computation~Constraint and logic programming</concept_desc>
<concept_significance>300</concept_significance>
</concept>
</ccs2012>
\end{CCSXML}

\ccsdesc[500]{Computing methodologies~Logic programming and answer set programming}
\ccsdesc[300]{Computing methodologies~Nonmonotonic, default reasoning and belief revision}
\ccsdesc[300]{Theory of computation~Constraint and logic programming}

% We no longer use \terms command
%\terms{Design, Algorithms, Performance}

\keywords{Logic Program, Belief Change, Strong Equivalence, Answer Set}

\acmformat{Sebastian Binnewies, Zhiqiang Zhuang, Kewen Wang, and Bela Stantic 2017. Syntax-Preserving Belief Change Operators for Logic Programs.}
% At a minimum you need to supply the author names, year and a title.
% IMPORTANT:
% Full first names whenever they are known, surname last, followed by a period.
% In the case of two authors, 'and' is placed between them.
% In the case of three or more authors, the serial comma is used, that is, all author names
% except the last one but including the penultimate author's name are followed by a comma,
% and then 'and' is placed before the final author's name.
% If only first and middle initials are known, then each initial
% is followed by a period and they are separated by a space.
% The remaining information (journal title, volume, article number, date, etc.) is 'auto-generated'.

\begin{bottomstuff}
%This work is supported by XXX funding XXX.

Author's addresses: Sebastian Binnewies {and} Zhiqiang Zhuang {and} Kewen Wang {and} Bela Stantic, School of Information and Communication Technology, Griffith University, Australia; Email: s.binnewies@griffith.edu.au, z.zhuang@griffith.edu.au, k.wang@griffith.edu.au, b.stantic@griffith.edu.au.
\end{bottomstuff}

\maketitle

\section{Introduction} \label{sec:intro}
A key ingredient for any machine to be considered \lq artificially intelligent\rq\ is a system to represent and reason about knowledge in the application domain of interest \cite{mccarthy1958programs}. In analogy to a human brain, such a system should be capable of storing information in some knowledge base and reasoning over existing information to deduce new information. Moreover, information in a knowledge base should be amenable to change, whether it be adding, deleting, or modifying information. The study of belief change \cite{doyle1979truth,fagin1983semantics,gardenfors1988knowledge,hansson1999textbook,harper1976rational,levi1980enterprise} concerns itself exactly with these kinds of dynamics in knowledge bases. It aims at providing mechanisms to change a knowledge base whenever new information is acquired. The majority of these mechanisms rely on two fundamental principles: the principle of primacy of new information, stating that new information should be treated with priority over existing information in the knowledge base, and the principle of minimal change, stating that as much existing information as possible should be preserved during a change operation \cite{dalal1988investigations}. 

An important endeavour to guide change operations on a knowledge base and by now the most widely-adopted belief change paradigm is the so-called \emph{AGM framework}, named after the initials of the author trio \cite{alchourron1985logic}. It classifies the possible changes to a knowledge base as \emph{expansion}, \emph{revision}, and \emph{contraction} operations. In an expansion, new information is incorporated into a knowledge base, regardless of any inconsistencies that may arise. A revision operation also incorporates new information into a knowledge base, but in such a way that the resulting knowledge base is consistent. This is achieved by discarding some existing information. During a contraction, no new information is added to a knowledge base but some existing information is removed from it. On the one hand, the framework provides a set of postulates that each rational change operator should satisfy, and, on the other hand, defines specific constructions of expansion, revision, and contraction that satisfy these criteria. While the underlying assumption of the AGM framework is that any information implied by a knowledge base is represented explicitly in the knowledge base, the \emph{belief base framework} of belief change \cite{fuhrmann1991theory,hansson1989new,rott1992modellings} does not require this assumption. Postulates and constructions for expansion, revision, and contraction operators in the belief base framework have been defined to complement those from the AGM model (\citeN{hansson1999textbook} provides a summary).

While the AGM and belief base frameworks have been applied to a variety of knowledge representation formalisms (an overview is given by \citeN{wassermann2011agm}), work on an adaptation to knowledge representation in the form of \emph{logic programs} \cite{colmerauer1996birth,kowalski1974predicate,lloyd1987foundations} has been slow to progress. 
%Logic programming \cite{colmerauer1996birth,kowalski1974predicate,lloyd1987foundations}, is a programming approach geared towards declarative, rather than procedural, problem-solving with an expressive, yet natural representation format. It forms the foundation of answer set programming \cite{gelfond1988stable,baral2003knowledge}, which has been used for practical applications such as work-force management \cite{brewka2011answer}, molecular biology \cite{gebser2010repair}, e-Tourism \cite{grasso2013answer}, or linguistics \cite{brooks2005character}.
A major challenge in the adaptation of the AGM and belief base frameworks to logic programming lies in the semantics of logic programs. While the frameworks and their previous adaptations are based on monotonic semantics, the standard \emph{answer set semantics} \cite{gelfond1988stable} of logic programs is nonmonotonic. Only recently have operators been proposed for belief revision in logic programs. Program-level revision \cite{delgrande2010program} and screened semi-revision \cite{krumpelmann2012belief} are initial approaches to logic program revision, yet have strict limitations in their expressiveness due to the nonmonotonicity of the underlying answer set semantics. A breakthrough arrived with the distance-based approach \cite{delgrande2013model} to logic program revision, which rests upon characterising an agent's beliefs in terms of the set of \emph{SE (strong equivalence) models} \cite{lifschitz2001strongly,turner2003strong} of a logic program. A logic program $P$ has the same set of SE~models as a program $Q$ if and only if, for any program $R$, the answer sets of $P$ combined with $R$ are exactly the same as the answer sets of $Q$ combined with $R$. SE~model semantics provides an alternative, monotonic characterisation for logic programs and thus circumvents obstacles presented by nonmonotonicity. To revise a program $P$ by a program $Q$, the distance-based revision operator determines those SE~models from the set of SE~models of $Q$ that are closest to the SE~models of~$P$. %Two notions of distance between models are borrowed from classic belief change to determine closeness \cite{dalal1988investigations,satoh1988nonmonotonic}. It is shown that the distance-based approach is well-behaved with respect to an adaptation of the AGM revision postulates to logic programs.

Even though the distance-based approach is a major milestone for logic program revision, it has some critical shortcomings. Firstly, as it relies on the set of SE~models of an entire program as the representation of beliefs expressed by the program, it operates on the \emph{program-level} only. This means that a program may freely be substituted with any other that has the same set of SE~models and the revision output will remain the same. However, the information expressed by a program is more than just its set of SE~models -- a program also encodes relationships between the atoms occurring in it \cite{leite1998generalizing}. Such relationships are expressed on the \emph{rule-level}, by the individual rules contained in a program. By neglecting information expressed on the rule-level, the distance-based approach fails to satisfy the property of \emph{preservation} \cite{inoue2004equivalence} and the property of \emph{support} \cite{inoue2004equivalence,slota2013rise}. This leads to some highly unintuitive results, as illustrated by the following two examples.

\begin{quote}
It is the 31st of December and I plan to drive from San Jose to San Francisco to see the New Year fireworks. Due to previous experience I believe that if there is heavy fog in San Francisco, then the city will cancel the fireworks. It has been clear and sunny for the last days, so I believe that it will not be foggy today either. I decide to check the weather forecast nonetheless, which says that there will be heavy fog tonight in San Francisco. Since I trust the forecast more than my own meteorological skills, I have to revise my beliefs. By employing the distance-based revision method, I would end up believing that it will be foggy, while being undecided whether the fireworks will be cancelled. Formally, let $P_1 = \{\, \bot \leftarrow fog.,\; no\_fireworks \leftarrow fog. \,\}$ and $Q_1 = \{\, fog. \,\}$. Then the distance-based revision of $P_1$ by $Q_1$ would return $\{\, fog. \,\}$.
\end{quote}
\begin{quote}
I drive from San Jose to San Francisco every morning for work. I can use the 101 highway or the 280 freeway, but neither is particularly quicker. However, I believe that if there are roadworks on the 101, then the 280 is quicker. I was told by a friend that there are roadworks on the 101 currently, so I have been travelling on the 280. Now I hear on the radio that the roadworks finished and revise my beliefs. Using the distance-based revision method, I would end up with the belief that there are no roadworks on the 101, while still keeping the belief that the 280 is quicker. Formally, let $P_2 = \{\, 101\_roadworks.,\; 280\_quicker \leftarrow 101\_roadworks. \,\}$ and $Q_2 = \{\, \bot \leftarrow 101\_roadworks. \,\}$. Then the distance-based revision of $P_2$ by $Q_2$ would return $\{\, \bot \leftarrow 101\_roadworks.,\; 280\_quicker. \,\}$. 
\end{quote}

The first example demonstrates that the distance-based approach does not satisfy the preservation property. I do not conclude that the fireworks will be cancelled, even though I know now that it will be foggy. The revision operation simply disregards the second rule of $P_1$, which expresses the relationship between fog and fireworks cancellation. The reason for this is that the set of SE~models of the first rule of $P_1$ is a proper subset of the set of SE~models of the second rule. Thus, the set of SE~models of $P_1$ is exactly the set of SE~models of the first rule, which means that the second rule is invisible in the program-level view. The second example demonstrates that the distance-based approach does not satisfy the support property. I keep believing that the 280 is quicker, although the grounds to believe so do not hold any longer. The problem is that the dependency relationship between 280\_quicker and 101\_roadworks is captured on the rule-level, by the set of SE~models of the second rule of $P_2$, but not on the program-level, by the set of SE~models of the entire program $P_2$.

A second shortcoming is that the distance-based approach makes the definition of a corresponding contraction operator difficult to come by. In classical logic, contraction can be defined in terms of revision by using the negation of a sentence. However, in logic programs we do not have the luxury of negation of a program. A workaround could be to use the complement of the set of SE~models of the contracting program~$Q$ and select from this set the SE~models that are closest to the ones of the initial program~$P$. Yet, such a method may return SE~models that are somewhat unrelated to $P$ or $Q$, especially when the complement consists of a large number of SE~models.

The motivation for this work is to address these limitations. In particular, we propose here, on the one hand, revision operators that take into account information expressed by a program on the program-level and the rule-level in order to avoid such unintuitive results as just shown. On the other hand, we present corresponding contraction operators with similar properties. %To ensure the rationality of our operators, we will evaluate them against adaptations of the postulates from the AGM and belief base frameworks. 
The main contributions of this work can be summarised as follows.

\begin{itemize}
\item We provide new translations of the AGM and belief base revision and contraction postulates to the logic programming setting and establish formal relationships between these postulates and to previous translations. %We further presented an adaptation of the Levi and Harper identities to logic programs.
\item We introduce two sets of belief change operators for logic programs -- partial meet revision and contraction operators and ensconcement revision and contraction operators -- and show that each operator satisfies the relevant belief base revision or contraction postulates as well as the majority of AGM revision or contraction postulates. We also demonstrate that our partial meet and ensconcement revision operators address the shortcomings of the distance-based approach to logic program revision and that they are generalisations of the screened semi-revision approach for logic programs.
\item We establish that our ensconcement operators are generalisations of our partial meet operators and that the Levi and Harper identities hold for our operators. We further show that the outcome of a revision or contraction operation remains unaffected whether an ensconcement is defined over rules or subsets of a program.
\item We propose an algorithm to optimise the operations of partial meet and ensconcement revision or contraction of a logic program. 
\item We connect our results to the classic belief change frameworks by showing that our operators possess similar properties as their counterparts in propositional logic, that they conform fully to the belief base framework, and that they align more closely to the AGM and belief base frameworks than the distance-based revision operators. 
\end{itemize}

The remainder of this paper is organised as follows. We first provide the preliminaries in Section~\ref{sec:prelim} and review related work in Section~\ref{sec:rel-work}. We then present new translations of the AGM and belief base revision and contraction postulates to logic programs in Section~\ref{sec:adapting}. In Sections~\ref{sec:pm} and~\ref{sec:ens}, we propose partial meet and ensconcement belief change operators for logic programs, respectively, and evaluate their suitability with respect to the relevant postulates and existing operators. We establish the formal relationships between our operators in Section~\ref{sec:conn-btw-operators}. In Section~\ref{sec:localisedbc}, we present an algorithm to optimise the operations of revision or contraction on a logic program. We finally discuss our findings in relation to the classic belief change frameworks in Section~\ref{sec:disc-classic-bc} and conclude with a summary in Section~\ref{sec:concl}. Preliminary results from Sections~\ref{sec:pm} and~\ref{sec:localisedbc} were presented in a conference paper \cite{binnewies2015partial}.

\section{Preliminaries} \label{sec:prelim}
We first briefly recall syntax and semantics of logic programs and then review the foundations of belief change. % and then cover existing approaches to belief change in logic programs.

\subsection{Logic Programming}
Let $\mathcal{A}$ be a finite vocabulary of propositional atoms. A \emph{rule}~$r$ over $\mathcal{A}$ has the form
\begin{equation}
a_1; \dotsc; a_k; not\, b_1; \dotsc; not\, b_l \, \leftarrow \,  c_1, \dotsc , c_m, not\, d_1, \dotsc, not\, d_n. \label{eq:rule}
\end{equation}
Here, all $a_i, b_i, c_i, d_i \in \mathcal{A}$ and $k,l,m,n \geq 0$. The operators \lq $not$\rq , \lq ;\rq , and \lq ,\rq\ stand for default negation, disjunction, and conjunction, respectively. For convenience, let $H^+(r) = \{a_1,\dotsc,a_k\}$, $H^-(r) = \{b_1,\dotsc,b_l\}$, $B^+(r) = \{c_1,\dotsc,c_m\}$, and $B^-(r) = \{d_1,\dotsc,d_n\}$. If $k = 1$ and $l = m = n = 0$, then $r$ is called a \emph{fact} and we omit \lq $\leftarrow$\rq ; if $k = l = 0$, then $r$ is a \emph{constraint} and we denote the empty disjunction by $\bot$. Let $At(r)$ and $At(R)$ denote the set of all atoms that occur in a rule of the form~(\ref{eq:rule}) and in a set of rules~$R$, respectively.  A \emph{(generalised) logic program} is a finite set of rules of the form~(\ref{eq:rule}). We write $\lpa$ for the class of all logic programs that can be constructed from $\mathcal{A}$.

An \emph{interpretation} $Y \subseteq \mathcal{A}$ satisfies a program~$P$, denoted by $Y \models P$, if and only if~(iff) it is a model of all rules under the standard definition for propositional logic such that each rule represents a conditional and default negation is transcribed to classical negation. Let $Mod(P) = \{\, Y \mid Y \models P \,\}$. An \emph{answer set} \cite{gelfond1988stable} of a program~$P$ is any subset-minimal interpretation $Y$ that satisfies the \emph{reduct} of $P$ with respect to $Y$, denoted by $P^Y$ and defined as: 
$$
P^Y = \{\, H^+(r) \leftarrow B^+(r) \mid r \in P, H^-(r) \subseteq Y, \text{ and } B^-(r) \cap Y = \emptyset \,\}.
$$
The set of all answer sets of $P$ is denoted by $AS(P)$.

An \emph{SE interpretation} is a tuple $(X,Y)$ of interpretations with $X \subseteq Y \subseteq \mathcal{A}$. We usually write, e.g., $(ab,ab)$ instead of $(\{a,b\},\{a,b\})$ for legibility. Let $\se$ be the set of all SE interpretations over $\mathcal{A}$. For any set $S$ of SE interpretations, by $\overline{S}$ we denote the complement of $S$ with respect to $\se$, that is, $\overline{S} = \se \setminus S$. An SE interpretation $(X,Y)$ is an \emph{SE~model} \cite{turner2003strong} of a program $P$ iff $Y \models P$ and $X \models P^Y$. The set of all SE~models of $P$ is denoted by $SE(P)$ and $P$ is \emph{satisfiable} iff $SE(P) \neq \emptyset$. An interpretation $Y$ is an answer set of $P$ iff $(Y,Y) \in SE(P)$ and, for any $X \subset Y$, $(X,Y) \not \in SE(P)$. Often we drop explicit set notation for rules and their union, e.g., for rules $r,r' \in P$, we use $SE(r)$ to denote $SE(\{r\})$ and write $SE(r \cup r')$ instead of $SE(\{r\} \cup \{r'\})$. Note that $SE(P) = \bigcap_{r \in P} SE(r)$. Given two programs $P$ and $Q$, we say that~$P$ is \emph{strongly equivalent} \cite{lifschitz2001strongly} to $Q$, denoted by $P \equiv_s Q$, iff $SE(P) = SE(Q)$, and $P$ \emph{implies}~$Q$, denoted by $P \models_s Q$, iff $SE(P) \subseteq SE(Q)$. In the particular case of $SE(P) \subset SE(Q)$, we say that $P$ \emph{strictly implies} $Q$. The relation $\models_s$ is antitonic with respect to the program subset relation, i.e., $Q \subseteq P$ implies $P \models_s Q$. Furthermore, we write $\models_s P$ to express $SE(P) = \se$.

SE~models are a refinement of answer sets as they provide more information about the atoms in a program and their dependencies. For example, each of the following programs $P_1,P_2,\dotsc,P_9$ over $\mathcal{A} = \{a,b\}$ has $\{\es\}$ as the only answer set but the sets of SE~models are different for each program:
\begin{align*}
P_1 &= \{\, \bot \leftarrow a. \,\}  & SE(P_1) &= \{(\es,\es),(\es,b),(b,b)\} \\
P_2 &= \{\, \bot \leftarrow b. \,\} & SE(P_2) &= \{(\es,\es),(\es,a),(a,a)\} \\
P_3 &= \{\, a \leftarrow b. \,\} & SE(P_3) &= \{(\es,\es),(\es,a),(a,a),(\es,ab),(a,ab),(ab,ab)\} \\
P_4 &= \{\, b \leftarrow a. \,\} & SE(P_4) &= \{(\es,\es),(\es,b),(b,b),(\es,ab),(b,ab),(ab,ab)\} \\
P_5 &= \{\, \bot \leftarrow a,b. \,\} & SE(P_5) &= \{(\es,\es),(\es,a),(a,a),(\es,b),(b,b)\} \\
P_6 &= \{\, \bot \leftarrow not\, a,b. \,\} & SE(P_6) &= \{(\es,\es),(\es,a),(a,a),(\es,ab),(a,ab),(b,ab),(ab,ab)\} \\
P_7 &= \{\, \bot \leftarrow a,not\, b. \,\} & SE(P_7) &= \{(\es,\es),(\es,b),(b,b),(\es,ab),(a,ab),(b,ab),(ab,ab)\} \\
P_8 &= \{\, a;not\, b. \,\} & SE(P_8) &= \{(\es,\es),(\es,a),(a,a),(a,ab),(ab,ab)\} \\
P_9 &= \{\, not\, a; b. \,\} & SE(P_9) &= \{(\es,\es),(\es,b),(b,b),(b,ab),(ab,ab)\}
\end{align*}
Informally, we can interpret the content of an SE~model $(X,Y)$ on a three-valued scale. Any atoms in $X$ are true, any atoms not in $Y$ are false, and any atoms in $Y$ but not in~$X$ are undefined.

\subsection{Belief Change} \label{sec:bc}
The AGM framework \cite{alchourron1985logic,gardenfors1988knowledge} defines \emph{expansion}, \emph{revision}, and \emph{contraction} as the change operations on a body of beliefs held by an agent, called a \emph{belief state} henceforth, in light of some new information. In an expansion, new beliefs are incorporated into a belief state, regardless of any inconsistencies that may arise. A revision operation also incorporates new beliefs into a belief state, but in such a way that the resulting belief state is consistent. This is achieved by discarding some existing beliefs. During a contraction, some beliefs in a belief state are removed without adding new beliefs.

In the AGM framework, a belief state is modelled as a \emph{belief set}, defined as a set of sentences from some logic-based language $\mathcal{L}$ that is closed under logical consequence, i.e., when all beliefs implied by a knowledge base are explicitly represented in the knowledge base. Let $K$ be a belief set, $\phi$ and $\psi$ sentences, $K_\bot$ denote the inconsistent belief set, and $Cn(\cdot)$ stand for a logical consequence function. By $\phi \equiv \psi$ we mean $Cn(\phi) = Cn(\psi)$. The expansion of $K$ by~$\phi$, written $K \oplus \phi$, is defined as $K \oplus \phi = Cn(K \cup \{\phi\})$. The AGM framework provides a set of postulates that any rational revision operator should satisfy. The postulates are listed as follows, where~$\circledast$ represents a revision operator.
\begin{enumerate}[({$\circledast$}1)]
\item $K \circledast \phi$ is a belief set
\item $\phi \in K \circledast \phi$
\item $K \circledast \phi \subseteq K \oplus \phi$
\item If $\lnot \phi \not \in K$, then $K \oplus \phi \subseteq K \circledast \phi$
\item $K \circledast \phi = K_\bot$ iff $\vdash \lnot \phi$
\item If $\phi_1 \equiv \phi_2$, then $K \circledast \phi_1 = K \circledast \phi_2$
\item $K \circledast (\phi \wedge \psi) \subseteq (K \circledast \phi) \oplus \psi$
\item If $\lnot \psi \not \in K \circledast \phi$, then $(K \circledast \phi) \oplus \psi \subseteq K \circledast (\phi \wedge \psi)$
\end{enumerate}

($\circledast$1) requires that the outcome of a revision is a belief set.\ ($\circledast$2) states that the revising sentence is contained in the revised belief set.\ ($\circledast$3) asserts that a belief set revised by a sentence is always a subset of the belief set expanded by that sentence.\ ($\circledast$3) and ($\circledast$4) together state that revision coincides with expansion in cases when the revising sentence is consistent with the initial belief set.\ ($\circledast$5) guarantees that a revision outcome is consistent, unless the revising sentence is logically impossible.\ ($\circledast$6) ensures that logically equivalent sentences lead to the same revision outcomes.\ ($\circledast$7) and ($\circledast$8) together enforce $K$ to be minimally changed in a revision by both $\phi$ and $\psi$, such that the outcome is the same as the expansion of $K \circledast \phi$ by $\psi$, provided that $\psi$ is consistent with $K \circledast \phi$.

In the concrete case that a belief state is represented as a finite set of propositional formulas, the following set of postulates is equivalent to the set ($\circledast$1)--($\circledast$8) \cite{katsuno1991propositional}. Let $\phi,\psi,\mu$ be propositional formulas.
\begin{enumerate}[({$\circledast$}1KM)]
\item $\phi \circledast \psi$ implies $\psi$
\item If $\phi \wedge \psi$ is satisfiable, then $\phi \circledast \psi \equiv \phi \wedge \psi$
\item If $\psi$ is satisfiable, then $\phi \circledast \psi$ is satisfiable
\item If $\phi_1 \equiv \phi_2$ and $\psi_1 \equiv \psi_2$, then $\phi_1 \circledast \psi_1 \equiv \phi_2 \circledast \psi_2$
\item $(\phi \circledast \psi) \wedge \mu$ implies $\phi \circledast (\psi \wedge \mu)$
\item If $(\phi \circledast \psi) \wedge \mu$ is satisfiable, then $\phi \circledast (\psi \wedge \mu)$ implies $(\phi \circledast \psi) \wedge \mu$
\end{enumerate}

($\circledast$1KM) requires that the revising formula can be derived from the revision outcome.\ ($\circledast$2KM) specifies that revision corresponds to conjunction whenever the revising formula is consistent with the formula to be revised.\ ($\circledast$3KM) guarantees consistency of a revision outcome whenever the revising formula is consistent.\ ($\circledast$4KM) states that revising logically equivalent formulas by logically equivalent formulas leads to logically equivalent results.\ ($\circledast$5KM) and ($\circledast$6KM) together stipulate that the revision by a conjunction leads to the same outcome as revising by one conjunct and then forming the conjunction with the other conjunct, provided that the conjunction thus formed is satisfiable. 

The AGM framework also provides a set of postulates that any rational contraction operator should satisfy. The postulates are given below, where $\ominus$ represents a contraction operator.
\begin{enumerate}[({$\ominus$}1)]
\item $K \ominus \phi $ is a belief set
\item $K \ominus \phi \subseteq K$
\item If $\phi \not \in K$, then $K \ominus \phi = K$
\item If $\not \vdash \phi$, then $\phi \not \in K \ominus \phi$
\item $K \subseteq (K \ominus \phi) \oplus \phi$
\item If $\phi_1 \equiv \phi_2$, then $K \ominus \phi_1 = K \ominus \phi_2$
\item $K \ominus \phi \cap K \ominus \psi \subseteq K \ominus \phi \wedge \psi$
\item If $\phi \not \in K \ominus \phi \wedge \psi$, then $K \ominus \phi \wedge \psi \subseteq K \ominus \phi$
\end{enumerate}

($\ominus$1) requires that the outcome of a contraction is a belief set.\ ($\ominus$2) ensures that no new beliefs are introduced during a contraction.\ ($\ominus$3) stipulates that the belief set remains unchanged during a contraction operation whenever the sentence to be contracted is not contained in it.\ ($\ominus$4) states that a contracting sentence is not a logical consequence of the contracted belief set, unless the sentence is a tautology.\ ($\ominus$5) requires that the original belief set can be recovered by expanding a contracted belief set by the sentence that was contracted.\ ($\ominus$6) ensures that logically equivalent sentences lead to the same contraction outcomes.\ ($\ominus$7) guarantees that any beliefs retained in a contraction by $\phi$ and in a contraction by $\psi$ are also retained in a contraction by both $\phi$ and $\psi$.\ ($\ominus$8) specifies that any beliefs retained in a contraction by both $\phi$ and $\psi$ are also retained in a contraction by $\phi$, whenever $\phi$ itself is not retained.

The appropriateness of the Recovery postulate ($\ominus$5) within this set of contraction postulates has been discussed intensively \cite{fuhrmann1991theory,hansson1991belief,makinson1987status,nayak1994foundational,niederee1991multiple}.  To replace the Recovery postulate in expressing that no beliefs should be retracted unduly during a contraction operation, alternative postulates were proposed. \citeN{hansson1991belief} offered the following postulate:
\begin{enumerate}[({$\ominus$}1r)]
\setcounter{enumi}{4}
\item If $\psi \in K \setminus (K \ominus \phi)$, then there is a set $K'$ such that $K \ominus \phi \subseteq K' \subset K$ and $\phi \not \in Cn(K')$ but $\phi \in Cn(K' \cup \{\psi\})$.
\end{enumerate}
The Relevance postulate~($\ominus$5r) states that a sentence $\psi$ should only be removed during the contraction of a sentence $\phi$ from $K$ if $\psi$ is relevant for implying $\phi$. In the presence of~($\ominus$1)--($\ominus$3), ($\ominus$5) is equivalent to~($\ominus$5r) in propositional logic \cite{hansson1991belief}. More recently, \citeN{ferme2008axiomatic} presented the following Disjunctive Elimination postulate~($\ominus$5de):
\begin{enumerate}[({$\ominus$}1de)]
\setcounter{enumi}{4}
\item If $\psi \in K \setminus (K \ominus \phi)$, then $K \ominus \phi \not \vdash \phi \vee \psi$.
\end{enumerate}
According to~($\ominus$5de), a sentence $\psi$ should only be removed during the contraction of a sentence $\phi$ from $K$ if the contraction result does not imply the disjuntion of $\phi$ and $\psi$. In the presence of~($\ominus$2)--($\ominus$3), ($\ominus$5r) is equivalent to ($\ominus$5de) in propositional logic \cite{ferme2008axiomatic}.

One of the classic constructions to implement belief change is \emph{partial meet contraction} \cite{alchourron1985logic}, which we recapitulate here. A set $K'$ is a \emph{remainder set} of a set $K \subseteq \mathcal{L}$ with respect to a sentence $\phi$ iff 
\begin{enumerate}[a)]
\item $K' \subseteq K$,
\item $K' \not \vdash \phi$, and
\item for any $K''$ with $K' \subset K'' \subseteq K: K'' \vdash \phi$.
\end{enumerate}
The set of all remainder sets of $K$ with respect to $\phi$ is denoted by $K \bot \phi$. A \emph{selection function} $\gamma$ \emph{for a belief set} $K$ is a function such that (i) if $K \bot \phi \neq \es$, then $\es \neq \gamma(K \bot \phi) \subseteq K \bot \phi$ and (ii) $\gamma(K \bot \phi) = \{K\}$ otherwise. A \emph{partial meet contraction operator $\ominus_\gamma$} for $K$ is defined as: $K \ominus_\gamma \phi = \bigcap \gamma (K \bot \phi)$. The following representation theorem shows that the set of postulates~($\ominus$1)--($\ominus$6) exactly characterises the class of partial meet contraction operators.

\begin{thm} \cite{alchourron1985logic}
For any belief set $K$, $\ominus_\gamma$ is a partial meet contraction operator for $K$ iff $\ominus_\gamma$ satisfies~($\ominus$1)--($\ominus$6).
\end{thm}

By placing further restrictions on the selection function, the representation theorem can be extended to the full set of postulates. A transitively relational selection function $\gamma'$ for $K$ is determined by a transitive relation $\unlhd$ over $2^K$ such that $\gamma'(K \bot \phi) = \{\, K' \in K \bot \phi \mid K'' \unlhd K' \text{ for all } K'' \in K \bot \phi \,\}$. A partial meet contraction operator $\ominus_{\gamma'}$ determined by a transitively relational selection function $\gamma'$ is called a \emph{transitively relational partial meet contraction operator}.

\begin{thm} \cite{alchourron1985logic}
For any belief set $K$, $\ominus_{\gamma'}$ is a transitively relational partial meet contraction operator for $K$ iff $\ominus_{\gamma'}$ satisfies~($\ominus$1)--($\ominus$8).
\end{thm} 

A corresponding \emph{(transitively relational) partial meet revision operator}~$\circledast_\gamma$ ($\circledast_{\gamma'}$) that satisfies ($\circledast$1)--($\circledast$6) (($\circledast$1)--($\circledast$8)) can be obtained from a (transitively relational) partial meet contraction operator via the \emph{Levi identity}: $K \circledast \phi = (K \ominus \lnot \phi) \oplus \phi$ \cite{gardenfors1981epistemic,levi1977subjunctives}. The inverse identity, which constructs a contraction operator from a revision operator, is due to \citeN{harper1976rational}: $K \ominus \phi = K \cap (K \circledast \lnot \phi)$.

While the AGM approach provides an effective framework to conduct belief change, the representation of belief states in the form of belief sets has some shortcomings (see \cite{hansson1999textbook} for a detailed discussion). From a practical perspective, main drawbacks of belief sets are that they are generally large objects, since all logical consequences of all beliefs are contained, and that it is impossible to distinguish between inconsistent belief sets, as inconsistent belief sets consist of the entire language. \emph{Belief bases} \cite{fuhrmann1991theory,hansson1989new,rott1992modellings} are an alternative representation of belief states. A belief base is a set of sentences from $\mathcal{L}$ that is not necessarily closed under logical consequence.

\citeN{hansson1993reversing} defined a \emph{partial meet base contraction operator} $-_\gamma$ for a belief base~$B$ as $B -_\gamma \phi = \bigcap \gamma (B \bot \phi)$ and showed that the following set of postulates exactly characterises the class of partial meet base contraction operators.

\begin{enumerate}[($-$1)]
\item $B - \phi \subseteq B$
\item If $\not \vdash \phi$, then $\phi \not \in Cn(B - \phi)$
\item If $\psi \in B \setminus (B - \phi)$, then there is a set $B'$ such that $B - \phi \subseteq B' \subset B$ and $\phi \not \in Cn(B')$ but $\phi \in Cn(B' \cup \{\psi\})$
\item If it holds for all $B' \subseteq B$ that $\phi \in Cn(B')$ iff $\psi \in Cn(B')$, then $B - \phi = B - \psi$
\end{enumerate}

\begin{thm} \cite{hansson1993reversing}
For any belief base $B$, $-_\gamma$ is a partial meet base contraction operator for $B$ iff $-_\gamma$ satisfies~($-$1)--($-$4).
\end{thm}
Note that ($-$1), ($-$2), and ($-$3) in the belief base setting correspond directly to ($\ominus$2), ($\ominus$4), and ($\ominus$5r) in the AGM setting, respectively. ($-$4) states that if any parts of $B$ which imply $\phi$ also imply $\psi$, then the same parts of $B$ will be retained in a contraction by $\phi$ as in a contraction by $\psi$.

He also defined a corresponding \emph{partial meet base revision operator} $\divideontimes_\gamma$ for a belief base $B$ as $B \divideontimes_\gamma \phi = (B -_\gamma \lnot \phi) \cup \{\phi\}$ and showed that the following set of postulates exactly characterises the class of partial meet base revision operators.

\begin{enumerate}[($\divideontimes$1)]
\item $\phi \in B \divideontimes \phi$
\item $B \divideontimes \phi \subseteq B \cup \{\phi\}$
\item If $\psi \in B \setminus (B \divideontimes \phi)$, then there is a set $B'$ such that $B \divideontimes \phi \subseteq B' \subset B \cup \{\phi\}$ and $\lnot \phi \not \in Cn(B')$ but $\lnot \phi \in Cn(B' \cup \{\psi\})$
\item If it holds for all $B' \subseteq B$ that $B' \cup \{\phi\}$ is consistent iff $B' \cup \{\psi\}$ is consistent, then $B \cap (B \divideontimes \phi) = B \cap (B \divideontimes \psi)$
\item If $\not \vdash \lnot \phi$, then $\lnot \phi \not \in Cn(B \divideontimes \phi)$
\end{enumerate}

\begin{thm} \cite{hansson1993reversing}
For any belief base $B$, $\divideontimes_\gamma$ is a partial meet base revision operator for $B$ iff $\divideontimes_\gamma$ satisfies~($\divideontimes$1)--($\divideontimes$5).
\end{thm}
The pendants to ($\divideontimes$1) and ($\divideontimes$2) in the AGM framework are ($\circledast$2) and ($\circledast$3), respectively. ($\divideontimes$3) requires $\psi$ to only be removed from $B$ if it would otherwise make the revision outcome inconsistent. ($\divideontimes$4) mandates that if any parts of $B$ which are consistent with~$\phi$ are also consistent with $\psi$, then the same parts of $B$ will be retained in a revision by~$\phi$ as in a revision by $\psi$. ($\divideontimes$5) is a weaker version of ($\circledast$5).

\citeN{williams1994logic} proposed further belief change operators for belief bases, which rely on an ordering over the sentences contained in a belief base, called ensconcement. An \emph{ensconcement associated with a belief base~$B$} is any total preorder $\preccurlyeq$ on~$B$ that satisfies the following conditions.
\begin{enumerate}[($\preccurlyeq$1)]
\item For all $\phi \in B: \{\, \psi \in B \mid \phi \prec \psi \,\} \not \vdash \phi$
\item For all $\phi,\psi \in B: \phi \preccurlyeq \psi$ iff $\vdash \psi$
\end{enumerate}
A sentence $\psi$ is at least as ensconced as a sentence $\phi$ iff $\phi \preccurlyeq \psi$, and $\psi$ is strictly more ensconced than $\phi$ iff $\phi \prec \psi$. Condition~($\preccurlyeq$1) states that sentences which are strictly more ensconced than a sentence~$\phi$ do not entail~$\phi$. Condition~($\preccurlyeq$2) requires any tautologies in the belief base to be most ensconced. The \emph{proper cut of $B$ for $\phi$} is $cut_\prec(\phi) = \{\, \psi \in B \mid \{\, \chi \in B \mid \psi \preccurlyeq \chi \,\} \not \vdash \phi \,\}$. An \emph{ensconcement contraction operator}~$\ominus_\preccurlyeq$ for $B$ is defined as: $\psi \in B \ominus_\preccurlyeq \phi$ iff $\psi \in B$ and either~$\vdash \phi$ or $cut_\prec(\phi) \cup \{\lnot \phi\} \vdash \psi$. An \emph{ensconcement revision operator}~$\circledast_\preccurlyeq$ for $B$ is defined as: $\psi \in B \circledast_\preccurlyeq \phi$ iff (i) $\psi = \phi$ or (ii) $\psi \in B$ and either $\vdash \lnot \phi$ or $cut_\prec(\lnot \phi) \cup \{\phi\} \vdash \psi$.

An ensconcement contraction operator $\ominus_\preccurlyeq$ satisfies ($-$1), ($-$2), and 
\begin{enumerate}[($-$1)]
\setcounter{enumi}{4}
\item If $\phi \not \in Cn(B)$, then $B - \phi = B$
\item If $\phi_1 \equiv \phi_2$, then $B - \phi_1 = B - \phi_2$
\item $B - \phi \wedge \psi = B - \phi$ or $B - \phi \wedge \psi = B - \psi$ or $B - \phi \wedge \psi = B - \phi \cap B - \psi$
\item If $\psi \in B \setminus (B - \phi)$, then $B - \phi \not \vdash \phi \vee \psi$
\end{enumerate}

\begin{thm} \cite{ferme2008axiomatic}
Let $B$ be a belief base and $\ominus_\preccurlyeq$ an ensconcement contraction operator for $B$. Then $\ominus_\preccurlyeq$ satisfies ($-$1), ($-$2), and ($-$5)--($-$8).\footnote{Please note that the proof of the representation theorem (Theorem~14 in \cite{ferme2008axiomatic}) contains an error, as acknowledged by the authors. The theorem only holds in the direction from operator to postulates as stated above.}
\end{thm}
Postulates ($-$5), ($-$6), and ($-$8) correspond directly to ($\ominus$3), ($\ominus$6), and ($\ominus$5de) in the AGM setting, respectively. ($-$7) states that a contraction by a conjuntion is the result of contracting by the first of the conjuncts, the result of contracting by the second of the conjuncts, or the common part of these two results. In the belief base framework, the relationship between ($-$3) and ($-$8) is different to the one between ($\ominus$5r) and ($\ominus$5de) in the AGM framework: ($-$3) implies ($-$8) but not vice versa \cite{ferme2008axiomatic}.

\section{Related Work} \label{sec:rel-work}
One of the key developments for adapting the AGM framework of belief change to logic programs came with the \emph{distance-based approach} to logic program revision \cite{delgrande2013model}. It is built on the monotonic SE semantics for logic programs and understands a belief state as the set of SE~models of a program. In that work, the formula-based revision postulates~($\circledast$1KM)--($\circledast$6KM) are translated to logic programs as follows, where a revision operator $*$ is a  function from $\lpa \times \lpa$ to $\lpa$ and the expansion of $P$ by $Q$, denoted $P \dotplus Q$, is understood as $P \dotplus Q = R$ such that $R \in \lpa$ and $SE(R) = SE(P) \cap SE(Q)$.
\begin{enumerate}[({$*$}1m)]
\item $P * Q \models_s Q$
\item If $P \dotplus Q$ is satisfiable, then $P * Q \equiv_s P \dotplus Q$
\item If $Q$ is satisfiable, then $P * Q$ is satisfiable
\item If $P_1 \equiv_s P_2$ and $Q_1 \equiv_s Q_2$, then $P_1 * Q_1 \equiv_s P_2 * Q_2$
\item $(P * Q) \dotplus R \models_s P * (Q \dotplus R)$
\item If $(P * Q) \dotplus R$ is satisfiable, then $P * (Q \dotplus R) \models_s (P * Q) \dotplus R$
\end{enumerate}

The approach adapts two revision operators from classic belief change to logic programs, namely, Dalal's revision operator \cite{dalal1988investigations} and Satoh's revision operator \cite{satoh1988nonmonotonic}. Informally, to revise a program $P$ by a program $Q$, the operators return those SE~models from the set of SE~models of $Q$ that are closest to the SE~models of $P$, where closeness is determined by Dalal's or Satoh's notion of distance. \citeN{delgrande2013model} identified that the adaptation of Satoh's revision operator gives more intuitive results than the adaptation of Dalal's revision operator, so we will focus on the former here. This restriction has no effect on our later discussions.

We briefly restate the definition and main result of the distance-based approach. Let~$\Delta$ stand for the symmetric difference between two sets $X,Y$, that is, $X \Delta Y = (X \setminus Y) \cup (Y \setminus X)$. For any two pairs of sets $(X,X'), (Y,Y')$, let
\begin{align*}
&(X,X') \Delta (Y,Y') = (X \Delta Y, X' \Delta Y'); \\
&(X,X') \subseteq (Y,Y') \text{ iff } X' \subseteq Y', \text{ and if } X' = Y', \text{ then } X \subseteq Y; \\
&(X,X') \subset (Y,Y') \text{ iff } (X,X') \subseteq (Y,Y') \text{ and } (Y,Y') \nsubseteq (X,X').
\end{align*}
For any two sets $E,E'$, let
\begin{align*}
\sigma(E,E') = \{\, A_1 \in E \mid &\text{ there exists a } B_1 \in E' \text{ such that for all } 
A_2 \in E \\
&\text { and for all } B_2 \in E' \text{ it holds that } A_1 \Delta B_1 \subseteq  A_2 \Delta B_2 \,\}.
\end{align*}

\begin{dfn} \cite{delgrande2013model}
Let $P,Q \in \lpa$. The revision of~$P$ by~$Q$, denoted $P \star Q$, is defined as $P \star Q = R$ such that $R \in \lpa$ and $SE(R) = SE(Q)$ if $SE(P) = \es$, and otherwise
\begin{align*}
SE(R) = \{\, (X,Y) \mid\ &Y \in \sigma(Mod(Q),Mod(P)), X \subseteq Y, \\
&\text{and if } X \subset Y, \text{ then } (X,Y) \in \sigma(SE(Q),SE(P)) \,\}.
\end{align*}
\end{dfn}

\begin{thm} \cite{delgrande2013model}
The revision operator $\star$ satisfies ($*$1m)--($*$5m). 
\end{thm}

The distance-based approach was extended by two representation theorems \cite{delgrande2013agm,schwind2013characterization}, stating that any logic program revision operator satisfying ($*$1m)--($*$6m) plus some additional conditions can be characterised by some preorder over a set of SE~models.

Besides the distance-based approach, few other methods for logic program revision have been proposed. The \emph{screened semi-revision approach} for logic programs \cite{krumpelmann2012belief} is based on answer set semantics and aligns itself with the belief base framework. The approach assumes a belief state to be the set of rules belonging to a program and combines adaptations of the constructions of semi-revision \cite{hansson1997semi} and screened revision \cite{makinson1997screened} into a screened consolidation operation for logic programs. The consolidation operator first finds all maximal subsets of one program that are consistent with a second program under answer set semantics, then selects exactly one of these subsets, and returns this subset together with the second program as the  outcome. 

We review the formal definitions of the screened consolidation operator and the main result here. Let $P \in \lpa$ and $Q \subseteq P$. The set of screened remainder sets of $P$ with respect to $Q$ is
$$
P \bot_! Q = \{\, R \mid Q \subseteq R \subseteq P, AS(R) \neq \es \text{ and for all } R'
\text{ with } R \subset R' \subseteq P: AS(R') = \es \,\}.
$$
A maxichoice selection function $\gamma_P$ for $P$ is a function such that for any $Q \in \lpa$: (i) if $P \bot_! Q \neq \es$, then $\gamma_P(P \bot_! Q) = R$ for some $R \in P \bot_! Q$, and (ii) if $P \bot_! Q = \es$, then $\gamma_P(P \bot_! Q) = P$.

\begin{dfn} \cite{krumpelmann2012belief} \label{dfn:scr}
Let $P,Q \in \lpa$ and $\gamma_P$ be a maxichoice selection function for $P$. A screened consolidation operator $\scr$ for $P$ is defined as $P \scr Q = \gamma_P (P \bot_! Q)$.
\end{dfn}
The authors propose the following adaptation of partial meet base revision postulates that any screened consolidation operator $!$ should satisfy, where $!$ is a function from $\lpa \times \lpa$ to $\lpa$, and show that $\scr$ is exactly characterised by these postulates.

\begin{enumerate}[($!$1)]
\item $Q \subseteq P \,!\, Q$
\item $P\,!\,Q \subseteq P$
\item If $r \in P \setminus (P\,!\,Q)$, then $AS(P\,!\,Q) \neq \es$ and $AS(P\,!\,Q \cup \{r\}) = \es$
\item If it holds for all $P' \subseteq P$ that $AS(P' \cup Q) \neq \es$ iff $AS(P' \cup R) \neq \es$, then $P \cap ((P \cup Q)\,!\,Q) = P \cap ((P \cup R)\,!\,R)$
\item If there exists some $P'$ such that $Q \subseteq P' \subseteq P$ and $AS(P') \neq \es$, then $AS(P\,!\,Q) \neq \es$
\end{enumerate}

\begin{thm} \cite{krumpelmann2012belief}
For any $P \in \lpa$, $\scr$ is a screened consolidation operator for $P$ iff $\scr$ satisfies ($!$1)--($!$5).
\end{thm}

The \emph{program-level approach} to logic program revision \cite{delgrande2010program} is also based on answer set semantics and assumes the beliefs that make up a belief state to be the answer sets of a program. The revision operation relies on extending the standard answer set semantics to three-valued answer set semantics for determining the outcome. To revise a program~$P$ by a program~$Q$, for each three-valued answer set $X$ of $Q$, all maximal subsets $R$ of $P$ are selected such that $X$ is a subset of each three-valued answer set $X'$ of $R \cup Q$. The revision operation returns a set of answer sets that correspond to each $X'$ as the result. In the author's view, the AGM revision postulates~($\circledast$3), ($\circledast$4), ($\circledast$7), and~($\circledast$8) are inappropriate in the context of nonmonotonic semantics. An adaptation of the remaining postulates is fulfilled by the revision operation.

The work of \citeN{zhuang2016reconsidering} concerns itself with the revision of a disjunctive logic program by another. The authors observed that for a belief revision operator in a nonmonotonic setting, the task of inconsistency resolving can be done not only by removing old beliefs but also by adding new beliefs. Based on this observation, they proposed a variant of partial meet revision. For resolving the inconsistency between the original and the new beliefs, the variant obtains not only maximal subsets of the initial program that are consistent with the new one, but also minimal supersets of the initial program that are consistent with the new one. A representation theorem is provided. Since the idea of resolving inconsistency by adding new beliefs is beyond the classic AGM approach, some extra postulates are required to characterise the variant.

\citeN{inoue2004equivalence} argue that neither the set of answer sets nor the set of SE~models of a program provide enough detail to revise a program. They illustrate that while two programs $\{\, a.,\, b \leftarrow not\, a. \,\}$ and $\{\, a. \,\}$ have the same set of SE~models, and thus the same set of answer sets, they should be treated differently during a revision operation, since $b$ should be derived from the first program whenever the rule $a.$ is discarded. We call this property \emph{preservation} here. They further argue that revision operations should distinguish between two programs $\{\, a.,\, b. \,\}$ and $\{\, a.,\, b \leftarrow a. \,\}$. While these two programs again share the same set of SE~models, they too should not be interchangeable, since after a removal of the rule $a.$, $b$ should not be derived from the latter program any more. This property has become known as \emph{support} \cite{slota2013rise}. To address these two issues, \citeN{inoue2004equivalence} introduced a new notion of program equivalence that is stricter than strong equivalence, called \emph{C-update equivalence}: any two programs $P_1,P_2 \in \lpa$ are C-update equivalent iff $P_1 \setminus P_2 \equiv_s P_2 \setminus P_1$.

A relative of belief revision is \emph{belief update} \cite{katsuno1992difference}. The difference between these two is usually understood in the way that belief revision addresses changes to a belief state brought about by some new information about a static world, whereas belief update covers changes to a belief state due to dynamics in the world described by the belief state. Similar to belief revision, the belief update framework prescribes a set of postulates that each rational update operator should satisfy and provides a construction that complies with it. A number of update operators for logic programs have been proposed, which differ greatly  in the degree of alignment to the classic belief update framework. On the whole, these approaches rely predominantly on syntactic transformations of programs and return a set of answer sets instead of an updated program as the outcome. The landscape of update operators has already been reviewed exhaustively in other places, for example, detailed overviews are given by \citeN{delgrande2004classification} and \citeN{slota2012updates}. Of interest in the current context is the \emph{exception-based update} approach \cite{slota2012robust}, which introduces \emph{RE (robust equivalence) models} as an extension of SE~models. An RE model of a program $P$ is any SE interpretation $(X,Y)$ such that $X \models P^Y$. The authors regard a belief state as the collection of the sets of RE models of the rules in a program. In the update operation, exceptions in the form of RE models are added to those sets of RE models of the initial program that are incompatible with the RE models of the updating program. Incompatibilities between two sets of RE models are determined by differences in the truth values of atoms occurring in both sets. The update operator satisfies the majority of the update postulates adapted to logic programs.

%A logic program with an associated ensconcement or entrenchment has similar design features as an ordered logic program. 
An \emph{ordered logic program} is a tuple $(P,<)$ such that $P$ is a logic program and $<$ is a preference ordering over the rules in $P$. Thus, it is conceivable to express the revision of $P$ by~$Q$ as $(P \cup Q,<)$, where $<$ is some appropriate preference ordering on $P \cup Q$ (from \cite{delgrande2007preference}, for example), and employ one of the different semantics proposed \cite{brewka1999preferred,delgrande2003framework,schaub2003semantic} to obtain the preferred answer sets of this ordered logic program. An ordered logic program can be transformed into a standard logic program, so that the preferred answer sets of the former are exactly the answer sets of the latter \cite{delgrande2003framework}. Yet, the transformed program may bear no syntactic relation to the original program. More importantly, an ordered logic program $(P \cup Q,<)$ may be inconsistent, i.e., may have no answer sets, even if $P$ and $Q$ themselves are consistent \cite{delgrande2004classification}. These characteristics %clearly separate ordered logic programs from our approach here and 
make ordered logic programs rather unsuitable as a methodology for logic program belief change in general.

\section{Adapting the Belief Change Frameworks} \label{sec:adapting}
Before we set out to define new constructions of logic program belief change, we translate the ideas of the classic belief change frameworks to the logic programming setting. We assume that a belief state is represented in the form of a program from $\lpa$ and new information to expand/revise/contract this program comes in the form of another program from $\lpa$. Even though a consequence relation for logic programs under SE semantics exists \cite{eiter2004simplifying,wong2008sound}, logic programs are per se not closed under logical consequence. Thus, it seems natural to align logic program belief change with the belief base framework. However, as the majority of previous approaches to logic program revision have focussed on adapting the AGM framework, we will consider it here as well to enable us to draw proper comparisons.

As pointed out above, in the case of propositional knowledge bases the set of formula-based revision postulates ($\circledast$1KM)--($\circledast$6KM) is equivalent to the set ($\circledast$1)--($\circledast$8). However, in the context of logic programs, it is worth having a closer look at the relationship between these two sets. Besides adaptations for the purpose of logic program \lq update\rq\ operations under answer set semantics \cite{eiter2002properties} or $N_2$ logic \cite{osorio2007updates}, we find that postulates for logic program revision have usually been built on the formula-based revision postulates ($\circledast$1KM)--($\circledast$6KM) until now \cite{delgrande2013model,delgrande2013agm,schwind2013characterization}. They are given in the form of~($*$1m)--($*$6m) (see Section~\ref{sec:rel-work}). We will now present a new set of postulates for logic program revision operators, based on the original AGM postulates ($\circledast$1)--($\circledast$8), and then explain their relationship to ($*$1m)--($*$6m). We will discover that the equivalence of the two sets of postulates under propositional logic does not carry over to logic programs, that, in fact, the adaptation of~($\circledast$1)--($\circledast$8) to logic programs leads to postulates that are in most cases stricter than~($*$1m)--($*$6m), since $Q \subseteq P$ implies $P \models_s Q$ but the converse does not hold.

Let $P,Q,R \in \lpa$ and a revision operator $*$ be a function from $\lpa \times \lpa$ to $\lpa$. In the following, we understand expansion, denoted by the operator~$+$, as $P + Q = P \cup Q$.
\begin{enumerate}[({$*$}1)]
\item $P * Q \in \lpa$ 
\item $Q \subseteq P * Q$
\item $P * Q \subseteq P + Q$
\item If $P + Q$ is satisfiable, then $P + Q \subseteq P * Q$
\item $P * Q$ is satisfiable iff $Q$ is satisfiable
\item If $Q \equiv_s R$, then $P * Q \equiv_s P * R$
\item $P * (Q + R) \subseteq (P * Q) + R$
\item If $(P * Q) + R$ is satisfiable, then $(P * Q) + R \subseteq P * (Q + R)$
\end{enumerate}

($*$1) is a basic but nonetheless crucial condition that is included in previous adaptations by the requirement that $*$ is a function from $\lpa \times \lpa$ to~$\lpa$. ($*$2) stipulates that the revising program is always included in a revision outcome and is a stronger version of ($*$1m) since ($*$2) implies ($*$1m) but not vice versa. ($*$3) requires that a revision outcome never contains elements not in~$P$ or $Q$. This condition is covered by previous adaptations only for the case that $P + Q$ is satisfiable in ($*$2m). Together, ($*$3) and ($*$4) state that revision coincides with expansion if $P + Q$ is satisfiable and thus imply ($*$2m) but not vice versa. ($*$5) is the stricter, biconditional version of ($*$3m) emphasising the consistency of a revision outcome. ($*$6) guarantees that revision by strongly equivalent programs leads to strongly equivalent results. Unlike belief sets, logic programs are not closed under logical consequence, so that we cannot expect from two different bodies of information that are merely strongly equivalent to be equal after a revision. Thus, the consequent is translated as ``$P * Q \equiv_s P * R$'' from the original postulate~($\circledast$6). ($*$4m) is stricter than ($*$6). ($*$7) and~($*$8) capture the minimal change condition by requiring that the revision of $P$ by the expansion of $Q$ with $R$ has the same outcome as $P * Q$ expanded with $R$, whenever the latter is satisfiable. ($*$7) and~($*$8) imply~($*$5m) and~($*$6m) but not vice versa, respectively.

A translation of the AGM contraction postulates to logic programs is listed below, where $P,Q,R \in \lpa$ and a contraction operator $\dotminus$ is a function from $\lpa \times \lpa$ to $\lpa$.
\begin{enumerate}[({$\dotminus$}1)]
\item $P \dotminus Q \in \lpa$
\item $P \dotminus Q \subseteq P$
\item If $P \not \models_s Q$, then $P \dotminus Q = P$
\item If $\not \models_s Q$, then $P \dotminus Q \not \models_s Q$
\item $P \subseteq (P \dotminus Q) + Q $
\item If $Q \equiv_s R$, then $P \dotminus Q = P \dotminus R$
\item $(P \dotminus Q) \cap (P \dotminus R) \subseteq P \dotminus (Q + R)$
\item If $P \dotminus (Q + R) \not \models_s Q$, then $P \dotminus (Q + R) \subseteq P \dotminus Q$
\end{enumerate} 

Any contraction outcome is required to be a logic program by ($\dotminus$1) and a subset of the initial program by ($\dotminus$2). According to ($\dotminus$3), if the beliefs to be contracted are not implied by the initial belief state, then nothing is to be retracted. ($\dotminus$4) requires that the beliefs to be contracted are not implied by the contracted belief state, unless they are tautologies. ($\dotminus$5) states that all parts of the initial program $P$ that are discarded in a contraction by $Q$ can be recovered by a subsequent expansion with $Q$. ($\dotminus$6) ensures that contraction by strongly equivalent programs leads to the same outcomes. ($\dotminus$7) demands that any parts retained in both $P \dotminus Q$ and $P \dotminus R$ are also retained in $P \dotminus (Q + R)$. Whenever $Q$ is not implied by the result of a contraction by $Q + R$, then ($\dotminus$8) states that this result is also retained in a contraction by $Q$ alone.

($\dotminus$1), ($\dotminus$2), ($\dotminus$5), ($\dotminus$6), and ($\dotminus$7) are direct translations of ($\ominus$1), ($\ominus$2), ($\ominus$5), ($\ominus$6), and ($\ominus$7), respectively. In the adaptation of ($\ominus$3) to ($\dotminus$3), we use \lq\lq $P \not \models_s Q$\rq\rq\ instead of \lq\lq $Q \nsubseteq P$\rq\rq . Belief sets are closed under logical consequence, which means $K \vdash \phi$ iff $\phi \in K$, so that either can be used as the condition in ($\ominus$3). However, for logic programs under SE semantics, we can conclude $P \models_s Q$ from $Q \subseteq P$ but not $Q \subseteq P$ from $P \models_s Q$. Therefore, it is more appropriate to use \lq\lq $P \not \models_s Q$\rq\rq\ than \lq\lq $Q \nsubseteq P$\rq\rq , as $P \not \models_s Q$ implies $Q \nsubseteq P$. For the same reason is \lq\lq $P \dotminus Q \not \models_s Q$\rq\rq\ used instead of \lq\lq $Q \nsubseteq P \dotminus Q$\rq\rq\ in ($\dotminus$4), and \lq\lq $P \dotminus (Q + R) \not \models_s Q$\rq\rq\ instead of \lq\lq $Q \nsubseteq P \dotminus (Q + R)$\rq\rq\ in ($\dotminus$8).

We now adapt the belief base revision postulates ($\divideontimes$1)--($\divideontimes$5) to logic programs, where again $P,Q,R \in \lpa$ and a revision operator $*$ is a function from $\lpa \times \lpa$ to $\lpa$.

\begin{enumerate}[({$*$}1b)]
\item $Q \subseteq P * Q$
\item $P * Q \subseteq P + Q$
\item If $r \in P \setminus (P * Q)$, then there exists a program $P'$ such that $P * Q \subseteq P' \subset P + Q$ and $P'$ is satisfiable but $P' \cup \{r\}$ is not satisfiable
\item If it holds for all $P' \subseteq P$ that $P' + Q$ is satisfiable iff $P' + R$ is satisfiable, then $P \cap (P * Q) = P \cap (P * R)$
\item If $Q$ is satisfiable, then $P * Q$ is satisfiable
\end{enumerate}

All five postulates are direct translations of ($\divideontimes$1)--($\divideontimes$5). We have ($*$1b) $= $ ($*$2) and ($*$2b) $=$ ($*$3). ($*$3b) requires a rule $r$ to be removed during a revision of $P$ by $Q$ if $r$ contributes to making $P$ irreconcilable with $Q$. ($*$4b) states that if all subsets of $P$ that agree with $Q$ also agree with $R$, then the same elements of $P$ will be retained in a revision by $Q$ as in a revision by $R$. ($*$5b) guarantees satisfiability of the revision result whenever the revising program itself is satisfiable. ($*$4b) is a stronger version of ($*$6) and ($*$5b) is a weaker version of ($*$5).

Finally, we translate the belief base contraction postulates ($-$1)--($-$8) to logic programs, where $P,Q,R \in \lpa$ and a contraction operator $\dotminus$ is a function from $\lpa \times \lpa$ to $\lpa$. 

\begin{enumerate}[($\dotminus$1b)]
\item $P \dotminus Q \subseteq P$
\item If $\not \models_s Q$, then $P \dotminus Q \not \models_s Q$
\item If $r \in P \setminus (P \dotminus Q)$, then there exists a program $P'$ such that $P \dotminus Q \subseteq P' \subset P$ and $P' \not \models_s Q$ but $P' \cup \{r\} \models_s Q$
\item If it holds for all $P' \subseteq P$ that $P' \not \models_s Q$ iff $P' \not \models_s R$, then $P \dotminus Q = P \dotminus R$
\item If $P \not \models_s Q$, then $P \dotminus Q = P$
\item If $Q \equiv_s R$, then $P \dotminus Q = P \dotminus R$
\item $P \dotminus (Q + R) = P \dotminus Q$ or $P \dotminus (Q + R) = P \dotminus R$ or $P \dotminus (Q + R) = (P \dotminus Q) \cap (P \dotminus R)$
\item If $r \in P \setminus (P \dotminus Q)$, then $SE(P \dotminus Q) \nsubseteq SE(Q) \cup SE(r)$
\end{enumerate}

These postulates are direct translations of ($-$1)--($-$8). We have ($\dotminus$1b) $=$ ($\dotminus$2), ($\dotminus$2b) $=$ ($\dotminus$4), ($\dotminus$5b) $=$ ($\dotminus$3), and ($\dotminus$6b) $=$ ($\dotminus$6). ($\dotminus$3b) requires a rule $r$ to only be removed during the contraction of $P$ by $Q$ if $r$ somehow contributes to implying $Q$. ($\dotminus$4b) states that if exactly those subsets of $P$ that do not imply $Q$ also do not imply $R$, then the same elements of $P$ will be retained in a contraction by $Q$ as in a contraction by $R$. As for revision,  ($\dotminus$4b) is again a stronger version of  ($\dotminus$6). ($\dotminus$7b) specifies that a contraction by two programs is the outcome of contracting the first program, the outcome of contracting the second program, or the intersection of these two outcomes. ($\dotminus$8b) stipulates that $r$ should only be removed if the contracted program has at least one SE~model that is not an SE~model of $Q$ or $r$.

%We further adapt the Relevance and the Disjunctive Elimination postulates to logic programs as follows:
%Both ($\dotminus$5r) and ($\dotminus$5de) express that a contraction should occur with minimal change to the initial belief state.

While the postulates~($\ominus$5r) and~($\ominus$5de) are equivalent in the presence of certain other postulates in the AGM framework for propositional logic, the relationship between the postulates ($\dotminus$3b) and ($\dotminus$8b) for logic programs is hierarchical, similar as in the belief base framework. Satisfaction of ($\dotminus$3b) implies satisfaction of ($\dotminus$8b) but not vice versa. 

\begin{prop} \label{prop:relevance-to-disjelim}
Let $\dotminus$ be a contraction operator on $\lpa$. If $\dotminus$ satisfies ($\dotminus$3b), then it satisfies ($\dotminus$8b).
\end{prop}

The other direction does not hold as evidenced later by Example~\ref{ex:cone-not-sat-relevance}.

The Levi identity \cite{gardenfors1981epistemic,levi1977subjunctives} allows us to construct a revision from a contraction and the Harper identity \cite{harper1976rational} a contraction from a revision. In their propositional form, they use the classical negation of a sentence $\alpha$ to construct a revision/contraction by $\alpha$. For logic programs, we do not have the classical negation of a program at at our disposal. However, we can abstractly represent the Levi and Harper identities for logic programs as follows, where $\overline{Q} \in \lpa$ iff $Q \in \lpa$ and $SE(\overline{Q}) = \seqc$. 
\begin{dfn} \label{dfn:levi-harper}
Let $P,Q \in \lpa$. If $*$ is a revision operator for $P$ and $\dotminus$ a contraction operator for $P$, then
\begin{align*}
P * Q = (P \dotminus \overline{Q}) + Q \quad &\text{ (Levi identity)}, \\
P \dotminus Q = P \cap (P * \overline{Q}) \quad &\text{ (Harper identity)}.
\end{align*}
\end{dfn}

We adapted the AGM and belief base postulates to logic programs in such a way that the revising or contracting beliefs are entire programs, not just individual rules. We have thus essentially defined conditions for what is known in classic belief change as choice contraction \cite{fuhrmann1994survey} and revision. In a choice contraction, a set of sentences may be contracted by either a single sentence or a set of sentences, for example, $K \ominus \{\phi, \psi\}$. Yet, since adapting the postulates that exactly characterise partial meet choice contractions would yield the postulates ($\dotminus$1b)--($\dotminus$3b) and ($\dotminus$6b), we do not have to pursue this distinction here further.

\section{Partial Meet Belief Change} \label{sec:pm}
Inspired by our motivation, we begin in this section with defining belief change operators for logic programs that preserve syntactic information, in order to take into account information on the program-level as well as on the rule-level. We first adapt the idea of a partial meet construction from classic belief change and formulate revision and contraction operators for logic programs. We utilise the translations of the AGM and belief base revision and contraction postulates from the previous section to test the rationality of our operators. %In addition, we show in detail how our revision operator addresses the limitations of the distance-based revision operator. 

\subsection{Partial Meet Revision} \label{sec:revpm}
As the basis for our construction of partial meet revision, we define a \emph{compatible set} of some program with respect to another program as the dual of a remainder set \cite{alchourron1985logic}.

\begin{dfn}[Compatible Set] \label{dfn:compatible-set}
Let $P,Q \in \lpa$. The set of \emph{compatible sets of $P$ with respect to $Q$} is
\begin{align*}
\pq = \{\, R \subseteq P  \mid \, & SE(R) \cap SE(Q) \neq \es \text{ and, for all } R', \\
& R \subset R' \subseteq P \text{ implies } SE(R') \cap SE(Q) = \es \,\}.
\end{align*}
\end{dfn}

Each compatible set is a maximal subset of $P$ that is consistent with~$Q$ under SE semantics. Each is thus a candidate to be returned together with~$Q$ as the outcome of a revision. To determine exactly which candidate(s) to choose, we employ a \emph{selection function}. In the classic case, a selection function for a belief set is defined only over a set of remainder sets of that belief set (see Section~\ref{sec:bc}). Since we plan to use our selection function for different types of sets, we define it freely with respect to an arbitrary set as follows.

\begin{dfn}[Selection Function] \label{dfn:selfun}
A \emph{selection function} $\gamma$ for a set $S$ is a function such that:
\begin{enumerate}
\item $\mathbb{S} \subseteq 2^S$,
\item $\gamma(\mathbb{S}) \subseteq \mathbb{S}$, and
\item if $\mathbb{S} \neq \es$, then $\gamma(\mathbb{S}) \neq \es$.
\end{enumerate}
\end{dfn}

A special case of our selection function is the following \emph{single-choice selection function}, which restricts the selection function~$\gamma$ to select at most one element of a set.

\begin{dfn}[Single-Choice Selection Function] \label{dfn:single-choice-sel-fun}
A \emph{single-choice selection function} $\gamma^1$ for a set $S$ is a function such that:
\begin{enumerate}[(1)]
\item $\mathbb{S} \subseteq 2^S$, 
\item $\gamma^1(\mathbb{S}) = R$ for some $R \in \mathbb{S}$, and 
\item if $\mathbb{S} \neq \es$, then $\gamma^1(\mathbb{S}) \neq \es$.
\end{enumerate}
\end{dfn}

We can now define \emph{partial meet revision} for logic programs as the intersection of the selected compatible sets added to $Q$.

\begin{dfn}[Partial Meet Revision] \label{dfn:revpm}
Let $P \in \lpa$ and $\gamma$ be a selection function for $P$. A \emph{partial meet revision operator} $\revpm$ for $P$ is defined such that for any $Q \in \lpa$:
$$
P \revpm Q =
\begin{cases}
\; P + Q &\text{ if } Q \text{ is not satisfiable,} \\
\; \bigcap \gpq + Q &\text{ otherwise.}
\end{cases}
$$
\end{dfn}

We illustrate the revision operation by an example. In the examples throughout this paper, we assume that the underlying language contains only the symbols that occur in the programs or that are otherwise mentioned explicitly.

\begin{ex} \label{ex:rev:a.b-if-a.*bot-if-a.}
Let $P = \{\, a.,\, b \leftarrow a. \,\}$ and $Q = \{\, \bot \leftarrow a.\,\}$. We have $\pq = \{\, \{\, b \leftarrow a. \,\} \,\} = \gpq $, for any selection function $\gamma$, and thus $P \revpm Q =  \{\, b \leftarrow a.,\, \bot \leftarrow a.\,\}$.
\qedex
\end{ex} 

The following theorem states which translated AGM revision postulates the revision operator~$\revpm$ satisfies.

\begin{thm} \label{thm:rev-pm}
The revision operator $\revpm$ satisfies ($*$1)--($*$6).
\end{thm}

As in the classic case, to satisfy the supplementary revision postulates ($*$7)--($*$8) we would need to place further restrictions on the selection function.

\begin{dfn}[Relational Selection Function] \label{dfn:mtr-sel-func}
%Let $P,Q \in \lpa$. 
A \emph{transitively relational selection function} $\gamma'$ for a set~$S$ is a selection function for $S$ such that $\mathbb{S} \subseteq 2^S$ and:
$$
\gamma'(\mathbb{S}) = 
\begin{cases}
\; \es &\text{ if } \mathbb{S} = \es, \\
\; \{\, R \in \mathbb{S} \mid R' \unlhd R \text{ for all } R' \in \mathbb{S} \,\} &\text{ otherwise,}
\end{cases}
$$
where $\unlhd$ is a transitive relation over $2^S$. The relation $\unlhd$ is \emph{maximised} iff $R \subset R'$ implies $R \lhd R'$ for all $R,R' \in 2^S$.
\end{dfn}

Yet, even if we make these additional restrictions on a selection function~$\gamma'$ for $P$, postulates~($*$7) and~($*$8) are not satisfied by~$*_{\gamma'}$, as shown respectively in the next two examples. For legibility, we confine the examples to abstract notation. Corresponding canonical logic programs can be constructed via the method provided by \citeN{eiter2013model}.

\begin{ex}
Let $P = \{r_1,r_2,r_3\}$ with $SE(r_1) = \{B,C\}$, $SE(r_2) = \{A,C\}$, $SE(r_3) = \{A,B,C\}$, and $\{r_1,r_3\} \unlhd \{r_2,r_3\} \unlhd \{r_1,r_3\}$. If $SE(Q) = \{A,B\}$ and $SE(Q+R) = \{A\}$, then $\pq = \{ \{r_1,r_3\}, \{r_2,r_3\} \}$ and thus $\bigcap \gamma'(\pq) = \{r_3\}$. On the other hand, we have $\pqr = \{ \{r_2,r_3\} \} = \bigcap \gamma'(\pqr)$. This means $(P *_{\gamma'} Q) + R = \{r_3\} \cup Q \cup R$ while $P *_{\gamma'} (Q + R) = \{r_2,r_3\} \cup Q \cup R$. 
\qedex
\end{ex} 

\begin{ex} \label{ex:revpm-8-counter}
Let $P = \{r_1,r_2,r_3,r_4,r_5\}$ with $SE(r_1) = \{A,B,C,D,E\}$, $SE(r_2) = \{A,B,C,E\}$, $SE(r_3) = \{A,E\}$, $SE(r_4) = \{B,E\}$, $SE(r_5) = \{D,E\}$, and $\{r_1,r_2\} \lhd \{r_1,r_5\} \lhd \{r_1,r_2,r_3\}, \{r_1,r_2,r_4\}$. If $SE(Q) = \{A,B,C,D\}$ and $SE(Q+R) = \{C,D\}$, then $\pq = \{ \{r_1,r_2,r_3\}, \{r_1,r_2,r_4\}, \{r_1,r_5\} \}$ and $\pqr = \{ \{r_1,r_2\}, \{r_1,r_5\} \}$. It follows from  $\gamma'(\pq) = \{ \{r_1,r_2,r_3\}, \{r_1,r_2,r_4\} \}$ that $\bigcap \gamma'(\pq) = \{r_1,r_2\}$, while $\gamma' (\pqr) = \{ \{r_1,r_5\} \} = \bigcap \gamma'(\pqr)$. Therefore, $(P *_{\gamma'} Q) + R = \{r_1,r_2\} \cup Q \cup R \nsubseteq \{r_1,r_5\} \cup Q \cup R = P *_{\gamma'} (Q + R)$. Note that $SE((P *_{\gamma'} Q) + R) = SE(\bigcap \gamma'(\pq)) \cap SE(Q + R) = \{A,B,C,E\} \cap \{C,D\} = \{C\} \neq \es$.
\qedex
\end{ex} 

Our partial meet revision operator does not satisfy the entire set of AGM revision postulates, only the subset of basic postulates ($*$1)--($*$6). This stands in contrast to the result in classical logics, according to which any partial meet revision operator is characterised by the set of basic and supplementary revision postulates. However, that result holds for logically closed belief sets, and when we consider our partial meet revision operator $\revpm$ in the light of the belief base framework, we obtain the following representation theorem.

%\begin{lem} \label{lem:PQ-shared-rules-in-result}
%Let $P,Q \in \lpa$ and $\gamma$ be a selection function for $P$. If $r \in P \cap Q$, then $r \in \bigcap \gpq$.
%\end{lem}

\begin{thm} \label{thm:rev-pm-bb}
An operator $\revpm$ is a partial meet revision operator for $P \in \lpa$ determined by a selection function $\gamma$ for $P$ iff $\revpm$ satisfies ($*$1b)--($*$5b).
\end{thm}

Returning to our motivation, we now give some formal examples to highlight the drawbacks of the distance-based revision operator and then show how our partial meet operator addresses these. Recall that the distance-based revision operator~$\star$ does not specify the structure of the revised program. For convenience, we provide a possible program that corresponds to the revision outcome for each example below, which we denote by $P \star Q$.

\begin{tabular}{lll}
1) & $P = \{\, a.,\, b \leftarrow a.\,\}$ & $P \star Q = \{\, \bot \leftarrow a.,\, b.\,\}$  \\
& $Q = \{\, \bot \leftarrow a.\,\}$  & $SE(P \star Q) = \{(b,b)\}$ \\[6pt]
2) & $P = \{\, \bot \leftarrow a.,\, b \leftarrow not\, a.\,\} $ & $ P \star Q = \{\, a.,\, b.\,\}$\\
& $ Q = \{\, a.\,\}$  & $SE(P \star Q) = \{(ab,ab)\} $ \\[6pt]
3) & $P = \{\, \bot \leftarrow a.,\, b \leftarrow a.\,\}$ & $P \star Q = \{\, a.\,\}$\\
& $Q = \{\, a.\,\}$  & $SE(P \star Q) = \{(a,a),(a,ab),(ab,ab)\}$ \\[6pt]
4) & $P = \{\, a.,\, b \leftarrow not\, a.\,\}$ & $P \star Q = \{\, \bot \leftarrow a.\,\}$ \\
& $Q = \{\, \bot \leftarrow a.\,\}$ & $SE(P \star Q) = \{(\emptyset,\emptyset),(\emptyset,b),(b,b)\}$ \\[6pt]
5) & $P = \{\, a.,\, b \leftarrow not\, c.\,\}$ & $P \star Q = \{\, a.,\, b.,\, \bot \leftarrow c.\,\}$\\
& $Q = \{\, \bot \leftarrow c.\,\}$  & $SE(P \star Q) = \{(ab,ab)\}$
\end{tabular}
%\begin{enumerate}[1)]
%\item $P = \{\, a.,\, b \leftarrow a.\,\}$ \quad $Q = \{\, \bot \leftarrow a.\,\}$\\
%$SE(P \star Q) = \{(b,b)\}$ \quad $P \star Q = \{\, \bot \leftarrow a.,\, b.\,\}$
%\item $P = \{\, \bot \leftarrow a.,\, b \leftarrow not\, a.\,\} $ \quad $ Q = \{\, a.\,\}$\\
%$SE(P \star Q) = \{(ab,ab)\} $ \quad $ P \star Q = \{\, a.,\, b.\,\}$
%\item $P = \{\, \bot \leftarrow a.,\, b \leftarrow a.\,\}$ \quad $Q = \{\, a.\,\}$\\
%$SE(P \star Q) = \{(a,a),(a,ab),(ab,ab)\}$ \quad $P \star Q = \{\, a.\,\}$
%\item $P = \{\, a.,\, b \leftarrow not\, a.\,\}$ \quad $Q = \{\, \bot \leftarrow a.\,\}$\\
%$SE(P \star Q) = \{(\emptyset,\emptyset),(\emptyset,b),(b,b)\}$ \quad $P \star Q = \{\, \bot \leftarrow a.\,\}$
%\item $P = \{\, a.,\, b \leftarrow not\, c.\,\}$ \quad $Q = \{\, \bot \leftarrow c.\,\}$\\
%$SE(P \star Q) = \{(ab,ab)\}$ \quad $P \star Q = \{\, a.,\, b.,\, \bot \leftarrow c.\,\}$
%\end{enumerate}

Examples~1) and~2) demonstrate that the revision operator~$\star$ does not satisfy the support property. In Example~1), the initial belief state expressed by program $P$ consists of~$a$ and $b$. In fact, the second rule in $P$ says that we believe $b$ \emph{if} we believe~$a$. After revising by the program $Q$, which simply states that we do not believe~$a$, we still believe $b$ even though the reason to believe $b$ is not given any more. The explanation for this is that the revision operator acts on a program-level, not on a rule-level, as it considers just the SE~models of the program in its entirety. However, the dependency of $b$ on $a$ is not captured by the SE~models of the program, only by the SE~models of the second rule. Therefore, $b$ is treated as an independent fact during the revision process. The situation is similar in Example~2). Here, we initially believe $b$ due to the absence of belief~$a$. After the revision, we continue to believe $b$ even though the grounds for $b$ do not exist any more.

Examples~3) and~4) demonstrate that the revision operator~$\star$ does not satisfy the preservation property. In Example~3), according to our initial belief state expressed by program~$P$, we believe $b$ whenever we believe $a$, but since we do not believe $a$ currently, we are indifferent with respect to $b$ at the moment. After revising by program $Q$, which states that we now believe $a$, we are still indifferent with respect to $b$. In Example~4), we initially believe $a$ and would believe $b$ whenever we do not believe $a$ or are indifferent with respect to $a$. After revising by $Q$, which carries the information that we do not believe $a$ any longer, we are still indifferent with respect to $b$. In both examples, the revision operation effectively disregards the second rule in~$P$. However, there is no justification for such behaviour, as the information in $Q$ only conflicts with the first rule in $P$, so that the second rule in $P$ can safely be retained and thus $b$ should be derived in both examples. This behaviour is due to the fact that the set of SE~models of $P$ is exactly the set of SE~models of the first rule. The second rule is invisible in the program-level view. Consequently, the revision operator returns a result as if $P$ had consisted merely of the first rule.

Finally, Example~4) stands in stark contrast to Example~5), where the revision operation coincides with expansion. In Example 5), we are indifferent with respect to $b$ initially and should believe $b$ when we do not believe $c$ or are indifferent with respect to $c$. The revising program $Q$ contains information that $c$ indeed does not hold. Thus,~$b$ is incorporated into the new belief state. Yet, Example~4) described a similar scenario in which $b$ is not included in the resulting belief state, thereby showing a clear discrepancy to the behaviour of the revision operator in Example~5). Since the set of SE~models of the first rule of $P$ in Example~5) is not a subset of the set of SE~models of the second rule, the latter is preserved during the revision operation. It becomes apparent from this comparison that some dependencies between atoms expressed in $P$ are respected by the revision operator~$\star$, while others are not. 

We now show that our partial meet revision operator~$\revpm$ addresses these shortcomings of the distance-based revision method. Below are the results of our partial meet revision operator for the five examples above. In each example, the result is independent of the choice of selection function. 

\begin{tabular}{ll}
1) &  $P \revpm Q = \{\, \bot \leftarrow a.,\, b \leftarrow a.\,\}$
\\
& $SE(P \revpm Q) = \{(\es,\es),(\es,b),(b,b)\}$ \\[6pt]
2) &  $ P \revpm Q = \{\, a.,\, b \leftarrow not\, a.\,\}$\\
& $SE(P \revpm Q) = \{(a,a),(a,ab),(ab,ab)\} $ \\[6pt]
3) & $P \revpm Q = \{\, a.,\, b \leftarrow a.\,\}$ \\
& $SE(P \revpm Q) = \{(ab,ab)\}$ \\[6pt]
4) &  $P \revpm Q = \{\, \bot \leftarrow a.,\, b \leftarrow not\, a.\,\}$ \\
& $SE(P \revpm Q) = \{(b,b)\}$  \\[6pt]
5) & $P \revpm Q = \{\, a.,\, b \leftarrow not\, c.,\, \bot \leftarrow c.\,\}$ \\
& $SE(P \revpm Q) = \{(ab,ab)\}$ 
\end{tabular}
%\begin{enumerate}[1)]
%\item $SE(P \revpm Q) = \{(\es,\es),(\es,b),(b,b)\}$ \quad
 %$P \revpm Q = \{\, \bot \leftarrow a.,\, b \leftarrow a.\,\}$
%\item $SE(P \revpm Q) = \{(a,a),(a,ab),(ab,ab)\} $ \quad
 %$ P \revpm Q = \{\, a.,\, b \leftarrow not\, a.\,\}$
%\item $SE(P \revpm Q) = \{(ab,ab)\}$ \quad $P \revpm Q = \{\, a.,\, b \leftarrow a.\,\}$
%\item $SE(P \revpm Q) = \{(b,b)\}$  \quad
% $P \revpm Q = \{\, \bot \leftarrow a.,\, b \leftarrow not\, a.\,\}$
%\item $SE(P \revpm Q) = \{(ab,ab)\}$  \quad
 %$P \revpm Q = \{\, a.,\, b \leftarrow not\, c.,\, \bot \leftarrow c.\,\}$
%\end{enumerate}

In Examples~1) and~2), the partial meet revision operator preserves the dependency of $b$ on~$a$ and $not\, a$, respectively. This is expressed on the syntactic level by the revised program $P \revpm Q$ and on the semantic level by $SE(P \revpm Q)$. In Examples~3) and~4), our partial meet revision operator takes into account all rules in a program, even those that may be \lq\lq invisible\rq\rq\ from a purely model-based perspective, as shown by the respective revision outcomes. Finally, regarding Examples~4) and~5), our partial meet revision operator treats the dependency of $b$ on $not\, a$ and $not\, c$, respectively, in the same manner and adds $b$ to the belief state uniformly in both examples. The reason for this behaviour is that our partial meet construction enables us to preserve information expressed on the rule-level by the individual rules in a program.

Revisiting our two real-life examples from the introduction, we can see that our partial meet revision operator returns the desired results. The first example (\lq\lq fireworks\rq\rq ) was formalised as Example~3) above, where $a = fog$ and $b = no\_fireworks$. Applying the revision operator~$\revpm$ returns $\{\, fog.,\, no\_fireworks \leftarrow fog. \,\}$, that is, it leaves us with the beliefs that it will be foggy and that there are no fireworks whenever it is foggy, from which we can derive that the fireworks will be cancelled. %This example further demonstrates that our operators do not \emph{harmfully eliminate} \cite{inoue2004equivalence} rules as do revision operators that consider only the set of SE~models of a program. 
The second example (\lq\lq 101 or 280\rq\rq ) was formalised as Example~1), where $a = 101\_roadworks$ and $b = 280\_quicker$. Applying the revision operator~$\revpm$ to this example returns $\{\, \bot \leftarrow 101\_roadworks.,\, 280\_quicker \leftarrow 101\_roadworks. \,\}$, that is, we believe that there are no roadworks on the 101 and that the 280 is quicker whenever there are roadworks on the 101, but not that the 280 is still the better choice. %This example also serves to validate that our operators respect the property of \emph{support} \cite{slota2013rise} that is lacking in revision approaches that operate purely on the program-level.

Regarding the screened semi-revision approach by \citeN{krumpelmann2012belief} (see Section~\ref{sec:rel-work}), we can show that our partial meet revision operator~$\revpm$ is a generalisation of their screened consolidation operator $\scr$. For any $P,Q \in \lpa$ with $Q \subseteq P$, let
\begin{align*}
P \botse Q = \{\, R \mid \, &Q \subseteq R \subseteq P, SE(R) \neq \es, \text{ and, for all } R', \\
&R \subset R' \subseteq P \text{ implies } SE(R') = \es \,\},
\end{align*}
and $P \scrse Q = \gamma_P(P \botse Q)$.

\begin{prop} \label{prop:prop:revpm-same-as-consolidation-SE}
Let $P,Q \in \lpa$. For any maxichoice selection function $\gamma_P$ for $P$, there exists a selection function $\gamma$ for $P$ such that $(P \cup Q) \scrse Q = P \revpm Q$.
\end{prop}

Conversely, if we translate our partial meet revision operator $\revpm$ to answer set semantics and restrict our selection function to be single-choice, then it will coincide with the screened consolidation operator $\scr$. Before we can do so, we need to translate our definition of compatible sets to answer set semantics. For any $P,Q \in \lpa$, let
\begin{align*}
\pqas = \{\, R \subseteq P  \mid \, &AS(R \cup Q) \neq \es \text{ and, for all } R', \\
& R \subset R' \subseteq P \text{ implies } AS(R' \cup Q) = \es \,\}.
\end{align*}

\begin{dfn}[Partial Meet Revision under Answer Set Semantics] \label{dfn:revpmas}
Let $P \in \lpa$ and $\gamma^1$ be a single-choice selection function for $P$. A partial meet revision operator $\revpmas$ for $P$ under answer set semantics is defined such that for any $Q \in \lpa$:
$$
P \revpmas Q =
\begin{cases}
\; P \cup Q &\text{ if } AS(Q) = \es \text{ and } \pqas = \es , \\
\; \gamma^1(\pqas) \cup Q &\text{ otherwise.}
\end{cases}
$$
\end{dfn}

\begin{prop} \label{prop:revpm-same-as-consolidation-AS}
Let $P,Q \in \lpa$ and $\gamma^1$, $\gamma_P$ be single-choice and maxichoice selection functions, respectively, for $P$ such that $\gamma^1(2^P) \cup Q = \gamma_P(\{\, R \cup Q \mid R \in 2^P \,\})$. Then $P \revpmas Q = (P \cup Q) \scr Q$.
\end{prop}

\subsection{Partial Meet Contraction} \label{sec:conpm}
Having defined a revision operator, we now turn to the case of belief contraction. %Again, we focus on individual rules of a program and their models as the basis of our construction. 
In line with classic belief change, the contraction of a program $P$ by a program~$Q$ should eliminate from $P$ all those beliefs from which $Q$ can be derived. We use the complement of $SE(Q)$, denoted by $\seqc$, to determine all maximal subsets of $P$ that do not imply~$Q$, called \emph{remainder sets}.

\begin{dfn}[Remainder Set] \label{dfn:remainder-set}
Let $P,Q \in \lpa$. The set of \emph{remainder sets of~$P$ with respect to $Q$} is
\begin{align*}
\pqm = \{\, R \subseteq P  \mid\, & SE(R) \cap \seqc \neq \es \text{ and, for all } R',\\
&R \subset R' \subseteq P \text{ implies } SE(R') \cap \seqc = \es \,\}.
\end{align*}
\end{dfn}

\begin{dfn}[Partial Meet Contraction] \label{dfn:conpm}
Let $P \in \lpa$ and $\gamma$ be a selection function for $P$. A \emph{partial meet contraction operator} $\conpm$ for $P$ is defined such that for any $Q \in \lpa$:
$$
P \conpm Q = 
\begin{cases}
\; P & \text{ if } \models_s Q, \\
\; \bigcap \gpqm & \text{ otherwise.}
\end{cases}
$$
\end{dfn}

The following example demonstrates the contraction operation.

\begin{ex} \label{ex:conpm}
Let $P = \{\, a.,\, b \leftarrow a. \,\}$ and $Q = \{\, a \leftarrow b.\,\}$. Since $\seqc = \{(\es,b),(b,b),(b,ab)\}$, $SE(\{\, a. \,\}) = \{(a,a),(a,ab),(ab,ab)\}$, and $SE(\{\, b \leftarrow a. \,\}) = \{(\es,\es),(\es,b),(b,b),(\es,ab),(b,ab),(ab,ab)\}$, it follows that $\pqm = \{\, \{\, b \leftarrow a. \,\} \,\} = \gpqm$, for any selection function $\gamma$, and thus we obtain $P \conpm Q = \{\, b \leftarrow a. \,\}$.
\qedex
\end{ex}

The next theorem lists the translated AGM contraction postulates that are fulfilled by~$\conpm$.

\begin{thm} \label{thm:con-pm}
The contraction operator $\conpm$ satisfies ($\dotminus$1)--($\dotminus$4) and ($\dotminus$6).
\end{thm}

It is easy to see from Example~\ref{ex:conpm} above that~$\conpm$ does indeed not satisfy ($\dotminus$5).

In the AGM framework, it is sufficient that the selection function~$\gamma$ is determined by a transitive relation so that~$\ominus$ satisfies~($\ominus$7). Here, we require the relation~$\unlhd$ to be maximised as well to guarantee satisfaction of~($\dotminus$7). This is in line with the result for partial meet base contractions, which also require the underlying selection function to be determined by a maximised transitive relation in order to satisfy such a property \cite{hansson1993reversing}.
%By restricting a selection function $\gamma'$ to be determined by a maximised transitive relation, we can show that $\dotminus_{\gamma'}$ satisfies~($\dotminus$7).

%\begin{lem} \label{lem:pqrm-subset-pq-union-pr}
%Let $P \in \lpa$. For any $Q,R \in \lpa$, it holds that $\pqrm \subseteq \pqm \cup \prm$.
%\end{lem}
%
%\begin{lem} \label{lem:gammapqrm-subset-gammapq-union-gammapr}
%Let $P \in \lpa$ and $\gamma'$ be determined by a maximised transitive relation. For any $Q,R \in \lpa$, it holds that $\gamma'(\pqrm) \subseteq \gamma'(\pqm) \cup \gamma'(\prm)$.
%\end{lem}

\begin{thm} \label{thm:conpm-sat-7}
Let $\gamma'$ be determined by a maximised transitive relation. The contraction operator $\dotminus_{\gamma'}$ satisfies ($\dotminus$7).
\end{thm}

The following example demonstrates that $\dotminus_{\gamma'}$ does not satisfy ($\dotminus$8).

\begin{ex}
Consider again $P$ and $\unlhd$ from Example~\ref{ex:revpm-8-counter}. If $SE(Q) = \{A,B,E\}$ and $SE(Q+R) = \{E\}$, then $\seqc = \{C,D\}$ and $\seqrc = \{A,B,C,D\}$. We obtain $\mathbb{P}_{Q+R}^- = \{ \{r_1,r_2,r_3\}, \{r_1,r_2,r_4\}, \{r_1,r_5\} \}$ and $\pqm = \{ \{r_1,r_2\}, \{r_1,r_5\} \}$. It follows from  $\gamma'(\mathbb{P}_{Q+R}^-) = \{ \{r_1,r_2,r_3\}, \{r_1,r_2,r_4\} \}$ that $\bigcap \gamma'(\mathbb{P}_{Q+R}^-) = \{r_1,r_2\}$, while $\gamma' (\pqm) = \{ \{r_1,r_5\} \}$ and thus  $\bigcap \gamma'(\pqm) = \{r_1,r_5\}$. Therefore, $P \dotminus_{\gamma'} (Q + R) = \{r_1,r_2\} \nsubseteq \{r_1,r_5\} = P \dotminus_{\gamma'} Q$. Note that $SE(P \dotminus_{\gamma'} (Q + R)) = SE(\bigcap \gamma'(\pqrm)) = \{A,B,C,E\} \nsubseteq \{A,B,E\} = SE(Q)$.
\qedex
\end{ex} 

As in the case of our revision operator earlier, our partial meet contraction operator does not properly align with the AGM framework, but the following representation theorem holds for the contraction operator $\conpm$ with respect to the belief base postulates.

\begin{thm} \label{thm:con-pm-bb}
An operator $\conpm$ is a partial meet contraction operator for $P \in \lpa$ determined by a selection function $\gamma$ for $P$ iff $\conpm$ satisfies ($\dotminus$1b)--($\dotminus$4b).
\end{thm}

Besides the representation theorem above via ($\dotminus$1b)--($\dotminus$4b), we have the following additional properties of~$\conpm$ regarding the remaining belief base contraction postulates ($\dotminus$5b)--($\dotminus$8b).

\begin{prop} \label{prop:con-pm-bb-5--8}
The contraction operator $\conpm$ satisfies ($\dotminus$5b), ($\dotminus$6b), and ($\dotminus$8b).
\end{prop}

The next example illustrates that $\conpm$ does not satisfy ($\dotminus$7b).

\begin{ex} \label{ex:conpm-not-7b}
Let $P = \{r_1,r_2,r_3,r_4\}$ with $SE(r_1) = \{A,B,C,D,E\}$, $SE(r_2) = \{A,B,E\}$, $SE(r_3) = \{A,C,E\}$, and $SE(r_4) = \{C,D,E\}$. If $SE(Q) = \{A,D,E\}$ and $SE(R) = \{B,C,E\}$, then $\seqc = \{B,C\}$, $\overline{SE(R)} = \{A,D\}$, and $\seqrc = \{A,B,C,D\}$. We thus have $\pqm = \{ \{r_1,r_2\}, \{r_1,r_3,r_4\} \}$, $\prm = \{ \{r_1,r_2,r_3\}, \{r_1,r_4\} \}$, and $\pqrm = \{ \{r_1,r_2,r_3\}, \{r_1,r_3,r_4\} \}$. Let $\gpqm = \pqm$, $\gamma(\prm) = \prm$, and $\gamma(\pqrm) = \pqrm$. It then follows that $\bigcap \gpqm = \{r_1\}$, $\bigcap \gamma(\prm) = \{r_1\}$, and $\bigcap \gamma(\pqrm) = \{r_1,r_3\}$. Therefore, $P \conpm (Q + R) \neq P \conpm Q = P \conpm R = (P \conpm Q) \cap (P \conpm R)$.
\end{ex}

\section{Ensconcement Belief Change} \label{sec:ens}
In this section, we will introduce further belief revision and contraction operators for logic programs, which are based on an ordering over the beliefs contained in a program. We begin by defining an \emph{ensconcement} relation for logic programs as follows.

\begin{dfn}[Ensconcement] \label{dfn:ens}
Let $P \in \lpa$. An \emph{ensconcement associated with} $P$ is any total preorder $\preceq$ on $P$ that satisfies the following conditions:
\begin{enumerate}[($\preceq$1)]
\item For any $r \in P$: $SE(\{\, r' \in P \setminus \{r\} \mid r \preceq r' \,\}) \not \subset SE(r)$\label{cond:ens-1}
\item For any $r,r' \in P$: $r \preceq r' \preceq r$ iff $\{r\} \equiv_s \{r'\}$ \label{cond:ens-2}
\end{enumerate}
\end{dfn}

Per this definition, an ensconcement associated with a logic program $P$ is simply an ordering over the rules occurring in $P$. With $P$ representing our entire set of beliefs, an ensconcement enables us to sort rules of $P$, which form our individual beliefs, hierarchically by their epistemic importance, or in other words, by how willing we are to give up one belief over another. Informally, $r \prec r'$ means that the beliefs represented by $r'$ are more important to us than the beliefs represented by $r$. 

Condition~($\preceq$\ref{cond:ens-1}) states that the set of SE~models of any rule or combination of rules at least as ensconced as a given rule $r$ may not be a proper subset of the set of SE~models of $r$. Condition~($\preceq$\ref{cond:ens-2}) requires that strongly equivalent rules are equally ensconced. Condition~($\preceq$\ref{cond:ens-1}) is formulated slightly stronger than Condition~($\preccurlyeq$1) (see Section~\ref{sec:bc}) from the original definition \cite{williams1994logic}. Condition~($\preccurlyeq$1) allows a sentence $\psi$, that implies a sentence $\phi$ without being equivalent to $\phi$, to be placed on the same ensconcement level as $\phi$. In contrast, Condition~($\preceq$\ref{cond:ens-1}) prohibits strict implication on the same ensconcement level. For instance, given rules $a.$, $a \leftarrow b.$, and $a;b.$ contained in some program, both $a \leftarrow b.$ and $a;b.$ must be strictly more ensconced than $a.$ according to Condition~($\preceq$\ref{cond:ens-1}), whereas in a direct adaptation of Condition~($\preccurlyeq$1) at least one of the two rules $a \leftarrow b.$ and $a;b.$ would have to be equally ensconced as $a.$. The merit of this additional restriction will become evident shortly, when we show some examples of applying an ensconcement to perform revision operations in Section~\ref{sec:reve}. The idea behind Condition~($\preccurlyeq$2) of the original definition is that any tautologies must be most ensconced, a requirement that is automatically captured in our Condition~($\preceq$\ref{cond:ens-1}).

Using the concept of logic program ensconcements, we now go on to define logic program revision and contraction operators and investigate their properties.

\subsection{Ensconcement Revision} \label{sec:reve}
During a revision operation, new information from a program~$Q$ is added to an initial belief state in the form of a program~$P$, and some beliefs from $P$ have to be given up to achieve a consistent outcome. When the beliefs in $P$ are ordered by an ensconcement, we can introduce the notion of a \emph{cut} to determine the specific level in the ensconcement where all beliefs on and above this level are consistent with the revising program. Since an ensconcement $\preceq$ associated with $P$ is a relation over all rules of $P$, when we write $r \preceq r'$, we implicitly mean $r \in P$ and $r' \in P$.

\begin{dfn}[Cut] \label{dfn:cut}
Let $P,Q \in \lpa$ and $\preceq$ be an ensconcement associated with~$P$. The \emph{(proper) cut of $P$ for $Q$}, written~$\cq$, is defined as 
$$
\cq = \left\{\, r \in  P \mid SE\left(\{\, r' \in P \mid r \preceq r' \,\}\right) \cap SE(Q) \neq \es \,\right\}.
$$
\end{dfn}

Some interesting properties of the cut are listed below.

\begin{lem} \label{lem:cut-properties}
Let $P,Q,R \in \lpa$ and $\preceq$ be an ensconcement associated with $P$.
\begin{enumerate}[a)]
\item If $P + Q$ is satisfiable, then $\cq = P$. \label{lem:cq-equal-P}
\item If $Q$ is satisfiable, then $\cq + Q$ is satisfiable. \label{lem:cq-plus-Q-satisfiable}
\item If $Q$ is not satisfiable, then $\cq = \es$. \label{lem:cq-equal-es}
\item If $Q \models_s R$, then $\cq \subseteq cut_\preceq(R)$. \label{lem:q-implies-r-then-cq-subset-cr}
\item $\cqr \subseteq \cq$. \label{lem:cqr-subset-cq}
\item If $\cq \models_s R$, then $\cqr = \cq$.
\end{enumerate}
\end{lem}

A cut is the principal element for the following definition of an ensconcement revision operator.

\begin{dfn}[Ensconcement Revision] \label{dfn:reve}
Let $P \in \lpa$ and $\preceq$ be an ensconcement associated with $P$. An \emph{ensconcement revision operator} $\reve$ for $P$ is defined such that for any $Q \in \lpa$:
$$
P \reve Q =
\begin{cases}
\; P + Q  &\text{ if } Q  \text{ is not satisfiable,} \\
\; \{\, r \in P \mid SE(\cq) \cap SE(Q) \subseteq SE(r) \,\} + Q  &\text{ otherwise.}
\end{cases}
$$
\end{dfn}

The revision operator $\reve$ retains all elements of the cut. This is an obvious requirement since the cut contains our most firmly held beliefs which are entirely consistent with~$Q$. In addition, any rule of $P$ not in the cut that shares the same SE~models with~$Q$ as the cut is retained as well. The example below illustrates the operation.

\begin{ex} \label{ex:reve}
Let $P = \{\, a.,\, a \leftarrow b.,\, b \leftarrow a. \,\}$ and $Q = \{\, \bot \leftarrow b. \,\}$. Figure~\ref{fig:reve-ex} shows all possible ensconcements associated with~$P$, with rules displayed at the top being more ensconced than rules at the bottom. We have the following results:
\begin{enumerate}[1.]
\item $cut_{\preceq_1}(Q) = \{\, a \leftarrow b.,\, b \leftarrow a. \,\}$ and $P *_{\preceq_1} Q = \{\, a \leftarrow b.,\, b \leftarrow a.,\, \bot \leftarrow b. \,\}$
\item $cut_{\preceq_2}(Q) = \{\, a \leftarrow b.,\, b \leftarrow a. \,\}$ and $P *_{\preceq_2} Q = \{\, a \leftarrow b.,\, b \leftarrow a.,\, \bot \leftarrow b. \,\}$
\item $cut_{\preceq_3}(Q) = \{\, a \leftarrow b.,\, b \leftarrow a. \,\}$ and $P *_{\preceq_3} Q = \{\, a \leftarrow b.,\, b \leftarrow a.,\, \bot \leftarrow b. \,\}$
\item $cut_{\preceq_4}(Q) = \{\, a \leftarrow b.\,\}$ and $P *_{\preceq_4} Q = \{\, a \leftarrow b.,\, \bot \leftarrow b. \,\}$
\item $cut_{\preceq_5}(Q) = \{\, a \leftarrow b.,\,  a. \,\}$ and $P *_{\preceq_5} Q = \{\, a \leftarrow b.,\, a.,\, \bot \leftarrow b. \,\}$ 
\qedex
\end{enumerate}
\begin{figure}
\small
\centering
\fbox{
\setlength{\unitlength}{0.75cm}
\begin{picture}(1.4,2.5)
%\thicklines
\put(-.25,2.1){$\{\, b \leftarrow a. \,\}$}
\put(-.375,1.7){\line(1,0){2.05}}
\put(-.25,1.1){$\{\, a \leftarrow b. \,\}$}
\put(-.375,0.7){\line(1,0){2.05}}
\put(-.25,0.1){$\{\, a. \,\}$}
\end{picture}
}
\fbox{
\setlength{\unitlength}{0.75cm}
\begin{picture}(3.3,1.5)
%\thicklines
\put(-.25,1.1){$\{\, a \leftarrow b. \,\} \;\{\, b \leftarrow a. \,\}$}
\put(-.39,0.7){\line(1,0){3.95}}
\put(-.25,0.1){$\{\, a. \,\}$}
\end{picture}
}
\fbox{
\setlength{\unitlength}{0.75cm}
\begin{picture}(1.4,2.5)
%\thicklines
\put(-.25,2.1){$\{\, a \leftarrow b. \,\}$}
\put(-.375,1.7){\line(1,0){2.05}}
\put(-.25,1.1){$\{\, b \leftarrow a. \,\}$}
\put(-.375,0.7){\line(1,0){2.05}}
\put(-.25,0.1){$\{\, a. \,\}$}
\end{picture}
}
\fbox{
\setlength{\unitlength}{0.75cm}
\begin{picture}(2.45,1.5)
%\thicklines
\put(-.25,1.1){$\{\, a \leftarrow b. \,\} $}
\put(-.39,0.7){\line(1,0){3.1}}
\put(-.25,0.1){$\{\, a. \,\} \;\{\, b \leftarrow a. \,\}$}
\end{picture}
}
\fbox{
\setlength{\unitlength}{0.75cm}
\begin{picture}(1.4,2.5)
%\thicklines
\put(-.25,2.1){$\{\, a \leftarrow b. \,\}$}
\put(-.375,1.7){\line(1,0){2.05}}
\put(-.25,1.1){$\{\, a. \,\}$}
\put(-.375,0.7){\line(1,0){2.05}}
\put(-.25,0.1){$\{\, b \leftarrow a. \,\}$}
\end{picture}
}
\\
\begin{picture}(1,10)
\put(-130,0){$\preceq_1$}
\put(-60,0){$\preceq_2$}
\put(3,0){$\preceq_3$}
\put(63,0){$\preceq_4$}
\put(120,0){$\preceq_5$}
\end{picture}
\caption{$\preceq_1$, $\preceq_2$, $\preceq_3$, $\preceq_4$, $\preceq_5$ of Example~\ref{ex:reve}}
\label{fig:reve-ex}
\end{figure}
\end{ex}

In Example~\ref{ex:reve}, the belief expressed by the combination of rules~$\{\, a. \,\}$ and~$\{\, b \leftarrow a. \,\}$ is inconsistent with the new information~$\{\, \bot \leftarrow b.\,\}$. Thus, at least one of these two rules must be discarded to reach a consistent belief state, while the rule~$\{\, a \leftarrow b. \,\}$ can be safely retained. The example shows that the revision operator~$\reve$ indeed retains~$\{\, a \leftarrow b. \,\}$ in all cases and discards one or both other rules depending on their ensconcement level. Whenever~$\{\, b \leftarrow a. \,\}$ is more ensconced than~$\{\, a. \,\}$, the latter is discarded and vice versa. Only when both rules are equally ensconced, that is, when we cannot make up our mind which the two beliefs we hold more firmly, the revision operator discards both.

We can see from the definition of $\reve$ that the set of SE~models of $P \reve Q$ is exactly the set of SE~models that are shared by $\cq$ and $Q$.

\begin{prop} \label{prop:p-rev-q-sequiv-cq-plus-q}
Let $P,Q \in \lpa$ and $\preceq$ be an ensconcement associated with $P$. Then $SE(P \reve Q) = SE(\cq + Q)$.
\end{prop}

The next theorem states which of the adapted AGM revision postulates the revision operator~$\reve$ satisfies.

\begin{thm} \label{thm:rev-ens}
The revision operator $\reve$ satisfies ($*$1)--($*$6) and ($*$8).
\end{thm}

The revision operator $\reve$ does not satisfy ($*$7), as shown in the next example.

\begin{ex}
Let $P = \{r_1,r_2,r_3\}$ with $SE(r_1) = \{B,C\}$, $SE(r_2) = \{A,C\}$, $SE(r_3) = \{A,B,C\}$, and $\preceq$ be an ensconcement associated with $P$ such that $r_1 \preceq r_2 \preceq r_1 \prec r_3$. If $SE(Q) = \{A,B\}$ and $SE(Q+R) = \{A\}$, then $\cq = \cqr = \{r_3\}$, yet $(P \reve Q) + R = \{r_3\} \cup Q \cup R$ while $P \reve (Q + R) = \{r_2,r_3\} \cup Q \cup R$. 
\qedex
\end{ex}

Even though ensconcement revision was originally defined for belief bases, our ensconcement revision operator $\reve$ satisfies the majority of AGM revision postulates for belief sets. Our operator does not satisfy any of the postulates that are unique to the belief base framework, as stated in the next theorem.

\begin{thm} \label{thm:reve-bb}
The revision operator~$\reve$ satisfies ($*$1b), ($*$2b), and ($*$5b).
\end{thm}

In the following two examples, we illustrate that $\reve$ does indeed not satisfy ($*$3b) and ($*$4b), respectively.

\begin{ex} \label{ex:reve-not-sat-relevance}
Let $\se = \{A, B, C, D\}$, $SE(Q) = \{A,B\}$, and $P = \{r_1,r_2,r_3\}$ with $SE(r_1) = \{C, D\}$, $SE(r_2) = \{B, C\}$, and $SE(r_3) = \{A, B, C\}$. If $\preceq$ is an ensconcement associated with $P$ such that $r_1 \preceq r_2 \preceq r_1 \prec r_3$, then $\cq = \{r_3\}$ and $P \reve Q = \{r_3\} + Q$. While $r_2 \in P \setminus (P \reve Q)$, there exists no program $P'$ such that $P \reve Q \subseteq P' \subseteq P + Q$ and $P'$ is satisfiable but $P' \cup \{r_2\}$ is not satisfiable.
\qedex
\end{ex}

\begin{ex} \label{ex:reve-not-sat-uniformity}
Let $\se = \{A, B, C, D\}$, $SE(Q) = \{A,B\}$, $SE(R) = \{A\}$, and $P = \{r_1,r_2,r_3\}$ with $SE(r_1) = \{C, D\}$, $SE(r_2) = \{A, C\}$, and $SE(r_3) = \{A, B, C\}$. If $\preceq$ is an ensconcement associated with $P$ such that $r_1 \preceq r_2 \preceq r_1 \prec r_3$, then it holds for any $P' \subseteq P$ that $P' + Q$ is satisfiable iff $P' + R$ is satisfiable. However, we have $\cq = \{r_3\}$ and $P \cap (P \reve Q) = \{r_3\}$, while $cut_\preceq(R) = \{r_3\}$ and $P \cap (P \reve R) = \{r_2, r_3\}$.
\qedex
\end{ex}

We will now examine the behaviour of our ensconcement revision operator with respect to the set of five examples from Section~\ref{sec:revpm}. In each example, the revision outcome is independent of the possible ensconcements that can be associated with $P$.

\begin{tabular}{ll}
1) &  $P \reve Q = \{\, \bot \leftarrow a.,\, b \leftarrow a.\,\}$ \\
& $SE(P \reve Q) = \{(\es,\es),(\es,b),(b,b)\}$\\[6pt]
2) &  $ P \reve Q = \{\, a.,\, b \leftarrow not\, a.\,\}$ \\
& $SE(P \reve Q) = \{(a,a),(a,ab),(ab,ab)\} $ \\[6pt]
3) & $P \reve Q = \{\, a.,\, b \leftarrow a.\,\}$ \\
& $SE(P \reve Q) = \{(ab,ab)\}$ \\[6pt]
4) &  $P \reve Q = \{\, \bot \leftarrow a.,\, b \leftarrow not\, a.\,\}$ \\
& $SE(P \reve Q) = \{(b,b)\}$  \\[6pt]
5) &  $P \reve Q = \{\, a.,\, b \leftarrow not\, c.,\, \bot \leftarrow c.\,\}$ \\
& $SE(P \reve Q) = \{(ab,ab)\}$ 
\end{tabular}
%\begin{enumerate}[1)]
%\item %$P = \{\, a.,\, b \leftarrow a.\,\}$ \quad $Q = \{\, \bot \leftarrow a.\,\}$\\
%$SE(P \reve Q) = \{(\es,\es),(\es,b),(b,b)\}$ \quad
 %$P \reve Q = \{\, \bot \leftarrow a.,\, b \leftarrow a.\,\}$
%\item %$P = \{\, \bot \leftarrow a.,\, b \leftarrow not\, a.\,\} $ \quad $ Q = \{\, a.\,\}$\\
%$SE(P \reve Q) = \{(a,a),(a,ab),(ab,ab)\} $ \quad
 %$ P \reve Q = \{\, a.,\, b \leftarrow not\, a.\,\}$
%\item %$P = \{\, a.,\, b \leftarrow not\, c.\,\}$ \quad $Q = \{\, \bot \leftarrow c.\,\}$\\
%$SE(P \reve Q) = \{(ab,ab)\}$  \quad
 %$P \reve Q = \{\, a.,\, b \leftarrow not\, c.,\, \bot \leftarrow c.\,\}$
%\item %$P = \{\, a.,\, b \leftarrow not\, a.\,\}$ \quad $Q = \{\, \bot \leftarrow a.\,\}$\\
%$SE(P \reve Q) = \{(b,b)\}$  \quad
% $P \reve Q = \{\, \bot \leftarrow a.,\, b \leftarrow not\, a.\,\}$
%\item %$P = \{\, \bot \leftarrow a.,\, b \leftarrow a.\,\}$ \quad $Q = \{\, a.\,\}$\\
%$SE(P \reve Q) = \{(ab,ab)\}$ \quad $P \reve Q = \{\, a.,\, b \leftarrow a.\,\}$
%\end{enumerate}

For all five examples, the ensconcement revision operator~$\reve$ returns the same desired results as the partial meet revision operator~$\revpm$. Examining in particular Examples~3) and 4), it now becomes evident why we diverted in our formulation of Condition~($\preceq$\ref{cond:ens-1}) from the classic definition. Condition~($\preceq$\ref{cond:ens-1}) prohibits strict implication on the same ensconcement level. Without this refined requirement, for Example~3) we could construct an ensconcement associated with $P$ such that $\bot \leftarrow a. \preceq b \leftarrow  a. \preceq \bot \leftarrow a.$, which would give us the outcome $P \reve Q = \{\, a. \,\}$. For Example~4), we could construct an ensconcement associated with $P$ such that $a. \preceq b \leftarrow not\, a. \preceq a.$, which would lead to the outcome $P \reve Q = \{\, \bot \leftarrow a. \,\}$. These revision outcomes would correspond exactly to the undesired results of the distance-based revision operator~$\star$, which we set out to avoid because they disrespect the preservation property.

\subsection{Ensconcement Contraction} \label{sec:cone}
We now use the concept of an ensconcement to present another contraction operator for logic programs. Analogous to revision, we first define for some $P,Q \in \lpa$ and an ensconcement~$\preceq$ associated with $P$ that 
$$
\cqm = \{\, r \in P \mid SE\left(\{\, r' \in P \mid r \preceq r' \,\}\right) \cap \seqc \neq \es \,\} .
$$

Some useful properties of $\cqm$ are listed here.

\begin{lem} \label{lem:cutm-properties}
Let $P,Q,R \in \lpa$.
\begin{enumerate}[a)]
\item If $P \not \models_s Q$, then $\cqm = P$. \label{lem:cqm-equal-P}
\item If $\not \models_s Q$, then $\cqm \not \models_s Q$. \label{lem:cqm-not-models-Q}
\item If $\models_s Q$, then $\cqm = \es$.
\item If $Q \models_s R$, then $cut_\preceq^-(R) \subseteq \cqm$. \label{lem:q-implies-r-then-crm-subset-cqm}
\item $\cqm \subseteq \cqrm$. \label{lem:cqm-subset-cqrm}
\item If $\cqm \models_s R$, then $\cqrm = \cqm$. \label{lem:cqrm-equal-cqm}
\item If $\cqm \not \models_s R$, then $\cqrm = cut^-_\preceq(R)$. \label{lem:cqrm-equal-crm}
\end{enumerate}
\end{lem}

\begin{dfn}[Ensconcement Contraction] \label{dfn:cone}
Let $P \in \lpa$ and~$\preceq$ be an ensconcement associated with $P$. 
An \emph{ensconcement contraction operator}~$\cone$ for~$P$ is defined such that for any $Q \in \lpa$:
$$
P \cone Q = 
\begin{cases}
\; P  &\text{ if } \models_s Q, \\
\; \{\, r \in P \mid SE(\cqm) \cap \seqc \subseteq SE(r) \,\} &\text{ otherwise.}
\end{cases}
$$
\end{dfn}

The contraction operator $\cone$ works in a dual way to the revision operator $\reve$. It relies on $\cqm$ to determine from which level upward in the ensconcement associated with~$P$ elements are retained in the operation, and adds any further parts of~$P$ that do not compromise the set of SE~models of $\cqm$ inconsistent with $Q$.

We can formalise the relationship between the SE~models of $P \cone Q$ and $\cqm$ as follows.

\begin{prop} \label{prop:SEcon-equal-SEcutminus}
Let $P,Q \in \lpa$ and $\preceq$ be an ensconcement associated with $P$. Then $SE(P \cone Q) \cap \seqc = SE(\cqm) \cap \seqc$.
\end{prop}

The contraction operator~$\cone$ satisfies all AGM contraction postulates except Recovery.

\begin{thm} \label{thm:con-ens}
The contraction operator~$\cone$ satisfies ($\dotminus$1)--($\dotminus$4) and ($\dotminus$6)--($\dotminus$8).
\end{thm}

The next example shows why $\cone$ does not satisfy the Recovery postulate ($\dotminus$5).

\begin{ex}
Consider again $P$ and $Q$ from Example~\ref{ex:conpm}. For any ensconcement~$\preceq$ associated with $P$, it holds that $P \cone Q = \{\, b \leftarrow a. \,\}$. Thus, $P = \{\, a.,\, b \leftarrow a. \,\} \nsubseteq \{\, b \leftarrow a.,\, a \leftarrow b. \,\} = (P \cone Q) + Q$.
\qedex
\end{ex}

The main reason for non-satisfaction of Recovery is that our ensconcement contraction operator $\cone$ operates on programs that are not logically closed, and Recovery is a key AGM postulate that characterises contractions of logically closed belief sets. On the other hand, our ensconcement contraction operator satisfies the same set of belief base postulates as its classic counterpart.

\begin{thm} \label{thm:cone-bb}
The contraction operator $\cone$ satisfies ($-$1b), ($-$2b) and ($-$5b)--($-$8b).
\end{thm}

The following two examples demonstrate that the contraction operator~$\cone$ does not satisfy~($\dotminus$3b) and ($\dotminus$4b), respectively.

\begin{ex} \label{ex:cone-not-sat-relevance}
Let $\se = \{A, B, C, D\}$, $SE(Q) = \{C\}$, and consider again $P$ and $\preceq$ from Example~\ref{ex:reve-not-sat-relevance}. Then $\cqm = \{r_3\} = P \cone Q$. While $r_2 \in P \setminus (P \cone Q)$, there exists no program $P'$ such that $P \cone Q \subseteq P' \subseteq P$ and $P' \not \models_s Q$ but $P' \cup \{r_2\} \models_s Q$.
\qedex
\end{ex}

\begin{ex} \label{ex:cone-not-sat-uniformity}
Let $\se = \{A, B, C, D\}$, $SE(Q) = \{C\}$, $SE(R) = \{B, C\}$, and consider again $P$ and $\preceq$ from Example~\ref{ex:reve-not-sat-uniformity}. Then it holds for any $P' \subseteq P$ that $P' \not \models_s Q$ iff $P' \not \models_s R$. However, we have $\cqm = \{r_3\}$ and $P \cone Q = \{r_3\}$, while $cut_\preceq^-(R) = \{r_3\}$ and $P \cone R = \{r_2, r_3\}$.
\qedex
\end{ex}

\section{Connections between the Operators} \label{sec:conn-btw-operators}
Having defined partial meet revision and contraction operators and ensconcement revision and contraction operators in the previous sections, we now establish the formal connections between them. We first relate partial meet revision to ensconcement revision and partial meet contraction to ensconcement contraction. We then investigate whether the granularity of ensconcements, that is, whether an ensconcement is defined over rules or subsets of a program, influences that relationship. Finally, we connect partial meet revision to partial meet contraction and ensconcement revision to ensconcement contraction via the Levi and Harper identities.

\subsection{Relating Partial Meet Operators to Ensconcement Operators}
We already saw from %the set of five running examples and from 
the set of postulates that~$\reve$ and~$\revpm$ satisfy, that partial meet revision and ensconcement revision share similar properties. In the following characterisation theorem we state the exact relationship between the two.
\begin{thm} \label{thm:revpm-reve}
Let $P,Q \in \lpa$. For any selection function $\gamma$, there exists an ensconcement $\preceq$ associated with $P$ such that $P \revpm Q = P \reve Q$.
\end{thm}

Theorem~\ref{thm:revpm-reve} asserts that $\revpm$ can be characterised in terms of $\reve$. The other direction is not possible, as shown in the example below.

\begin{ex} \label{ex:no-gamma-for-ensc}
Let $P = \{\, a.,\, b.,\, c. \,\}$, $Q = \{\, \bot \leftarrow a. \,\}$, and~$\preceq$ be the ensconcement associated with $P$ as shown in Figure~\ref{fig:preceq-revpm-reve}. It follows that $\cq = \{\, c. \,\}$ and $P \reve Q = \{\, c.,\, \bot \leftarrow a. \,\}$. Yet $\pq = \{\, \{\, b.,\, c. \,\}\,\} = \gpq$, for any selection function $\gamma$, so that $P \revpm Q = \{\, b.,\, c.,\, \bot \leftarrow a. \,\}$. We have $P \reve Q \neq P \revpm Q$.
\begin{figure}
\small
\centering
\fbox{
\setlength{\unitlength}{0.75cm}
\begin{picture}(1.8,1.5)
%\thicklines
\put(-.1,1.1){$\{\, c. \,\}$}
\put(-.38,0.7){\line(1,0){2.45}}
\put(-.1,.1){$\{\, a. \,\} \; \{\, b. \,\}  \,$}
\end{picture}
}
\caption{$\preceq$ of Example~\ref{ex:no-gamma-for-ensc}}
\label{fig:preceq-revpm-reve}
\end{figure}
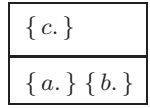
\qedex
\end{ex}

On the one hand, the requirement $SE(\cq) \cap SE(Q) \subseteq SE(R)$ in Definition~\ref{dfn:reve} requires any subset $R$ of $P$ that is not part of the cut to have \emph{all} SE~models shared by the cut and~$Q$, in order to be included in the revision outcome. In the previous example, the SE~models shared by the cut and~$Q$ are $(c,c), (c,bc)$, and $(bc,bc)$. Since $(c,c) \not \in SE(\{\, b. \,\})$ (and also $(c,bc) \not \in SE(\{\, b. \,\})$), it follows that $\{\, b. \,\} \nsubseteq P \reve Q$. 
On the other hand, the definition of partial meet revision is based on compatible sets, which are required to be maximal and to share only a minimum of one SE~model with $Q$ (Definition~\ref{dfn:compatible-set}). This requirement limits the result of a partial meet revision for this example to the one above, regardless of the the type of selection function employed.

We also find that the partial meet contraction operator $\conpm$ can be characterised in terms of the ensconcement contraction operator~$\cone$, formalised in the next theorem.

\begin{thm} \label{thm:conpm-cone}
Let $P,Q \in \lpa$. For any selection function $\gamma$, there exists an ensconcement $\preceq$ associated with $P$ such that $P \conpm Q = P \cone Q$.
\end{thm}

The other direction of this theorem does not hold. Consider again $P$ and $\preceq$ from Example~\ref{ex:no-gamma-for-ensc} and let $Q = \{\, a. \,\}$. It is easy to see that $P \conpm Q \neq P \cone Q$, for any selection function~$\gamma$.

\subsection{Granularity of Ensconcements}
For our partial meet construction, we determined the outcome of a revision or contraction operation by employing a function that selects among \emph{subsets} of a program. For our ensconcement construction, we then used an ordering over individual rules of a program to select the \emph{rules} to retain during a revision or contraction operation. Given the Characterisation Theorems~\ref{thm:revpm-reve} and~\ref{thm:conpm-cone} that hold only in one direction, it is worth investigating whether this difference in granularity, subsets or rules, plays a critical role in determining revision or contraction outcomes. To do so, we will now consider program subsets as the objects of change for our ensconcement revision and contraction operators. We begin with the definition of an ensconcement over subsets of a program.

\begin{dfn}[Ensconcement over Subsets] \label{dfn:subset-ens}
Given $P \in \lpa$, a \emph{subset-ensconcement associated with} $P$ is any total preorder~$\preceq^R$ on $2^P$ that satisfies the following conditions:
\begin{enumerate}[($\preceq^R$1)]
\item For any $R \subseteq P$: $SE(\{\, R' \subseteq P \setminus R \mid R \preceq R' \,\}) \not \subset SE(R)$
\item For any $R,R' \subseteq P$: $R \preceq R' \preceq R$ iff $R \equiv_s R'$
\end{enumerate}
\end{dfn}

We define revision and contraction operators based on $\preceq^R$ as follows.

\begin{dfn}[Subset-Ensconcement Revision] \label{dfn:reveR}
Let $P \in \lpa$ and $\preceq^R$ be a subset-ensconcement associated with~$P$. A \emph{subset-ensconcement revision operator} $\reveR$ for $P$ is defined such that for any $Q \in \lpa$:
$$
P \reveR Q =
\begin{cases}
\; P + Q &\text{ if } Q  \text{ is not satisfiable,} \\
\; \{\, R \subseteq P \mid SE(cut_{\preceq^R}(Q)) \cap SE(Q) \subseteq SE(R) \,\} + Q &\text{ otherwise,}
\end{cases}
$$
where $cut_{\preceq^R}(Q) = \left\{\, R \subseteq  P \mid SE\left(\{\, R' \subseteq P \mid R \preceq R' \,\}\right) \cap SE(Q) \neq \es \,\right\}$.
\end{dfn}

\begin{dfn}[Subset-Ensconcement Contraction] \label{dfn:coneR}
Let $P \in \lpa$ and $\preceq^R$ be a subset-ensconcement associated with~$P$. A \emph{subset-ensconcement contraction operator} $\coneR$ for $P$ is defined such that for any $Q \in \lpa$:
$$
P \coneR Q =
\begin{cases}
\; P &\text{ if } \models_s Q, \\
\; \{\, R \subseteq P \mid SE(cut_{\preceq^R}^-(Q)) \cap \seqc \subseteq SE(R) \,\} &\text{ otherwise,}
\end{cases}
$$
where $cut_{\preceq^R}^-(Q) = \{\, R \subseteq  P \mid SE\left(\{\, R' \subseteq P \mid R \preceq R' \,\}\right) \cap \seqc \neq \es \,\}$.
\end{dfn}

It turns out that it does not matter whether an ensconcement over subsets of a program or only over the individual rules is used to determine a revision or contraction outcome, provided that the individual rules are ordered in the same way in both ensconcements, as stated in the following theorem.

%\begin{lem} \label{lem:subset-not-above-rule}
%Let $\preceq^R$ be a subset-ensconcement associated with some $P \in \lpa$ and $R \subseteq P$. For any rule $r \in R$, it holds that $R \preceq^R \{r\}$.
%\end{lem}

\begin{thm} \label{thm:subset-equal-rule-based-change}
Let $P,Q \in \lpa$, $\preceq$ be an ensconcement associated with $P$, and $\preceq^R$ a subset-ensconcement associated with~$P$ such that $\{r\} \preceq^R \{r'\}$ iff $r \preceq r'$ for all $r,r' \in P$. Then $P \reve Q = P \reveR Q$ (or $P \cone Q = P \coneR Q$, alternatively).
\end{thm}

%According to Theorem~\ref{thm:subset-equal-rule-based-change}, it does not matter whether an ensconcement over subsets of a program or only over the individual rules is used to determine a revision or contraction outcome, provided that the individual rules are ordered in the same way in both ensconcements. %The reason is obvious -- by Lemma~\ref{lem:subset-not-above-rule} a subset is at most as ensconced as its lowest contained rule, so that any subset can be included in the cut only when all its constituents are included as well.

\subsection{Relating Revision Operators to Contraction Operators}
We now formalise the connection between partial meet revision and partial meet contraction with the help of the Levi and Harper identities as given in Definition~\ref{dfn:levi-harper}.

\begin{prop} \label{prop:revpm-via-conpm}
Let $P \in \lpa$, $\gamma$ be a selection function for $P$, and $*$ an operator for $P$ such that for any $Q \in \lpa$: $P * Q = (P \conpm \overline{Q}) + Q$. Then $P * Q = P \revpm Q$.
\end{prop}

\begin{prop} \label{prop:conpm-via-revpm}
Let $P \in \lpa$, $\gamma$ be a selection function for $P$, and $\dotminus$ an operator for $P$ such that for any $Q \in \lpa$: $P \dotminus Q = P \cap (P \revpm \overline{Q})$. Then $P \dotminus Q = P \conpm Q$.
\end{prop}

The characterisation via Levi and Harper identities also holds for ensconcement revision and ensconcement contraction.

\begin{prop} \label{prop:reve-via-cone}
Let $P \in \lpa$, $\preceq$ be an ensconcement associated with $P$, and $*$ an operator for $P$ such that for any $Q \in \lpa$: $P * Q = (P \cone \overline{Q}) + Q$. Then $P * Q = P \reve Q$.
\end{prop}

\begin{prop} \label{prop:cone-via-reve}
Let $P \in \lpa$, $\preceq$ be an ensconcement associated with $P$, and $\dotminus$ an operator for $P$ such that for any $Q \in \lpa$: $P \dotminus Q = P \cap (P \reve \overline{Q})$. Then $P \dotminus Q = P \cone Q$.
\end{prop}

\section{Localised Belief Change} \label{sec:localisedbc}
In Sections~\ref{sec:pm} and~\ref{sec:ens}, we introduced two new sets of belief change operators for logic programs. While the definitions of our operators are based on classic declarative constructions, such formulations may not be optimal for practical implementations. In particular, the formation of a set of compatible sets to conduct a partial meet revision or contraction requires that all possible combinations of all rules in a program are evaluated with respect to their sets of SE~models. When dealing with logic programs that contain a large number of rules, where only a small number of them are actually affected by the change operation, this procedure entails unreasonable costs. In this section, we present an algorithm to minimise these costs. We begin by identifying the subsets of a program, called \emph{modules}, relevant to another program.

\begin{dfn}[Module] \label{dfn:mod}
Let $P \in \lpa$ and $a \in \mathcal{A}$. For any rule $r \in P$ with $a \in At(r)$, we recursively construct
$M(P)^r_i\oa$ as 
$$
\{r\} \cup \{\, r' \in P \mid At(r') \cap \left(At(r) \cup At(M(P)^r_{i-1}\oa)\right) \setminus \{a\}  \neq \es \,\}
$$
for $i > 0$ and $M(P)^r_0\oa = \es$.

Since $P$ is finite and $M(P)^r_i\oa$ is monotonic with respect to $i$, the sequence $\bigcup_{i=0}^\infty M(P)^r_i\oa$ will reach a fixpoint. We denote the fixpoint by $M(P)^r\oa$ and call it the \emph{module of $P$ related to $r$ including $a$} (or the \emph{$r$-module including $a$}, if~$P$ is clear from the context).
\end{dfn}

\begin{ex} \label{ex:mod}
Let $r_1$: $a.$, $r_2$: $b \leftarrow a.$, $r_3$: $c \leftarrow not\, b.$, and $P = \{r_1, r_2, r_3\}$. The modules that can be constructed from $P$ are: $M(P)^{r_1}\oa = \{r_1\}, M(P)^{r_2}\oa = \{r_2,r_3\}, M(P)^{r_2}|_b = \{r_1,r_2\}, M(P)^{r_3}|_b = \{r_3\}$, and $M(P)^{r_3}|_c = \{r_1,r_2, r_3\}$.
\qedex
\end{ex}

Starting with a given atom $a$ and a given rule $r$ from $P$, the recursive definition first finds all rules in $P$ that share atoms with $r$ except for $a$. Then it finds all rules in $P$ that share atoms with $r$ or any of the rules found in the first step except for $a$, and so on. It does not matter whether atoms appear in the head or the body of a rule, or whether they occur with or without default negation. The resulting module is the collection of rules in $P$ that are related to $r$ through shared atoms. The reason for excluding $a$ will become clear after the following definition of a set of relevant modules.

\begin{dfn}[Relevant Module] \label{dfn:relmods}
Let $P \in \lpa$. Given an atom $a \in \mathcal{A}$, we define the \emph{set of all modules of $P$ including $a$} as:
$$
\mathcal{M}(P)\oa = \{\, M(P)^r\oa \mid r \in P \text{ and } a \in At(r) \,\}.
$$
Given $Q \in \lpa$, we define the \emph{set of all modules of $P$ relevant to $Q$} as:
$$
\mpq = \{\, M(P)^r\oa \mid r \in P \text{ and } a \in At(r) \cap At(Q) \,\}.
$$
\end{dfn}

Essentially, the definition of a set of modules extracts those rules from a program that may be affected during a revision or contraction by another program. It thus aims for the same goal as the language-splitting technique in propositional logic \cite{parikh1999beliefs}, which splits a knowledge base into several partitions either relevant or irrelevant to a belief change. However, a distinct feature in the previous definitions is the construction of a module based on \emph{each} rule in which a certain atom occurs. This feature allows us to get a closer look at which rules may conflict with some given information. Consider the program $\{\, a \leftarrow b.,\, \bot \leftarrow b. \,\}$. If we were to add the information that \lq\lq $b$ holds\rq\rq\ to this program, it would conflict with the latter rule but not with the first one. By creating a module for each occurrence of $b$, we split the program into two modules (one for each rule) and can assess the compatibility of each module with the new information separately. This also separates our approach from the method to compute \emph{compartments} \cite{hansson2002local,wassermann2000resource}. That method assumes a graph representation of a belief base, where each sentence of the belief base is a node and edges connect sentences that share at least one atom. The parts of a belief base relevant to a given sentence $\phi$ for a change operation are the sentences that can be reached from $\phi$. Thus, it does not distinguish between occurrences of $\phi$ as in our method. Furthermore, our method constructs modules for each individual atom occurring in $Q$ and thus ensures that we are dealing with minimal units of $P$ in a change operation. Obviously, a module may not be unique to a certain rule or a certain given atom so that modules may overlap or coincide.

%Defining \emph{compartments} \cite{hansson2002local} is an alternative method to extract parts of a knowledge base and to conduct change operations on the selected parts only, without affecting the remaining parts of the knowledge base. The compartment around a formula $\alpha$ in a knowledge base $B$ is the set of all minimal consistent subsets of $B$ that imply $\alpha$ or $\lnot \alpha$. Localised versions of revision and contraction can be achieved by applying the respective operator on each selected compartment instead of on the entire knowledge base. Only the formulas in each compartment are affected by the change operation. However, selecting compartments is a semantic method and still requires evaluating  based on a given consequence relation and as such complements our syntax-based approach proposed here, which builds modules based on the occurrence of symbols in formulas, i.e., logic programs in our case.

We say that a set of rules $R$ \emph{conflicts} with some program $Q$ if $SE(R) \cap SE(Q) = \emptyset$. All rules of $P$ that conflict with~$Q$ are included in some module or combination of modules from $\mpq$.

\begin{prop} \label{prop:conflict}
Let $P,Q \in \lpa$ and $SE(P) \neq \es \neq SE(Q)$. For any $R \subseteq P$, if $SE(R) \cap SE(Q) = \emptyset$ and for all $R' \subset R$ it holds that $SE(R') \cap SE(Q) \neq \emptyset$, then there exists $\mathbb{M} \in 2^{\mpq}$ such that $R \subseteq \bigcup \mathbb{M}$.
\end{prop}

\begin{cor} \label{cor:conflict}
Let $P,Q \in \lpa$ and $SE(P) \neq \es$. Then $SE(P) \cap SE(Q) = \emptyset$ if and only if $SE \left(\bigcup \mpq \right) \cap SE(Q) = \emptyset$.
\end{cor}

\begin{algorithm}
%\DontPrintSemicolon % Some LaTeX compilers require you to use \dontprintsemicolon instead
\KwIn{a set $\mathcal{M}$ of modules, an operator $\circ$, a program $Q$}
\KwOut{the set $\mathcal{M}$ of changed modules}
$n \gets1$\;
\While{$n \leq \lvert \mathcal{M} \rvert$}{
	\ForEach{$\mathbb{M} \subseteq \mathcal{M}$ such that $\lvert \mathbb{M} \rvert = n$}{ \label{ln:comb}
  		\uIf{$\circ$ is a revision operator and $SE(\bigcup \mathbb{M}) \cap SE(Q) = \es$}{
      		\ForEach{$M \in \mathbb{M}$} {
        replace $M$ with $(\bigcup \mathbb{M} \circ Q)\setminus Q$ in $\mathcal{M}$\;
			}
		}
		\uElseIf{$\circ$ is a contraction operator and $SE(\bigcup \mathbb{M}) \cap \seqc = \es$}{
      		\ForEach{$M \in \mathbb{M}$} {
        replace $M$ with $\bigcup \mathbb{M} \circ Q$ in $\mathcal{M}$\;
			}
		}
	}
  $n \gets n + 1$\;
}
\Return{$\mathcal{M}$}\;
\caption{\textsc{ModChange}}
\label{alg:modrev}
\end{algorithm}

We are now ready to introduce an optimisation algorithm for logic program revision and contraction based on modules. Algorithm~\ref{alg:modrev} resolves potential conflicts for all possible combinations of modules by applying revision or contraction on a modular level. It performs a bottom-up construction by first taking all 1-combinations (singleton sets of modules) of~$\mathcal{M}$ and substituting a module with its changed version if they are not the same. It then takes all 2-combinations of $\mathcal{M}$, which may now contain some changed modules, and replaces each module of the combination with the changed version of the combination if required. Replacing \emph{each} module of a combination with the outcome guarantees that the algorithm considers all possible combinations. The algorithm terminates after handling the combination of all modules in $\mathcal{M}$. The algorithm performs $\lvert \mathcal{M} \rvert^{{\lvert \mathcal{M} \rvert}/2}$ operations in the worst case, so its complexity is exponential. Since the formation of $\pq$ or $\pqm$ requires $\lvert 2^P \rvert^{\lvert 2^P \rvert/2}$ operations in the worst case, the algorithm performs better whenever $\lvert \mathcal{M} \rvert$ is less than $\lvert 2^P \rvert$. 

The next theorem states that the algorithm \textsc{ModChange} reduces a partial meet revision or contraction operation on a logic program to the revision or contraction operation on the relevant subsets of that program, given a suitable selection function~$\gamma$. In the following, let $P \setminus \mpq = \{\, r \in P \mid \text{for all } M \in \mpq: r \not \in M\,\}$ and $\mpq^\circ$ denote the output of Algorithm~\ref{alg:modrev} for the inputs $\mpq$, $\circ \in \{\revpm,\conpm,\reve,\cone\}$, and~$Q$.

\begin{thm} \label{thm:alg-same-as-op}
For any $P,Q \in \lpa$, there exists a selection function $\gamma$ for $P$ such that $P \revpm Q = P \setminus \mpq + \bigcup \mpq^{\revpm} + Q$ (or $P \conpm Q = P \setminus \mpq + \bigcup \mpq^{\conpm}$, respectively).
\end{thm}

The reason why Theorem~\ref{thm:alg-same-as-op} does not hold for any arbitrary selection function is that during the operation of \textsc{ModChange} a selection function chooses from subsets of modules, while it chooses from subsets of a program during the operation of $\revpm$ per Definition~\ref{dfn:revpm} ($\conpm$ per Definition~\ref{dfn:conpm}, respectively). Consequently, the result of the former operation may not in all cases correspond to the result of the latter operation.

We have already established in Theorems~\ref{thm:revpm-reve} and~\ref{thm:conpm-cone} that~$\revpm$ and~$\conpm$ can be characterised in terms of~$\reve$ and~$\cone$, respectively. Therefore, we can directly extend Theorem~\ref{thm:alg-same-as-op} to ensconcement revision and contraction.

\begin{cor} \label{cor:ens-alg}
For any $P,Q \in \lpa$, there exists an ensconcement $\preceq$ associated with $P$ such that $P \reve Q = P \setminus \mpq + \bigcup \mpq^{\reve} + Q$ (or $P \cone Q = P \setminus \mpq + \bigcup \mpq^{\cone}$, respectively).
\end{cor}

\section{Discussion} \label{sec:disc-classic-bc}
From our investigations in the previous sections, we can extract two main findings. Firstly, the belief change operators for logic programs that we proposed here are able to address the unintuitive  behaviour of the distance-based approach. The latter takes a holistic, program-level view on logic programs and their revisions, by assuming a belief state to be the set of SE~models of the entire program. Due to its focus on the program-level, the distance-based approach neglects information about relationships between atoms that is only captured on the rule-level, by the individual rules of a program, and therefore violates the properties of preservation and support. 

For our approach, we adapted partial meet and ensconcement constructions, which allowed us to define operators that are more sensitive with respect to the information expressed by the individual rules of a program. In particular, we considered the rules of a program in their syntactic form as a belief state and thus as the objects of change, which made it possible to preserve necessary information on the syntactic level during a change operation. This characteristic turned out to be key for satisfying the preservation and support properties. It should be noted, however, that our operators are still model-based and not purely syntactic operators. Our operators rely on satisfaction defined under SE~model semantics to determine compatible sets, remainder sets, ensconcements, and cuts, not on syntactic transformations of program components. Rather, our operators bridge the gap between purely semantic and purely syntactic methods, which is why we call them syntax-preserving.

Secondly, we found that our operators fit properly into the belief base framework for belief change.
%The AGM and belief base frameworks for belief change provide elegant structures to define change operations on bodies of beliefs. A crucial part of these frameworks are the postulates that characterise rational behaviour of such change operations. We translated the AGM and belief base postulates to the logic program setting and %then defined new operators for logic program belief change by adapting constructions from propositional logic that fit exactly within the classic frameworks. We evaluated each of our operators against them. 
Tables~\ref{tbl:AGM-rev},~\ref{tbl:AGM-con},~\ref{tbl:bb-rev}, and~\ref{tbl:bb-con} provide an overview of the postulates that are satisfied by each operator. Our partial meet revision operator $\revpm$ satisfies all basic AGM revision postulates ($*$1)--($*$6). However, the partial meet revision operator does not satisfy the supplementary postulates ($*$7)--($*$8), even with further restrictions on the selection function, which allow operators under propositional logic to satisfy the supplementary postulates. A similar situation exists for our partial meet contraction operator~$\conpm$. The partial meet contraction operator satisfies all basic postulates ($\dotminus$1)--($\dotminus$6) with the exception of the controversial Recovery postulate ($\dotminus$5). It satisfies ($\dotminus$7) but not ($\dotminus$8) of the supplementary postulates when the selection function is restricted to be determined by a maximised transitive relation. 

\begin{table}
\tbl{AGM revision postulates satisfied by $*_{\gamma'}$ and $\reve$ \label{tbl:AGM-rev}} 
{
%\centering
\begin{tabular}{|c|c|c|c|c|c|c|c|c|}
\hline
&($*$1)&($*$2)&($*$3)&($*$4)&($*$5)&($*$6)&($*$7)&($*$8)\\
\hline
$*_{\gamma'}$ &\checkmark &\checkmark &\checkmark &\checkmark &\checkmark &\checkmark & &  \\
\hline
$\reve$ &\checkmark &\checkmark &\checkmark &\checkmark &\checkmark &\checkmark & &\checkmark \\
%\hhline{|=|=|=|=|=|=|=|=|}
%&($*$1)&($*$2m)&($*$3m)&($*$4m)&($*$5)&($*$6)&($*$6w)\\
\hline
\end{tabular}
}
\vspace*{.25cm}
\tbl{AGM contraction postulates satisfied by $\dotminus_{\gamma'}$ and $\cone$ \label{tbl:AGM-con}} 
{
\begin{tabular}{|c|c|c|c|c|c|c|c|c|}
\hline
&($\dotminus$1) &($\dotminus$2) &($\dotminus$3) &($\dotminus$4) &($\dotminus$5) &($\dotminus$6) &($\dotminus$7) &($\dotminus$8) \\
\hline
$\dotminus_{\gamma'}$ &\checkmark &\checkmark &\checkmark &\checkmark  & &\checkmark  &\checkmark & \\
\hline
$\cone$ &\checkmark &\checkmark &\checkmark &\checkmark & &\checkmark &\checkmark &\checkmark  \\
%\hhline{|=|=|=|=|=|=|=|=|}
%&($*$1)&($*$2m)&($*$3m)&($*$4m)&($*$5)&($*$6)&($*$6w)\\
\hline
\end{tabular}
}
\end{table}

\begin{table}
\tbl{Base revision postulates satisfied by $\revpm$ and $\reve$ \label{tbl:bb-rev}} 
{
%\centering
\begin{tabular}{|c|c|c|c|c|c|}
\hline
&($*$1b)&($*$2b)&($*$3b)&($*$4b)&($*$5b)\\
\hline
$\revpm$ &\checkmark &\checkmark &\checkmark &\checkmark &\checkmark  \\
\hline
$\reve$ &\checkmark &\checkmark & & &\checkmark \\
%\hhline{|=|=|=|=|=|=|=|=|}
%&($*$1)&($*$2m)&($*$3m)&($*$4m)&($*$5)&($*$6)&($*$6w)\\
\hline
\end{tabular}
}
\vspace*{.25cm}
\tbl{Base contraction postulates satisfied by $\conpm$ and $\cone$ \label{tbl:bb-con}} 
{
\begin{tabular}{|c|c|c|c|c|c|c|c|c|}
\hline
&($\dotminus$1b) &($\dotminus$2b) &($\dotminus$3b) &($\dotminus$4b) &($\dotminus$5b) &($\dotminus$6b) &($\dotminus$7b) &($\dotminus$8b) \\
\hline
$\conpm$ &\checkmark &\checkmark &\checkmark &\checkmark  &\checkmark &\checkmark  & &\checkmark \\
\hline
$\cone$ &\checkmark &\checkmark & & &\checkmark &\checkmark &\checkmark &\checkmark  \\
%\hhline{|=|=|=|=|=|=|=|=|}
%&($*$1)&($*$2m)&($*$3m)&($*$4m)&($*$5)&($*$6)&($*$6w)\\
\hline
\end{tabular}
}
\end{table}

On the other hand, evaluating our partial meet revision and contraction operators with respect to the belief base postulates showed that they exhibit the same characteristics as the partial meet base revision and contraction operators for propositional logic \cite{hansson1993reversing}. The partial meet revision operator is represented by ($*$1b)--($*$5b) (Theorem~\ref{thm:rev-pm-bb}) and the partial meet contraction operator by ($\dotminus$1b)--($\dotminus$4b) (Theorem~\ref{thm:con-pm-bb}), and thus both operators fit neatly into the belief base framework.

Our ensconcement revision operator $\cone$ satisfies the same basic AGM postulates ($*$1)--($*$6) as our partial meet revision operator and in addition ($*$8), but also disrespects ($*$7). Our ensconcement contraction operator~$\cone$ satisfies all AGM postulates ($\dotminus$1)--($\dotminus$8) with the exception of ($\dotminus$5). Since an ensconcement-based construction is essentially geared towards belief bases \cite{williams1994logic}, our ensconcement operators should align well with the belief base framework. Indeed, our ensconcement contraction operator~$\cone$ satisfies the same set of belief base postulates as its counterpart for propositional logic \cite{ferme2008axiomatic}, that is, ($\dotminus$1b), ($\dotminus$2b), and ($\dotminus$5b)--($\dotminus$8b). As the classic belief base revision postulates ($\divideontimes$1)--($\divideontimes$5) have originally been proposed to characterise partial meet base revision operations, their applicability to characterise ensconcement revision operations is limited. It would be interesting for future work to define a set of belief base revision postulates that can exactly characterise ensconcement revision operations. Until then, we can use our adaptation of the Harper identity to show that any ensconcement contraction operator determined by our ensconcement revision operator~$\reve$ satisfies ($\dotminus$1b), ($\dotminus$2b), and ($\dotminus$5b)--($\dotminus$8b).

It is now left to examine how our operators compare to the distance-based operator in terms of the AGM and belief base postulates. Tables~\ref{tbl:AGM-star} and~\ref{tbl:bb-star} display the postulates satisfied by the distance-based operator $\star$. We can see that both our partial meet revision operator~$\revpm$ and our ensconcement revision operator~$\reve$ are better-behaved than~$\star$ on the scale of AGM postulates as well as on the scale of belief base postulates. It should be noted, however, that~$\star$ satisfies ($*$4m), which has a stricter antecedent than~($*$6), but neither~$\revpm$ nor~$\reve$ satisfies ($*$4m). 
%We demonstrate this in the next two examples.

\begin{table}
\tbl{AGM revision postulates satisfied by $\star$ \label{tbl:AGM-star}} 
{
%\centering
\begin{tabular}{|c|c|c|c|c|c|c|c|c|}
\hline
&($*$1)&($*$2)&($*$3)&($*$4)&($*$5)&($*$6)&($*$7)&($*$8)\\
\hline
$\star$ &\checkmark & & & &\checkmark &\checkmark & & \\
\hline
\end{tabular}
}
\vspace*{.25cm}
\tbl{Base revision postulates satisfied by $\star$ \label{tbl:bb-star}} 
{
\centering
\begin{tabular}{|c|c|c|c|c|c|}
\hline
&($*$1b)&($*$2b)&($*$3b)&($*$4b)&($*$5b)\\
\hline
$\star$ & & & & &\checkmark \\
\hline
\end{tabular}
}
\end{table}

%\begin{ex} \label{ex:revpm-syn-indep}
%Let $P = \{\, a.,\, b. \,\}$, $P' = \{\, a.,\, b \leftarrow a. \,\}$, and $Q = \{\, \bot \leftarrow a.\,\}$. For any selection function $\gamma$ for $P$ and any selection function $\gamma'$ for $P'$, we have $\pq = \{\, \{\, b. \,\} \,\} = \gpq $ and $\pq' =  \{\, \{\, b \leftarrow a. \,\} \,\} = \gamma(\pq')$. Thus, $P \revpm Q =  \{\, b.,\, \bot \leftarrow a.\,\}$ and $P' *_{\gamma'} Q = \{\, b \leftarrow a.,\, \bot \leftarrow a.\,\}$.
%\qedex
%\end{ex} 
%
%\begin{ex} \label{ex:reve-syn-indep}
%Consider again $P$, $P'$, and $Q$ from Example~\ref{ex:revpm-syn-indep}. Let $\preceq$ be the ensconcement associated with $P$ such that $\{\, a. \,\} \prec \{\, b. \,\}$. Then $P \reve Q = \{\, b.,\, \bot \leftarrow a. \,\}$, while for any ensconcement $\preceq'$ associated with $P'$, it holds that $P' *_{\preceq'} Q = \{\, b \leftarrow a.,\, \bot \leftarrow a. \,\}$.
%\qedex
%\end{ex}
%
%Although $P \equiv_s P'$ in these two examples, we find that $P \revpm Q = P \reve Q \not \equiv_s P' *_{\preceq'} Q = P' *_{\gamma'} Q$, because $SE(P \revpm Q) = SE(P \reve Q) =  \{(b,b)\} \neq \{(\es,\es),(\es,b),(b,b)\} = SE(P' *_{\preceq'} Q) = SE(P' *_{\gamma'} Q)$. In fact, this behaviour is to be somewhat expected. As ($*$4m) is the postulate that expresses full syntax-independence, it is arguably too strict a requirement for any revision operator that respects semantic as well as syntactic information.

\section{Conclusion} \label{sec:concl}
%In this work, we addressed some shortcomings of existing approaches to belief revision in logic programs. We adapted the AGM and belief base frameworks for belief change to logic programs and proposed two constructions to perform belief change in logic programs. Our adaptation is novel in that our change operators preserve information on the program-level and rule-level of a program and that our translation of the AGM framework is closer to the original formulation than previous translations. It turned out that our choice of defining partial meet and ensconcement constructions not only allows our operators to preserve more information from a logic program than pure semantic operators, but it also facilitated a natural definition of a contraction operator for logic programs, the first in the field to the best of our knowledge.

In this work, we presented two new constructions of belief change in logic programs. Our specific aim was to overcome the drawbacks of existing semantic revision operators, namely, that they do not satisfy the properties of preservation and support. These are fundamental properties for a logic program belief change operator to return intuitive results. For this purpose, we chose to adapt partial meet and ensconcement constructions from classic belief change, which allowed us to define syntax-preserving belief change operators for logic programs that satisfy preservation and support.

Our approach is novel in that the partial meet and ensconcement constructions not only enabled our operators to preserve more information from a logic program during a change operation than purely semantic operators, but they also facilitated natural definitions of contraction operators for logic programs, the first in the field to the best of our knowledge.

In order to evaluate the rationality of our operators, we translated the revision and contraction postulates from the classic AGM and belief base frameworks to logic programs. %Our translation of the AGM framework is closer to the original formulation than existing translations. 
We established that our operators fit properly within the belief base framework and showed their interdefinability. We also demonstrated that our operators align more closely to the AGM and belief base frameworks than the distance-based revision operators and that they generalise the screened semi-revision operator. We further presented an algorithm to optimise our revision and contraction operations. In future work, we aim to develop and implement an algorithm for constructing the relevant modules of a program. 

While our partial meet and ensconcement operators specify how a program changes during a revision or contraction operation, they do not specify how the associated selection function or ensconcement relation changes. A selection function or ensconcement is associated with a program before the change operation. During the change operation, some rules may be discarded from the program and some new rules may be added to it in the case of revision, and any effects on the initial selection function or ensconcement or entrenchment should be taken into consideration. Such an extension to our approach would be worthwhile to pursue in the future.
\appendix
\section*{APPENDIX: PROOFS}
\setcounter{section}{1}
%\appendixhead{BINNEWIES}

%\subsection*{Proofs for Section~\ref{sec:adapting}}
\textsc{Proposition~\ref{prop:relevance-to-disjelim}}:
Let $\dotminus$ be a contraction operator on $\lpa$. If $\dotminus$ satisfies ($\dotminus$3b), then it satisfies ($\dotminus$8b).
\begin{proof}
Proof by contrapositive: Let $r \in P \setminus (P \dotminus Q)$. Assume $SE(P \dotminus Q) \subseteq SE(Q) \cup SE(r)$. Then for all $P'$ such that $P \dotminus Q \subseteq P' \subseteq P$ with $SE(P') \nsubseteq SE(Q)$: $SE(P') \cap SE(r) \nsubseteq SE(Q)$, due to the assumption and due to $SE(P') \subseteq SE(P \dotminus Q)$.
\end{proof}

\textsc{Theorem~\ref{thm:rev-pm}}:
The revision operator $\revpm$ satisfies ($*$1)--($*$6).

\begin{proof}
\\
($*$1): Follows directly from Definition~\ref{dfn:revpm}. \\
($*$2): Follows directly from Definition~\ref{dfn:revpm}. \\
($*$3): If $Q$ is not satisfiable, then $P \revpm Q = P + Q$. Otherwise, since $\bigcap \gpq \subseteq P$ we have $\bigcap \gpq + Q \subseteq P + Q$. \\
($*$4): If $P + Q$ is satisfiable, then $\pq = \{P\} = \gpq$, for any selection function~$\gamma$,  and thus $P \revpm Q = P + Q$. \\
($*$5): If $Q$ is not satisfiable, then $P \revpm Q = P + Q$ is not satisfiable. If $Q$ is satisfiable, then for any $R \in \pq$, $R + Q$ is satisfiable, which implies $P \revpm Q$ is satisfiable. \\
($*$6): Follows directly from Definition~\ref{dfn:revpm}.
\end{proof}

%\textsc{Lemma~\ref{lem:PQ-shared-rules-in-result}}:
%Let $P,Q \in \lpa$ and $\gamma$ be a selection function for $P$. If $r \in P \cap Q$, then $r \in \bigcap \gpq$.

\begin{lem} \label{lem:PQ-shared-rules-in-result}
Let $P,Q \in \lpa$ and $\gamma$ be a selection function for $P$. If $r \in P \cap Q$, then $r \in \bigcap \gpq$.
\end{lem}

\begin{proof}
Let $P,Q \in \lpa$. Assume there exists $r \in (P \cap Q) \setminus (\bigcap \gpq)$. Then there exists $R \in \gpq : r \not \in R$. It follows that $SE(R \cup \{r\}) \cap SE(Q) \neq \es$ since $SE(R) \cap SE(Q) \neq \es$ by Definition~\ref{dfn:compatible-set} and $r \in Q$ implies $SE(Q) \subseteq SE(r)$. This is a contradiction because $R$ is maximal by Definition~\ref{dfn:compatible-set}.
\end{proof}

\textsc{Theorem~\ref{thm:rev-pm-bb}}:
An operator $\revpm$ is a partial meet revision operator for $P \in \lpa$ determined by a selection function $\gamma$ for $P$ iff $\revpm$ satisfies ($*$1b)--($*$5b).

\begin{proof} We first show that a partial meet revision operator $\revpm$ for $P$ determined by a given selection function $\gamma$ for $P$ satisfies ($*$1b)--($*$5b). \\
($*$1b): Since ($*$1b) $=$ ($*$2) and $\revpm$ satisfies ($*$2), $\revpm$ also satisfies ($*$1b). \\
($*$2b): Since ($*$2b) $=$ ($*$3) and $\revpm$ satisfies ($*$3), $\revpm$ also satisfies ($*$2b). \\
($*$3b): Let $r \in P$. Assume that for all $P'$ with $P \revpm Q \subseteq P' \subset P + Q$ and $P'$ being satisfiable, it holds that $P' \cup \{r\}$ is satisfiable. In particular, for each $R \in \pq$ with $P \revpm Q \subseteq R \cup Q$, this implies $R \cup Q \cup \{r\}$ is satisfiable. As each $R$ is subset-maximal, it follows that $r \in R$ and thus $r \in \bigcap \gpq$. From Definition~\ref{dfn:revpm} we can then conclude $r \not \in P \setminus (P \revpm Q)$. \\
($*$4b): For all $P' \subseteq P$, let $P' + Q$ be satisfiable iff $P' + R$ is satisfiable. Then $\pq = \mathbb{P}_R$ by Definition~\ref{dfn:compatible-set} and so $\bigcap \gpq = \bigcap \gamma(\mathbb{P}_R)$ as well as $P \cap \bigcap \gpq = P \cap \bigcap \gamma(\mathbb{P}_R)$. By Lemma~\ref{lem:PQ-shared-rules-in-result} we obtain $(P \cap \bigcap \gpq) \cup (P \cap Q) = (P \cap \bigcap \gamma(\mathbb{P}_R)) \cup (P \cap R)$. This means $P \cap (\bigcap \gpq \cup Q) = P \cap (\bigcap \gamma(\mathbb{P}_R) \cup R)$. Thus, $P \cap (P \revpm Q) = P \cap (P \revpm R)$. \\
($*$5b): If $Q$ is satisfiable, then for any $R \in \pq$, $R + Q$ is satisfiable, which implies $P \revpm Q$ is satisfiable.

We now show that any operator $\cg$ for $P$ satisfying ($*$1b)--($*$5b) is a partial meet revision operator for $P$ determined by some selection function for $P$. We first find a selection function $\gamma$ for $P$. Let $\gamma$ be such that (i) if $\pq = \es$, then $\gpq = \es$ and (ii) $\gpq = \{\, R \in \pq \mid P \cap (P \cg Q) \subseteq R \,\}$ otherwise. 

We begin by showing that $\gamma$ is a function. If $\pq = \mathbb{P}_R$, then $P \cap (P \cg Q) = P \cap (P \cg R)$ by ($*$4b). This means $\gpq = \gamma(\mathbb{P}_R)$ according to our definition of $\gamma$.

We next show that $\gamma$ is a selection function. Clearly, $\gpq \subseteq \pq$ by our definition of $\gamma$. If $\pq \neq \es$, then $Q$ is satisfiable by Definition~\ref{dfn:compatible-set} and thus $P \cg Q$ is satisfiable by ($*$5b). Since $Q \subseteq P \cg Q$ by ($*$1b) and $P \cg Q \subseteq P \cup Q$ by ($*$2b), it follows that $(P \cap (P \cg Q)) \cup Q$ is satisfiable. This means that there exists $R \in \pq$ such that $P \cap (P \cg Q) \subseteq R$. From our definition of $\gamma$ we therefore obtain that $\gpq \neq \es$. 

Finally, we show that $\cg$ is a partial meet revision operator for $P$, that is, $P \cg Q = P \cup Q$ if $Q$ is not satisfiable and $P \cg Q = \bigcap \gpq \cup Q$ otherwise. Consider first the limiting case that $Q$ is not satisfiable. If $r \in P \setminus (P \cg Q)$, then there exists $P'$ such that $P \cg Q \subseteq P' \subset P \cup Q$ and $P'$ is satisfiable but $P' \cup \{r\}$ is not satisfiable by ($*$3b). This is a contradiction since $Q \subseteq P'$ by ($*$1b). Therefore, it holds for all $r \in P$ that  $r \in P \cg Q$, that is, $P \subseteq P \cg Q$. Since $Q \subseteq P \cg Q$ by ($*$1b) and $P \cg Q \subseteq P \cup Q$ by ($*$2b), we can conclude $P \cg Q = P \cup Q$. 

Assume now that $Q$ is satisfiable. Let $r \in P \setminus (P \cg Q)$. If $\pq = \es$, then it follows from ($*$1b) and ($*$3b) that $P \cg Q = Q$. Since $\gpq = \es$ by our definition of $\gamma$, we thus have $P \cg Q = Q = \bigcap \gpq \cup Q$. If $\pq \neq \es$, then it follows directly from our definition of $\gamma$ that $P \cap (P \cg Q) \subseteq \bigcap \gpq$. From ($*$1b) and ($*$2b) we then obtain $P \cg Q \subseteq \bigcap \gpq \cup Q$. To show the converse inclusion, first assume the case that $P \cup Q$ is satisfiable. This implies that for any $P' \subseteq P \cup Q$ it holds that $P'$ is satisfiable. Applying ($*$3b), we obtain $P \setminus (P \cg Q) = \es$ and thus $P \subseteq P \cg Q$. From ($*$1b) and ($*$2b) it follows that $P \cg Q = P \cup Q$. Moreover, due to the assumption that $P \cup Q$ is satisfiable and Definition~\ref{dfn:compatible-set}, we have $\pq = \{P\}$. By our definition of $\gamma$, we obtain $\gpq = \{P\}$ and thus  $\bigcap \gpq = P$ and can conclude $P \cg Q = \bigcap \gpq \cup Q$. Lastly, assume the case that $P \cup Q$ is not satisfiable. We will show that $r \not \in P \cg Q$ implies $r \not \in \bigcap \gpq \cup Q$. If $r \not \in P$, then $r \not \in (P \cg Q) \setminus Q$ by ($*$1b) and ($*$2b) and $r \not \in \bigcap \gpq$ by Definition~\ref{dfn:compatible-set}. Since $r \not \in P \cg Q$ implies $r \not \in Q$ by ($*$1b), it follows that $r \not \in ((P \cg Q) \setminus Q) \cup Q) = P \cg Q$ and $r \in \bigcap \gpq \cup Q$. Now assume $r \in P \setminus (P \cg Q)$. According to ($*$3b), then there exists $P'$ such that $P \cg Q \subseteq P' \subset P \cup Q$ and $P'$ is satisfiable but $P' \cup \{r\}$ is not satisfiable. This means that there exists $R \in \pq$ such that $P \cap P' \subseteq R$ and $r \not \in R$. Since $P \cap (P \cg Q) \subseteq P \cap P' \subseteq R$, we obtain from our definition of $\gamma$ that $R \in \gpq$. We can thus conclude from $r \not \in  R$ that $r \not \in \bigcap \gpq$.
\end{proof}

\textsc{Proposition~\ref{prop:prop:revpm-same-as-consolidation-SE}}:
Let $P,Q \in \lpa$. For any maxichoice selection function $\gamma_P$ for $P$, there exists a selection function $\gamma$ for $P$ such that $(P \cup Q) \scrse Q = P \revpm Q$.

\begin{proof}
Let $P,Q \in \lpa$ and $\gamma_P$ be a maxichoice selection function for $P$. We will prove by cases.

Case 1: $Q$ is not satisfiable. This means $(P \cup Q) \botse Q = \es$ by definition of $P \botse Q$. Then $\gamma_P((P \cup Q) \botse Q) = P \cup Q$ by definition of $\gamma_P$. It follows that $(P \cup Q) \scrse Q = P \cup Q = P \revpm Q$ by definition of $\scrse$ and Definition~\ref{dfn:revpm}.

Case 2: $Q$ is satisfiable and for all $R \subseteq P: SE(R \cup Q) = \es$. Then $(P \cup Q) \botse Q = \{Q\}$ by definition of $\botse$ and $\pq = \es = \bigcap \gpq$ by Definition~\ref{dfn:compatible-set}. It follows that $(P \cup Q) \scrse Q = \gamma_P ((P \cup Q) \botse Q) = Q = P \revpm Q$ by definition of $\scrse$ and Definition~\ref{dfn:revpm}.

Case 3: $Q$ is satisfiable and there exists $R \subseteq P: SE(R \cup Q) \neq \es$ such that for all $R'$ with $R \subset R' \subseteq P: SE(R') = \es$. Then $R \cup Q \in (P \cup Q) \botse Q$ by definition of $\botse$ and $R \in \pq$ by Definition~\ref{dfn:compatible-set}. Let $\gamma$ be a maxichoice selection function for $P$ such that $\gamma(2^P) \cup Q = \gamma_P(\{\, R \cup Q \mid R \in 2^P \,\})$. This implies $\gamma_P((P \cup Q) \botse Q) = \gpq \cup Q = \bigcap \gpq \cup Q$. Thus, $(P \cup Q) \scrse Q = P \revpm Q$ by definition of $\scrse$ and Definition~\ref{dfn:revpm}.
\end{proof}

\textsc{Proposition~\ref{prop:revpm-same-as-consolidation-AS}}:
Let $P,Q \in \lpa$ and $\gamma^1$, $\gamma_P$ be single-choice and maxichoice selection functions, respectively, for $P$. If $\gamma^1(2^P) \cup Q = \gamma_P(\{\, R \cup Q \mid R \in 2^P \,\})$ for any $\gamma_P$, then $P \revpmas Q = (P \cup Q) \scr Q$.

\begin{proof}
Let $P,Q \in \lpa$ and $\gamma^1$, $\gamma_P$ be single-choice and maxichoice selection functions, respectively, for $P$. Assume that $\gamma^1(2^P) \cup Q = \gamma_P(\{\, R \cup Q \mid R \in 2^P \,\})$ for any $\gamma_P$. We will prove by cases.

Case 1: $AS(Q) = \es$ and for all $R \subseteq P: AS(R \cup Q) = \es$. This implies $\pqas = \es$ by definition of $\pqas$ and thus $P \revpm Q = P \cup Q$ by Definition~\ref{dfn:revpmas}. It also implies $(P \cup Q) \bot_! Q = \es$ by definition of $P \bot_! Q$ and therefore $\gamma_P((P \cup Q) \bot_! Q) = P \cup Q$ by definition of $\gamma_P$. We can conclude $(P \cup Q) \scr Q = P \cup Q$ by Definition~\ref{dfn:scr}.

Case 2: $AS(Q) \neq \es$ and for all $R \subseteq P: AS(R \cup Q) = \es$. Then $\pqas = \es = \gamma^1(\pqas)$ by definitions of $\pqas$ and $\gamma^1$. We also have $(P \cup Q) \bot_! Q = \{Q\} = \gamma_P((P \cup Q) \bot_! Q)$ by definitions of $\bot_!$ and $\gamma_P$. It follows that $P \revpmas Q = Q = (P \cup Q) \scr Q$ by Definitions~\ref{dfn:revpmas} and~\ref{dfn:scr}.

Case 3: There exists $R \subseteq P: AS(R \cup Q) \neq \es$ and for all $R'$ with $R \subset R' \subseteq P : AS(R' \cup Q) = \es$. Then $R \in \pqas$ by definition of $\pqas$ and $R \cup Q \in (P \cup Q) \bot_! Q$ by definition of $\bot_!$. Due to the assumption, it holds that $\gamma^1(\pqas) \cup Q = \gamma_P((P \cup Q) \bot_! Q)$. Thus, $P \revpmas Q = (P \cup Q) \scr Q$ by Definitions~\ref{dfn:revpmas} and~\ref{dfn:scr}.
\end{proof}

\textsc{Theorem~\ref{thm:con-pm}}:
The contraction operator $\conpm$ satisfies ($\dotminus$1)--($\dotminus$4) and ($\dotminus$6).

\begin{proof}
\\
($\dotminus$1): Follows directly from Definition~\ref{dfn:conpm}. \\
($\dotminus$2): Follows directly from Definition~\ref{dfn:conpm}. \\
($\dotminus$3): If $P \not \models_s Q$, then $\pqm = \{P\} = \gpqm$, for any selection function $\gamma$, and thus $P \conpm Q = P$. \\
($\dotminus$4): Let $\not \models_s Q$. For any $R \in \pqm$, $R \not \models_s Q$, which implies $P \conpm Q \not \models_s Q$. \\
($\dotminus$6): Follows directly from Definition~\ref{dfn:conpm}.
\end{proof}

%\textsc{Lemma~\ref{lem:pqrm-subset-pq-union-pr}}:
%Let $P \in \lpa$. For any $Q,R \in \lpa$, it holds that $\pqrm \subseteq \pqm \cup \prm$.

\begin{lem} \label{lem:pqrm-subset-pq-union-pr}
Let $P \in \lpa$. For any $Q,R \in \lpa$, it holds that $\pqrm \subseteq \pqm \cup \prm$.
\end{lem}

\begin{proof}
Let $P,Q,R \in \lpa$. It follows from $\overline{SE(Q+R)} = \seqc \cup \overline{SE(R)}$ and the definition of $\pqrm$ that $\pqrm = \{\, S \in \pqm \cup \prm \mid S \not \subset S' \text{ for any } S' \in \pqm \cup \prm \,\}$. Thus, $\pqrm \subseteq \pqm \cup \prm$.
\end{proof}

%\textsc{Lemma~\ref{lem:gammapqrm-subset-gammapq-union-gammapr}}:
%Let $P \in \lpa$ and $\gamma'$ be determined by a maximised transitive relation. For any $Q,R \in \lpa$, it holds that $\gamma'(\pqrm) \subseteq \gamma'(\pqm) \cup \gamma'(\prm)$.

\begin{lem} \label{lem:gammapqrm-subset-gammapq-union-gammapr}
Let $P \in \lpa$ and $\gamma'$ be determined by a maximised transitive relation. For any $Q,R \in \lpa$, it holds that $\gamma'(\pqrm) \subseteq \gamma'(\pqm) \cup \gamma'(\prm)$.
\end{lem}

\begin{proof}
Let $P,Q,R \in \lpa$. Assume there exists $S \in \gamma'(\pqrm): S \not \in  \gamma'(\pqm) \cup \gamma'(\prm)$. Then $S \not \in  \gamma'(\pqm)$ and $S \not \in \gamma'(\prm)$. Case 1: If $S \not \in \pqm$, this means that $S \in \prm$ by Lemma~\ref{lem:pqrm-subset-pq-union-pr}. From Definition~\ref{dfn:mtr-sel-func} it follows that there exists $S' \in \prm: S \subset S'$, a contradiction since $S \in \pqrm$. Case 2: $S \in \pqm$. Follows analogously as Case 1.
\end{proof}

\textsc{Theorem~\ref{thm:conpm-sat-7}}:
Let $\gamma'$ be determined by a maximised transitive relation. The contraction operator $\dotminus_{\gamma'}$ satisfies ($\dotminus$7).

\begin{proof}
Let $P,Q,R \in \lpa$ and $r \in (P \dotminus_{\gamma'} Q) \cap (P \dotminus_{\gamma'} R)$. This means that $r \in \bigcap \gamma'(\pqm)$ and $r \in \bigcap \gamma'(\prm)$. By Lemma~\ref{lem:gammapqrm-subset-gammapq-union-gammapr}, we have for all $S \in \gamma'(\pqrm): r \in S$, so that $r \in \bigcap \gamma'(\pqrm)$. Thus, $(P \dotminus_{\gamma'} Q) \cap (P \dotminus_{\gamma'} R) \subseteq P \dotminus_{\gamma'} (Q+R)$.
\end{proof}

\textsc{Theorem~\ref{thm:con-pm-bb}}:
An operator $\conpm$ is a partial meet contraction operator for $P \in \lpa$ determined by a selection function $\gamma$ for $P$ iff $\conpm$ satisfies ($\dotminus$1b)--($\dotminus$4b).

\begin{proof} We first show that a partial meet contraction operator $\conpm$ for $P$ determined by a given selection function $\gamma$ for $P$ satisfies ($\dotminus$1b)--($\dotminus$4b). \\
($\dotminus$1b): Follows from ($\dotminus$1b) $=$ ($\dotminus$2) and satisfaction of ($\dotminus$2). \\
($\dotminus$2b): Follows from ($\dotminus$2b) $=$ ($\dotminus$4) and satisfaction of ($\dotminus$4). \\
($\dotminus$3b): Let $r \in P$. Assume that for all $P'$ with $P \conpm Q \subseteq P' \subset P$ and $P' \not \models_s Q$, it holds that $P' \cup \{r\} \not \models_s Q$. In particular, for each $R \in \pqm$ with $P \conpm Q \subseteq R$, this implies $R \cup \{r\} \not \models_s Q$. As each $R$ is subset-maximal by Definition~\ref{dfn:remainder-set}, it follows that $r \in R$ and thus $r \in P \conpm Q$. \\
($\dotminus$4b): For all $P' \subseteq P$, let $P' \not \models_s Q$ iff $P' \not \models_s R$. Then $\pqm = \prm$ by Definition~\ref{dfn:remainder-set} and so $\gpqm = \gamma(\prm)$ as well as $\bigcap \gpqm = \bigcap \gamma(\prm)$. Thus, $P \conpm Q = P \conpm R$ by Definition~\ref{dfn:conpm}. 

We now show that any operator $\cg$ for $P$ satisfying ($\dotminus$1b)--($\dotminus$4b) is a partial meet contraction operator for $P$ determined by some selection function for $P$. We first find a selection function $\gamma$ for P. Let $\gamma$ be such that (i) if $\pqm = \es$, then $\gpqm = \es$ and (ii) $\gpqm = \{\, R \in \pqm \mid P \cg Q \subseteq R \,\}$ otherwise. 

We begin by showing that $\gamma$ is a function. If $\pqm = \mathbb{P}_R^-$, then $P \cg Q = P \cg R$ by ($\dotminus$4b). This means $\gpqm = \gamma(\mathbb{P}_R^-)$ according to our definition of $\gamma$.

We next show that $\gamma$ is a selection function. Clearly, $\gpqm \subseteq \pqm$ by our definition of $\gamma$. If $\pqm \neq \es$, then $\not \models_s Q$ by Definition~\ref{dfn:remainder-set} and thus $P \cg Q \not \models_s Q$ by ($\dotminus$2b). It follows from $P \cg Q \subseteq P$ due to ($\dotminus$1b) that there exists $R \in \pqm$ such that $P \cg Q \subseteq R$. From our definition of $\gamma$ we therefore obtain that $\gpqm \neq \es$. 

Finally, we show that $\cg$ is a partial meet contraction operator for $P$, that is, $P \cg Q = P $ if $\models_s Q$ and $P \cg Q = \bigcap \gpqm$ otherwise. Consider first the limiting case that $\models_s Q$. If $r \in P \setminus (P \cg Q)$, then there exists $P'$ such that $P \cg Q \subseteq P' \subset P$ and $P' \not \models_s Q$ but $P' \cup \{r\} \models_s Q$ by ($\dotminus$3b). This is a contradiction since $\models_s Q$. Therefore, it holds for all $r \in P$ that $r \in P \cg Q$, that is, $P \subseteq P \cg Q$. Since $P \cg Q \subseteq P$ by ($\dotminus$1b), we can conclude $P \cg Q = P$. 

Assume now that $\not \models_s Q$. Let $r \in P \setminus (P \cg Q)$. If $\pqm = \es$, then it follows from ($\dotminus$2b) and ($\dotminus$3b) that $P \cg Q = \es$. Since $\gpqm = \es$ by our definition of $\gamma$, we thus have $P \cg Q = \bigcap \gpqm$. If $\pqm \neq \es$, then it follows directly from our definition of $\gamma$ that $P \cg Q \subseteq \bigcap \gpqm$. To show the converse inclusion, first assume the case that $P \not \models_s Q$. This implies that for any $P' \subseteq P$ it holds that $P' \not \models_s Q$. Applying ($\dotminus$3b), we obtain $P \setminus (P \cg Q) = \es$ and thus $P \subseteq P \cg Q$. From ($\dotminus$1b) it follows that $P \cg Q = P$. Moreover, due to the assumption that $P \not \models_s Q$ and Definition~\ref{dfn:remainder-set}, we have $\pqm = \{P\}$. By our definition of $\gamma$, we obtain $\gpqm = \{P\}$ and thus $\bigcap \gpqm = P$ and can conclude $P \cg Q = \bigcap \gpqm$. Lastly, assume the case that $P \models_s Q$. We will show that $r \not \in P \cg Q$ implies $r \not \in \bigcap \gpqm$. If $r \not \in P$, then $r \not \in P \cg Q$ by ($\dotminus$1b) and $r \not \in \bigcap \gpqm$ by Definition~\ref{dfn:remainder-set}. Now assume $r \in P \setminus (P \cg Q)$. According to ($\dotminus$3b), then there exists $P'$ such that $P \cg Q \subseteq P' \subset P$ and $P' \not \models_s Q$ but $P' \cup \{r\} \models_s Q$. This means that there exists $R \in \pqm$ such that $ P' \subseteq R$ and $r \not \in R$. Since $P \cg Q \subseteq P' \subseteq R$, we obtain from our definition of $\gamma$ that $R \in \gpqm$. We can thus conclude from $r \not \in  R$ that $r \not \in \bigcap \gpqm$.
\end{proof}

\textsc{Proposition~\ref{prop:con-pm-bb-5--8}}:
The contraction operator $\conpm$ satisfies ($\dotminus$5b), ($\dotminus$6b), and ($\dotminus$8b).

\begin{proof}
\\
($\dotminus$5b): Since ($\dotminus$5b) $=$ ($\dotminus$3) and $\conpm$ satisfies ($\dotminus$3), $\conpm$ also satisfies ($\dotminus$5b). \\
($\dotminus$6b): Since ($\dotminus$6b) $=$ ($\dotminus$6) and $\conpm$ satisfies ($\dotminus$6), $\conpm$ also satisfies ($\dotminus$6b). \\
($\dotminus$8b): Follows from satisfaction of ($\dotminus$3b) and Proposition~\ref{prop:relevance-to-disjelim}.
\end{proof}

\textsc{Lemma~\ref{lem:cut-properties}}:
Let $P,Q,R \in \lpa$ and $\preceq$ be an ensconcement associated with $P$.
\begin{enumerate}[a)]
\item If $P + Q$ is satisfiable, then $\cq = P$. 
\item If $Q$ is satisfiable, then $\cq + Q$ is satisfiable. 
\item If $Q$ is not satisfiable, then $\cq = \es$. 
\item If $Q \models_s R$, then $\cq \subseteq cut_\preceq(R)$. 
\item $\cqr \subseteq \cq$. 
\item If $\cq \models_s R$, then $\cqr = \cq$.
\end{enumerate}

\begin{proof}
\\
a) -- d) Follow directly from Definition~\ref{dfn:cut}. \\
e) Follows directly from~\ref{lem:q-implies-r-then-cq-subset-cr}). \\
f) Let $\cq \models_s R$. It follows that $SE(\cq) \cap SE(Q) \subseteq SE(R)$, which implies $SE(\cq) \cap SE(Q) \cap SE(R) \neq \es$ since $SE(\cq) \cap SE(Q) \neq \es$ by Definition~\ref{dfn:cut}. We can rewrite this as $SE(\cq) \cap SE(Q + R) \neq \es$. Thus, $\cq \subseteq \cqr$ by Definition~\ref{dfn:cut}. By Lemma~\ref{lem:cut-properties} \ref{lem:cqr-subset-cq}), we obtain $\cq = \cqr$.
\end{proof}

\textsc{Proposition~\ref{prop:p-rev-q-sequiv-cq-plus-q}}:
Let $P,Q \in \lpa$ and $\preceq$ be an ensconcement associated with $P$. Then $SE(P \reve Q) = SE(\cq + Q)$.

\begin{proof}
Let $P,Q \in \lpa$ and $\preceq$ be an ensconcement associated with $P$. If $Q$ is not satisfiable, then by Definition~\ref{dfn:reve} we have $SE(P \reve Q) = SE(P + Q) = \es$ and by Lemma~\ref{lem:cut-properties} \ref{lem:cq-equal-es}) we also have $SE(\cq + Q) = SE(Q) = \es$. Otherwise, by Definition~\ref{dfn:reve}, $P \reve Q = \cq \cup (P \reve Q) \setminus (\cq + Q) \cup Q$, which means $SE(P \reve Q) = SE(\cq) \cap SE((P \reve Q) \setminus (\cq + Q)) \cap SE(Q)$. Since for any $r \in (P \reve Q) \setminus (\cq + Q): SE(\cq) \cap SE(Q) \subseteq SE(r)$, we obtain $SE(P \reve Q) = SE(\cq) \cap SE(Q)$.
\end{proof}

\textsc{Theorem~\ref{thm:rev-ens}}:
The revision operator $\reve$ satisfies ($*$1)--($*$6) and ($*$8).

\begin{proof}
\\
($*$1): Follows directly from Definition~\ref{dfn:reve}. \\
($*$2): Follows directly from Definition~\ref{dfn:reve}. \\
($*$3): If $Q$ is not satisfiable, then $P \reve Q = P + Q$ by Definition~\ref{dfn:reve}. Otherwise, for any $r \in P \reve Q$ it holds that $r \in P \cup Q$, which implies $P \reve Q \subseteq P + Q$. \\
($*$4): If $P + Q$ is satisfiable, then $\cq = P$ by Lemma~\ref{lem:cut-properties} \ref{lem:cq-equal-P}). Since $SE(P) \subseteq SE(r)$ for all $r \in P$, it follows from Definition~\ref{dfn:reve} that $P \reve Q = P + Q$.\\
($*$5): If $Q$ is not satisfiable, then $P \reve Q = P + Q$ is not satisfiable. Now assume~$Q$ is satisfiable. By Lemma~\ref{lem:cut-properties} \ref{lem:cq-plus-Q-satisfiable}) it holds that $\cq + Q$ is satisfiable. Since $P\reve Q \equiv_s \cq + Q$ by Proposition~\ref{prop:p-rev-q-sequiv-cq-plus-q}, it follows that $P \reve Q$ is satisfiable. \\
($*$6): Follows directly from Definition~\ref{dfn:reve}. \\
($*$8): Let $(P \reve Q) + R$ be satisfiable. By Definition~\ref{dfn:reve}, this means $SE(\cq) \cap SE((P \reve Q) \setminus \cq) \cap SE(Q) \cap SE(R) \neq \es$ and thus $SE(\cq) \cap  SE(Q) \cap SE(R) \neq \es$. From Lemma~\ref{lem:cut-properties} \ref{lem:cqr-subset-cq}) and Definition~\ref{dfn:cut} it follows that $\cqr = \cq$. Then obviously $SE(\cqr) \cap SE(Q) = SE(\cq) \cap SE(Q)$, which implies $SE(\cqr) \cap SE(Q) \cap SE(R) \subseteq SE(\cq) \cap SE(Q)$. It thus also holds that $\{\, r \in P \mid SE(\cq) \cap SE(Q) \subseteq SE(r) \,\} \subseteq \{\, r \in P \mid SE(\cqr) \cap SE(Q+R) \subseteq SE(r) \,\}$, from which we can conclude that $(P \reve Q) + R \subseteq P \reve (Q + R)$.
\end{proof}

\textsc{Theorem~\ref{thm:reve-bb}}:
The revision operator~$\reve$ satisfies ($*$1b), ($*$2b), and ($*$5b).

\begin{proof}
\\
($*$1b): Since ($*$1b) $=$ ($*$2) and $\reve$ satisfies ($*$2), $\reve$ also satisfies ($*$1b). \\
($*$2b): Since ($*$2b) $=$ ($*$3) and $\reve$ satisfies ($*$3), $\reve$ also satisfies ($*$2b). \\
($*$5b): If $Q$ is satisfiable, then by Lemma~\ref{lem:cut-properties} \ref{lem:cq-plus-Q-satisfiable}) it holds that $\cq + Q$ is satisfiable. Since $P\reve Q \equiv_s \cq + Q$ by Proposition~\ref{prop:p-rev-q-sequiv-cq-plus-q}, it follows that $P \reve Q$ is satisfiable.
\end{proof}

\textsc{Lemma~\ref{lem:cutm-properties}}:
Let $P,Q,R \in \lpa$.
\begin{enumerate}[a)]
\item If $P \not \models_s Q$, then $\cqm = P$. 
\item If $\not \models_s Q$, then $\cqm \not \models_s Q$. 
\item If $\models_s Q$, then $\cqm = \es$.
\item If $Q \models_s R$, then $cut_\preceq^-(R) \subseteq \cqm$. 
\item $\cqm \subseteq \cqrm$.
\item If $\cqm \models_s R$, then $\cqrm = \cqm$. 
\item If $\cqm \not \models_s R$, then $\cqrm = cut^-_\preceq(R)$.
\end{enumerate}

\begin{proof}
\\
a) -- d) Follow directly from the definition of~$\cqm$. \\
e) Follows directly from~\ref{lem:q-implies-r-then-crm-subset-cqm}). \\
f) Assume $\cqrm \neq \cqm$. Then $\cqm \subset \cqrm$ by Lemma~\ref{lem:cutm-properties}~\ref{lem:cqm-subset-cqrm}). Let $r \in \cqrm \setminus \cqm$. This means $SE(\{\, r' \in P \mid r \preceq r' \,\}) \cap \overline{SE(Q+R)} \neq \es$ and $SE(\{\, r' \in P \mid r \preceq r' \,\}) \cap \seqc = \es$ by Definition of $\cqm$. Thus, $SE(\{\, r' \in P \mid r \preceq r' \,\}) \cap \overline{SE(R)} \neq \es$. Furthermore, $\cqm \subseteq \{\, r' \in P \mid r \preceq r' \,\}$ by Definition of $\cqm$. We therefore obtain $SE(\cqm) \cap \overline{SE(R)} \neq \es$, which implies $\cqm \not \models_s R$. \\
g) Assume $\cqrm \neq cut^-_\preceq(R)$. Then $cut^-_\preceq(R) \subset \cqrm$ by Lemma~\ref{lem:cutm-properties}~\ref{lem:cqm-subset-cqrm}). Let $r \in \cqrm \setminus cut^-_\preceq(R)$. This means $SE(\{\, r' \in P \mid r \preceq r' \,\}) \cap \overline{SE(Q+R)} \neq \es$ and $SE(\{\, r' \in P \mid r \preceq r' \,\}) \cap \overline{SE(R)} = \es$ by Definition of $cut^-_\preceq(R)$. Thus, $SE(\{\, r' \in P \mid r \preceq r' \,\}) \cap \seqc \neq \es$ and so $\{\, r' \in P \mid r \preceq r' \,\} \subseteq \cqm$ by Definition of $\cqm$. We therefore obtain $SE(\cqm) \cap \overline{SE(R)} = \es$, which implies $\cqm \models_s R$.
\end{proof}

\textsc{Proposition~\ref{prop:SEcon-equal-SEcutminus}}:
Let $P,Q \in \lpa$ and $\preceq$ be an ensconcement associated with $P$. Then $SE(P \cone Q) \cap \seqc = SE(\cqm) \cap \seqc$.

\begin{proof}
If $\models_s Q$, then $\seqc = \es$ and $SE(P \cone Q) \cap \seqc = SE(\cqm) \cap \seqc$. Otherwise, by Definition~\ref{dfn:cone}, $P \cone Q = \cqm \cup (P \cone Q) \setminus \cqm$, which means $SE(P \cone Q) \cap \seqc = SE(\cqm) \cap SE((P \cone Q) \setminus \cqm) \cap \seqc$. For all $r \in (P \cone Q) \setminus \cqm$ we have $SE(\cqm) \cap \seqc \subseteq SE(r)$, so that we obtain $SE(P \cone Q) \cap \seqc = SE(\cqm) \cap \seqc$.
\end{proof}

\textsc{Theorem~\ref{thm:con-ens}}:
The contraction operator~$\cone$ satisfies ($\dotminus$1)--($\dotminus$4) and ($\dotminus$6)--($\dotminus$8).

\begin{proof}
\\
($\dotminus$1): Follows directly from Definition~\ref{dfn:cone}. \\
($\dotminus$2): Follows directly from Definition~\ref{dfn:cone}. \\
($\dotminus$3): If $P \not \models_s Q$, then $\cqm = P$ by Lemma~\ref{lem:cutm-properties} \ref{lem:cqm-equal-P}). Since $SE(P) \subseteq SE(r)$ for all $r \in P$, this means $SE(\cqm) \cap \seqc \subseteq SE(r)$ for all $r \in P$ and by Definition~\ref{dfn:cone} we thus have $P \cone Q = P$. \\
($\dotminus$4): Let $\not \models_s Q$. If $P \not \models_s Q$, then $P \cone Q = P \not \models_s Q$ by~($\dotminus$3). Now assume $P \models_s Q$ and let $S = SE(\cqm) \cap \seqc$. For each $r \in P$, if $r \in P \cone Q$, then $S \subseteq SE(r)$ by Definition~\ref{dfn:cone}, and thus $S \subseteq SE(P \cone Q)$. Since $S \cap SE(Q) = \emptyset$, we obtain $P \cone Q \not \models_s Q$. \\
($\dotminus$6): Follows directly from Definition~\ref{dfn:cone}. \\
($\dotminus$7): For all $r \in (P \cone Q) \cap (P \cone R): SE(\cqm) \cap \seqc \subseteq SE(r)$ and $SE(cut_{\preceq}^-(R)) \cap \overline{SE(R)} \subseteq SE(r)$. This implies $SE(cut_{\preceq}^-(Q+R)) \cap \seqc \subseteq SE(r)$ since $\cqm \subseteq cut_{\preceq}^-(Q+R)$ by Lemma~\ref{lem:cutm-properties}~\ref{lem:cqm-subset-cqrm}) and $SE(cut_{\preceq}^-(Q+R)) \cap \overline{SE(R)} \subseteq SE(r)$ since $cut_{\preceq}^-(R) \subseteq cut_{\preceq}^-(Q+R)$ by Lemma~\ref{lem:cutm-properties}~\ref{lem:cqm-subset-cqrm}). From $\overline{SE(Q+R)} = \seqc \cup \overline{SE(R)}$ we obtain $SE(cut_{\preceq}^-(Q+R)) \cap \overline{SE(Q+R)} \subseteq SE(r)$ and thus $r \in P \cone (Q+R)$ by Definition~\ref{dfn:cone}. \\
($\dotminus$8): Assume $SE(P \cone (Q+R)) \nsubseteq SE(Q)$. Then, $SE(cut_{\preceq}^-(Q+R)) \cap SE((P \cone (Q+R)) \setminus cut_{\preceq}^-(Q+R)) \nsubseteq SE(Q)$, which means $SE(cut_{\preceq}^-(Q+R)) \cap \seqc \neq \es$ (i). Recall that $\cqm$ is maximal and $\cqm \subseteq cut_{\preceq}^-(Q+R)$ (ii) by Lemma~\ref{lem:cutm-properties}~\ref{lem:cqm-subset-cqrm}). From (i) and (ii) it follows that $\cqm = cut_{\preceq}^-(Q+R)$. Since $\seqc \subseteq \overline{SE(Q+R)}$, we have $SE(\cqm) \cap \seqc \subseteq SE(cut_{\preceq}^-(Q+R)) \cap \overline{SE(Q+R)}$, which implies $P \cone (Q+R) \subseteq P \cone Q$ by Definition~\ref{dfn:cone}.
\end{proof}

\textsc{Theorem~\ref{thm:cone-bb}}:
The contraction operator $\cone$ satisfies ($-$1b), ($-$2b) and ($-$5b)--($-$8b).

\begin{proof}
\\
($\dotminus$1b): Follows from ($\dotminus$1b) $=$ ($\dotminus$2) and satisfaction of ($\dotminus$2). \\
($\dotminus$2b): Follows from ($\dotminus$2b) $=$ ($\dotminus$4) and satisfaction of ($\dotminus$4). \\
($\dotminus$5b): Follows from ($\dotminus$5b) $=$ ($\dotminus$3) and satisfaction of ($\dotminus$3). \\
($\dotminus$6b): Follows from ($\dotminus$6b) $=$ ($\dotminus$6) and satisfaction of ($\dotminus$6). \\
($\dotminus$7b): If $\models_s Q + R$, then $P \cone (Q + R) = P$ by Definition~\ref{dfn:cone} and $\models_s Q$ and $\models_s R$, which means $P \cone Q = P$ and $P \cone R = P$ by Definition~\ref{dfn:cone}. Now let $\not \models_s Q + R$. We proceed by cases.

Case 1: $\cqm \models_s R$. Then $\cqrm = \cqm$ by Lemma~\ref{lem:cutm-properties}~\ref{lem:cqrm-equal-cqm}). Let $r \in P \cone (Q+R)$. This means $r \in P$ by ($\dotminus$2) and $SE(\cqrm) \cap \seqrc \subseteq SE(r)$ by Definition~\ref{dfn:cone}. It follows that $SE(\cqrm) \cap \seqc \subseteq SE(r)$. Due to the case assumption, we obtain $SE(\cqm) \cap \seqc \subseteq SE(r)$ and thus $r \in P \cone Q$ by Definition~\ref{dfn:cone}. Now let $r \in P \cone Q$. This means $r \in P$ by ($\dotminus$2) and $SE(\cqm) \cap \seqc \subseteq SE(r)$ by Definition~\ref{dfn:cone}. Then $SE(\cqrm) \cap \seqc \subseteq SE(r)$ due to the case assumption. It also follows from the case assumption that $SE(\cqm) \cap \overline{SE(R)} = \es$ and $SE(\cqrm) \cap \overline{SE(R)} = \es$. We thus have $SE(\cqrm) \cap (\seqc \cup \overline{SE(R)}) \subseteq SE(r)$, that is, $SE(\cqrm) \cap \seqrc \subseteq SE(r)$. Therefore, $r \in P \cone (Q + R)$ by Definition~\ref{dfn:cone}.

Case 2: $cut_\preceq^-(R) \models_s Q$. Follows analogous to Case 1 so that $P \cone (Q + R) = P \cone R$.

Case 3: $\cqm \not \models_s R$ and $cut_\preceq^-(R) \not \models_s Q$. Then $\cqrm = \cqm = cut_\preceq^-(R)$ by Lemma~\ref{lem:cutm-properties}~\ref{lem:cqrm-equal-crm}). Let $r \in P \cone (Q + R)$. This means $r \in P$ by ($\dotminus$2) and $SE(\cqrm) \cap \seqrc \subseteq SE(r)$ by Definition~\ref{dfn:cone}. We thus have $SE(\cqrm) \cap \seqc \subseteq SE(r)$ and $SE(\cqrm) \cap \overline{SE(R)} \subseteq SE(r)$. From the case assumption it follows that $SE(\cqm) \cap \seqc \subseteq SE(r)$ and $SE(cut_\preceq^-(R)) \cap \overline{SE(R)} \subseteq SE(r)$. This means $r \in P \cone Q$ and $r \in P \cone R$ by Definition~\ref{dfn:cone} and therefore $r \in (P \cone Q) \cap (P \cone R)$.
\\
($\dotminus$8b): Assume $SE(P \cone Q) \subseteq SE(Q) \cup SE(r)$, that is, $SE(P \cone Q) \cap \seqc \subseteq SE(r)$. By Proposition~\ref{prop:SEcon-equal-SEcutminus}, $SE(P \cone Q) \cap \seqc = SE(\cqm) \cap \seqc$, so that we obtain $SE(\cqm) \cap \seqc \subseteq SE(r)$. This implies $r \in P \cone Q$ by Definition~\ref{dfn:cone}. We can conclude $r \in P$ by~($\dotminus$2).
\end{proof}

\textsc{Theorem~\ref{thm:revpm-reve}}:
Let $P,Q \in \lpa$. For any selection function $\gamma$, there exists an ensconcement $\preceq$ associated with $P$ such that $P \revpm Q = P \reve Q$.

\begin{proof}
Let $P,Q \in \lpa$ and $\gamma$ be a selection function that determines the outcome of $P \revpm Q$. By $S = (P \revpm Q) \cap P = \bigcap \gpq$ we denote the subset of $P$ that is retained in the revision and by $S' = P \setminus S$ the subset of $P$ that is discarded. We can then create an ensconcement $\preceq$ associated with $P$ that has a minimal number of levels, such that for all $r \in S$ and for all $r' \in S'$: $r' \prec r$. We now show that $(P \reve Q) \cap P = S$. Clearly, $\cq = S$ by Definition~\ref{dfn:cut}, which implies $S \subseteq (P \reve Q) \cap P$. Assume that there exists an $r' \in S'$ with $SE(\cq) \cap SE(Q) \subseteq SE(r')$. Then for each selected compatible set $R \in \gpq$ it would hold that $r' \in R$ because $R$ is maximal by the definition of $\pq$. Yet this implies $r' \in S$, a contradiction.
\end{proof}

\textsc{Theorem~\ref{thm:conpm-cone}}:
Let $P,Q \in \lpa$. For any selection function $\gamma$, there exists an ensconcement $\preceq$ associated with $P$ such that $P \conpm Q = P \cone Q$.

\begin{proof}
Follows analogously to the proof of Theorem~\ref{thm:revpm-reve}.
\end{proof}

%\textsc{Lemma~\ref{lem:subset-not-above-rule}}:
%Let $\preceq^R$ be a subset-ensconcement associated with some $P \in \lpa$ and $R \subseteq P$. For any rule $r \in R$, it holds that $R \preceq^R \{r\}$.
\begin{lem} \label{lem:subset-not-above-rule}
Let $\preceq^R$ be a subset-ensconcement associated with some $P \in \lpa$ and $R \subseteq P$. For any rule $r \in R$, it holds that $R \preceq^R \{r\}$.
\end{lem}

\begin{proof}
Since $R \models_s \{r\}$, it follows from Conditions~($\preceq^R$1) and~($\preceq^R$2) that $\{r\} \not \prec^R R$.
\end{proof}

\textsc{Theorem~\ref{thm:subset-equal-rule-based-change}}:
Let $P,Q \in \lpa$, $\preceq$ be an ensconcement associated with $P$, and $\preceq^R$ a subset-ensconcement associated with~$P$ such that $\{r\} \preceq^R \{r'\}$ iff $r \preceq r'$ for all $r,r' \in P$. Then $P \reve Q = P \reveR Q$ (or $P \cone Q = P \coneR Q$, alternatively).

\begin{proof}
Let $P,Q \in \lpa$, $\preceq$ an ensconcement associated with $P$, and $\preceq^R$ a subset-ensconcement associated with $P$. Assume that $\{r\} \preceq^R \{r'\}$ iff $r \preceq r'$ for all $r,r' \in P$. From Lemma~\ref{lem:subset-not-above-rule} it is clear that $\cq = cut_{\preceq^R}(Q)$, which implies for all $r \in (P \reve Q)\setminus Q: SE(cut_{\preceq^R}(Q)) \cap SE(Q) \subseteq SE(\{r\})$, and thus $P \reve Q \subseteq P \reveR Q$. Since $SE(R) \subseteq SE(r)$ for any $R \subseteq P$ and each $r \in R$, we also have $P \reveR Q \subseteq P \reve Q$.

Analogous for contraction.
%Sketch: Lemma~\ref{lem:subset-not-above-rule} implies $P \reve Q \subseteq P \rever Q$. Since $SE(R) \subseteq SE(r)$ for any $R \subseteq P$ and each $r \in R$, we also have $P \rever Q \subseteq P \reve Q$. 
\end{proof}

\textsc{Proposition~\ref{prop:revpm-via-conpm}}:
Let $P \in \lpa$, $\gamma$ be a selection function for $P$, and $*$ an operator for $P$ such that for any $Q \in \lpa$: $P * Q = (P \conpm \overline{Q}) + Q$. Then $P * Q = P \revpm Q$.

\begin{proof}
Let $\seqc = \se$. Then $P \conpm \overline{Q} = P$ by Definition~\ref{dfn:conpm} and thus $P * Q = (P \conpm \overline{Q}) + Q = P + Q = P \revpm Q$ by Definition~\ref{dfn:revpm}. Otherwise, $\seqc \neq \se$ such that $\mathbb{P}^-_{\overline{Q}} = \{\, R \subseteq P \,  \mid \, SE(R) \cap SE(Q) \neq \es$ and for all $R'$ with $R \subset R' \subseteq P : SE(R') \cap SE(Q) = \es \,\} = \pq$ by Definition~\ref{dfn:compatible-set}. It follows that $P * Q = (P \conpm \overline{Q}) + Q = \bigcap \gpq + Q = P \revpm Q$ by Definitions~\ref{dfn:revpm} and~\ref{dfn:conpm}.
\end{proof}

\textsc{Proposition~\ref{prop:conpm-via-revpm}}:
Let $P \in \lpa$, $\gamma$ be a selection function for $P$, and $\dotminus$ an operator for $P$ such that for any $Q \in \lpa$: $P \dotminus Q = P \cap (P \revpm \overline{Q})$. Then $P \dotminus Q = P \conpm Q$.

\begin{proof}
Let $\seqc = \es$. Then $P \revpm \overline{Q} = P + \overline{Q}$ by Definition~\ref{dfn:revpm} and thus $P \dotminus Q = P \cap (P \revpm \overline{Q}) = P = P \conpm Q$ by Definition~\ref{dfn:conpm}. Otherwise, $\seqc \neq \es$ such that $\mathbb{P}_{\overline{Q}} = \{\, R \subseteq P \,  \mid \, SE(R) \cap \seqc \neq \es$ and for all $R'$ with $R \subset R' \subseteq P : SE(R') \cap \seqc = \es \,\} = \pqm$ by Definition~\ref{dfn:remainder-set}. It follows that $P \dotminus Q = P \cap (P \revpm \overline{Q}) = P \cap (\bigcap \gpqm + \overline{Q})$ by Definition~\ref{dfn:revpm}. Assume there exists a rule $r \in \overline{Q}$ with $r \in P \setminus \bigcap \gamma(\pqm)$. Then there exists an $R \in \gpqm: r \not \in R$, a contradiction since $SE(\overline{Q})= \seqc \subseteq SE(r)$ and $R$ is maximal. We therefore obtain $P \dotminus Q = P \conpm Q$ by Definition~\ref{dfn:conpm}.
\end{proof}

\textsc{Proposition~\ref{prop:reve-via-cone}}:
Let $P \in \lpa$, $\preceq$ be an ensconcement associated with $P$, and $*$ an operator for $P$ such that for any $Q \in \lpa$: $P * Q = (P \cone \overline{Q}) + Q$. Then $P * Q = P \reve Q$.

\begin{proof}
Let $\seqc = \se$. Then $P \cone \overline{Q} = P$ by Definition~\ref{dfn:cone} and thus $P * Q = (P \cone \overline{Q}) + Q = P + Q = P \reve Q$ by Definition~\ref{dfn:reve}. Otherwise, $\seqc \neq \se$ such that $cut^-_\preceq(\overline{Q}) = \{\, r \in P \mid SE\left(\{\, r' \in P \mid r \preceq r' \,\}\right) \cap SE(Q) \neq \es \,\} = \cq$ by Definition~\ref{dfn:cut}. It follows that $P * Q = (P \cone \overline{Q}) + Q = \{\, r \in P \mid SE(\cq) \cap SE(Q) \subseteq SE(r) \,\} + Q = P \reve Q$ by Definitions~\ref{dfn:reve} and~\ref{dfn:cone}.
\end{proof}

\textsc{Proposition~\ref{prop:cone-via-reve}}:
Let $P \in \lpa$, $\preceq$ be an ensconcement associated with $P$, and $\dotminus$ an operator for $P$ such that for any $Q \in \lpa$: $P \dotminus Q = P \cap (P \reve \overline{Q})$. Then $P \dotminus Q = P \cone Q$.

\begin{proof}
Let $\seqc = \es$. Then $P \reve \overline{Q} = P + \overline{Q}$ by Definition~\ref{dfn:reve} and thus $P \dotminus Q = P \cap (P \reve \overline{Q}) = P = P \cone Q$ by Definition~\ref{dfn:cone}. Otherwise, $\seqc \neq \es$ such that $cut_\preceq(\overline{Q}) = \{\, r \in P \mid SE\left(\{\, r' \in P \mid r \preceq r' \,\}\right) \cap \seqc \neq \es \,\} = \cqm$ by definition of $\cqm$. It follows that $P \dotminus Q = P \cap (P \reve \overline{Q}) = P \cap (\{\, r \in P \mid SE(\cqm) \cap \seqc \subseteq SE(r) \,\} + \overline{Q})$ by Definition~\ref{dfn:reve}. Assume there exists a rule $r' \in \overline{Q}$ with $r' \in P \setminus \{\, r \in P \mid SE(\cqm) \cap \seqc \subseteq SE(r) \,\}$. Since $SE(\overline{Q}) = \seqc \subseteq SE(r')$, it holds that $SE(\cqm) \cap \seqc \subseteq SE(r')$. This implies $r' \in \{\, r \in P \mid SE(\cqm) \cap \seqc \subseteq SE(r) \,\}$, a contradiction. We therefore obtain $P \dotminus Q = P \cone Q$ by Definition~\ref{dfn:cone}.
\end{proof}

\textsc{Proposition~\ref{prop:conflict}}:
Let $P,Q \in \lpa$ and $SE(P) \neq \es \neq SE(Q)$. For any $R \subseteq P$, if $SE(R) \cap SE(Q) = \emptyset$ and for all $R' \subset R: SE(R') \cap SE(Q) \neq \emptyset$, then there exists $\mathbb{M} \in 2^{\mpq}$ such that $R \subseteq \bigcup \mathbb{M}$.

\begin{proof}
Let $P,Q$ be satisfiable logic programs and $R \subseteq P$ such that $SE(R) \cap SE(Q) = \emptyset$ and for each $R' \subset R: SE(R') \cap SE(Q) \neq \emptyset$. Then there exists some $a_j \in \mathcal{A}$ such that $a_j \in At(Q)$ and there exist one or more rules $r_i \in R$ for each $a_j$ such that $a_j \in At(r_i)$. For each $r_i$, there exists a corresponding $r_i$-module $M(P)^{r_i}|_{a_j}$ including $a_j$, such that $r_i \in M(P)^{r_i}|_{a_j}$. It follows from Definition~\ref{dfn:mod} that for all remaining rules $r' \in R \setminus {r_i} : r' \in \bigcup_{i,j} M(P)^{r_i}|_{a_j}$.
\end{proof}

\textsc{Corollary~\ref{cor:conflict}}:
Let $P,Q \in \lpa$ and $SE(P) \neq \es$. Then $SE(P) \cap SE(Q) = \emptyset$ if and only if $SE \left(\bigcup \mpq \right) \cap SE(Q) = \emptyset$.

\begin{proof}
\lq\lq If\rq\rq : Since $SE(P) \subseteq SE\left(\bigcup \mpq \right)$, if $SE\left(\bigcup \mpq \right) \cap SE(Q) = \emptyset$, then also $SE(P) \cap SE(Q) = \emptyset$.

\lq\lq Only if\rq\rq : Follows from Proposition~\ref{prop:conflict} if $Q$ is satisfiable. Trivial if $Q$ is not satisfiable.
\end{proof}

\textsc{Theorem~\ref{thm:alg-same-as-op}}:
For any $P,Q \in \lpa$, there exists a selection function $\gamma$ for $P$ such that $P \revpm Q = P \setminus \mpq + \bigcup \mpq^{\revpm} + Q$ (or $P \conpm Q = P \setminus \mpq + \bigcup \mpq^{\conpm}$, respectively).

\begin{proof}
To prove the equation for revision, we need to show that $P \setminus \mpq + \bigcup \mpq^{\revpm} = \bigcap \gpq$ for some $\gamma$. Let $Z \subseteq P$ be the set of rules that are eliminated during the operation of $\revpm$ per Definition~\ref{dfn:revpm}, i.e., $P \revpm Q = \bigcap \gpq + Q = P \setminus Z + Q$, and let $Z' \subseteq P$ be the set of rules that are eliminated by \textsc{ModChange}.

We first show that $Z \subseteq Z'$. Assume that $Z' = \es$ until the last iteration of the while-loop. In the last iteration, we have $n = \lvert \mpq \rvert$ and \textsc{ModChange} computes $\bigcup \mpq \revpm Q = \bigcup \mpq^{\revpm}$. Thus, $P \setminus \mpq + \bigcup \mpq^{\revpm} = P \setminus \mpq + (\bigcup \mpq \revpm Q)$. Let $\mathbb{M}_Q$ denote the set $\{\, R \subseteq \bigcup \mpq \mid SE(R) \cap SE(Q) \neq \es \text{ and, for all } R', R \subset R' \subseteq P \text{ implies } SE(R') \cap SE(Q) = \es \,\}$. If it holds for all $R \in \gpq$ that $R \cap \bigcup \mpq \in \gamma(\mathbb{M}_Q)$, then $P \setminus \mpq + (\bigcup \mpq \revpm Q) = ((P \setminus \mpq) \cup \bigcup \mpq) \revpm Q = P \revpm Q$, which implies $Z = Z'$.

We now show that $Z' \subseteq Z$. Assume that each revision operation in the following is the most restrictive type, that is, for any set $M$, $\gamma(M) = M$. Thus, if $r \in \bigcap \gpq$, then $r \in R$ for all $R \in \pq$. For each $\mathbb{M}$ as specified in the outer $\mathtt{foreach}$ loop of \textsc{ModChange}, let $z'$ be the set of rules eliminated during the revision of $\bigcup \mathbb{M}$ by $Q$: $z' = \bigcup \mathbb{M} \setminus ((\bigcup \mathbb{M} \revpm Q)\setminus Q)$. From $SE(P) \subseteq SE(\bigcup \mathbb{M})$ it then follows that $z' \cap \bigcap \gpq = \es$. Since $\bigcup z' = Z'$, we obtain $Z' \cap \bigcap \gpq = \es$.

Analogous for contraction.
\end{proof}

\textsc{Corollary~\ref{cor:ens-alg}}:
For any $P,Q \in \lpa$, let $P \setminus \mpq = \{\, r \in P \mid \text{for all } M \in \mpq: r \not \in M\,\}$ and $\mpq^\circ$ denote the output of Algorithm~\ref{alg:modrev} for the inputs $\mpq$, $\circ \in \{\reve,\cone \}$, and~$Q$. Then $P \reve Q = P \setminus \mpq + \bigcup \mpq^{\reve} + Q$ (or $P \cone Q = P \setminus \mpq + \bigcup \mpq^{\cone}$, respectively) for some ensconcement $\preceq$ associated with $P$.

\begin{proof}
Follows directly from Theorems~\ref{thm:revpm-reve}, \ref{thm:conpm-cone}, and~\ref{thm:alg-same-as-op}.
\end{proof}

% Acknowledgments %%%%%%%%%%%%%%%%%%%%%%%%%%%%%%%%%%%%%%
%\begin{acks}
%XXX Acknowledgements XXX
%\end{acks}

% Bibliography
\bibliographystyle{ACM-Reference-Format-Journals}
\bibliography{Binnewies}


\begin{thebibliography}{00}

%%% the Web address.

\ifx \showCODEN    \undefined \def \showCODEN     #1{\unskip}     \fi
\ifx \showDOI      \undefined \def \showDOI       #1{{\tt DOI:}\penalty0{#1}\ }
  \fi
\ifx \showISBNx    \undefined \def \showISBNx     #1{\unskip}     \fi
\ifx \showISBNxiii \undefined \def \showISBNxiii  #1{\unskip}     \fi
\ifx \showISSN     \undefined \def \showISSN      #1{\unskip}     \fi
\ifx \showLCCN     \undefined \def \showLCCN      #1{\unskip}     \fi
\ifx \shownote     \undefined \def \shownote      #1{#1}          \fi
\ifx \showarticletitle \undefined \def \showarticletitle #1{#1}   \fi
\ifx \showURL      \undefined \def \showURL       #1{#1}          \fi

\bibitem[\protect\citeauthoryear{Alchourr{\'o}n, G{\"a}rdenfors, and
  Makinson}{Alchourr{\'o}n et~al\mbox{.}}{1985}]%
        {alchourron1985logic}
{Carlos~E. Alchourr{\'o}n}, {Peter G{\"a}rdenfors}, {and} {David Makinson}.
  1985.
\newblock \showarticletitle{On the logic of theory change: Partial meet
  contraction and revision functions}.
\newblock {\em Journal of Symbolic Logic\/} {50}, 2 (1985), 510--530.
\newblock


\bibitem[\protect\citeauthoryear{Binnewies, Zhuang, and Wang}{Binnewies
  et~al\mbox{.}}{2015}]%
        {binnewies2015partial}
{Sebastian Binnewies}, {Zhiqiang Zhuang}, {and} {Kewen Wang}. 2015.
\newblock \showarticletitle{Partial meet revision and contraction in logic
  programs}. In {\em Proceedings of the Twenty-Ninth AAAI Conference on
  Artificial Intelligence, AAAI 2015}. 1439--1445.
\newblock


\bibitem[\protect\citeauthoryear{Brewka and Eiter}{Brewka and Eiter}{1999}]%
        {brewka1999preferred}
{Gerhard Brewka} {and} {Thomas Eiter}. 1999.
\newblock \showarticletitle{Preferred answer sets for extended logic programs}.
\newblock {\em Artificial Intelligence\/} {109}, 1--2 (1999), 297--356.
\newblock


\bibitem[\protect\citeauthoryear{Colmerauer and Roussel}{Colmerauer and
  Roussel}{1996}]%
        {colmerauer1996birth}
{Alain Colmerauer} {and} {Philippe Roussel}. 1996.
\newblock \showarticletitle{The Birth of Prolog}.
\newblock In {\em History of Programming languages---{II}}. 331--367.
\newblock


\bibitem[\protect\citeauthoryear{Dalal}{Dalal}{1988}]%
        {dalal1988investigations}
{Mukesh Dalal}. 1988.
\newblock \showarticletitle{Investigations into a theory of knowledge base
  revision: Preliminary report}. In {\em Proceedings of the Seventh National
  Conference on Artificial Intelligence}. 475--479.
\newblock


\bibitem[\protect\citeauthoryear{Delgrande}{Delgrande}{2010}]%
        {delgrande2010program}
{James~P. Delgrande}. 2010.
\newblock \showarticletitle{A program-level approach to revising logic programs
  under the answer set semantics}.
\newblock {\em Theory and Practice of Logic Programming\/} {10}, Special Issue
  4--6 (2010), 565--580.
\newblock


\bibitem[\protect\citeauthoryear{Delgrande, Peppas, and Woltran}{Delgrande
  et~al\mbox{.}}{2013}]%
        {delgrande2013agm}
{James~P. Delgrande}, {Pavlos Peppas}, {and} {Stefan Woltran}. 2013.
\newblock \showarticletitle{{AGM}-style belief revision of logic programs under
  answer set semantics}.
\newblock In {\em Logic Programming and Nonmonotonic Reasoning}. Lecture Notes
  in Computer Science, Vol. 8148. 264--276.
\newblock


\bibitem[\protect\citeauthoryear{Delgrande, Schaub, and Tompits}{Delgrande
  et~al\mbox{.}}{2003}]%
        {delgrande2003framework}
{James~P. Delgrande}, {Torsten Schaub}, {and} {Hans Tompits}. 2003.
\newblock \showarticletitle{A framework for compiling preferences in logic
  programs}.
\newblock {\em Theory and Practice of Logic Programming\/} {3}, 2 (2003),
  129--187.
\newblock


\bibitem[\protect\citeauthoryear{Delgrande, Schaub, and Tompits}{Delgrande
  et~al\mbox{.}}{2007}]%
        {delgrande2007preference}
{James~P. Delgrande}, {Torsten Schaub}, {and} {Hans Tompits}. 2007.
\newblock \showarticletitle{A preference-based framework for updating logic
  programs}.
\newblock In {\em Logic Programming and Nonmonotonic Reasoning}. Lecture Notes
  in Computer Science, Vol. 4483. 71--83.
\newblock


\bibitem[\protect\citeauthoryear{Delgrande, Schaub, Tompits, and
  Wang}{Delgrande et~al\mbox{.}}{2004}]%
        {delgrande2004classification}
{James~P. Delgrande}, {Torsten Schaub}, {Hans Tompits}, {and} {Kewen Wang}.
  2004.
\newblock \showarticletitle{A classification and survey of preference handling
  approaches in nonmonotonic reasoning}.
\newblock {\em Computational Intelligence\/} {20}, 2 (2004), 308--334.
\newblock


\bibitem[\protect\citeauthoryear{Delgrande, Schaub, Tompits, and
  Woltran}{Delgrande et~al\mbox{.}}{2013}]%
        {delgrande2013model}
{James~P. Delgrande}, {Torsten Schaub}, {Hans Tompits}, {and} {Stefan Woltran}.
  2013.
\newblock \showarticletitle{A model-theoretic approach to belief change in
  answer set programming}.
\newblock {\em ACM Transactions on Computational Logic\/} {14}, 2 (2013),
  14:1--14:46.
\newblock


\bibitem[\protect\citeauthoryear{Doyle}{Doyle}{1979}]%
        {doyle1979truth}
{Jon Doyle}. 1979.
\newblock \showarticletitle{A truth maintenance system}.
\newblock {\em Artificial Intelligence\/} {12}, 3 (1979), 231--272.
\newblock


\bibitem[\protect\citeauthoryear{Eiter, Fink, P{\"u}hrer, Tompits, and
  Woltran}{Eiter et~al\mbox{.}}{2013}]%
        {eiter2013model}
{Thomas Eiter}, {Michael Fink}, {J{\"o}rg P{\"u}hrer}, {Hans Tompits}, {and}
  {Stefan Woltran}. 2013.
\newblock \showarticletitle{Model-based recasting in answer-set programming}.
\newblock {\em Journal of Applied Non-Classical Logics\/} {23}, 1--2 (2013),
  75--104.
\newblock


\bibitem[\protect\citeauthoryear{Eiter, Fink, Sabbatini, and Tompits}{Eiter
  et~al\mbox{.}}{2002}]%
        {eiter2002properties}
{Thomas Eiter}, {Michael Fink}, {Giuliana Sabbatini}, {and} {Hans Tompits}.
  2002.
\newblock \showarticletitle{On properties of update sequences based on causal
  rejection}.
\newblock {\em Theory and Practice of Logic Programming\/} {2}, 6 (2002),
  711--767.
\newblock


\bibitem[\protect\citeauthoryear{Eiter, Fink, Tompits, and Woltran}{Eiter
  et~al\mbox{.}}{2004}]%
        {eiter2004simplifying}
{Thomas Eiter}, {Michael Fink}, {Hans Tompits}, {and} {Stefan Woltran}. 2004.
\newblock \showarticletitle{Simplifying logic programs under uniform and strong
  equivalence}.
\newblock In {\em Logic Programming and Nonmonotonic Reasoning}. Lecture Notes
  in Computer Science, Vol. 2923. 87--99.
\newblock


\bibitem[\protect\citeauthoryear{Fagin, Ullman, and Vardi}{Fagin
  et~al\mbox{.}}{1983}]%
        {fagin1983semantics}
{Ronald Fagin}, {Jeffrey~D. Ullman}, {and} {Moshe~Y. Vardi}. 1983.
\newblock \showarticletitle{On the semantics of updates in databases}. In {\em
  Proceedings of the 2nd {ACM} {SIGACT}-{SIGMOD} Symposium on Principles of
  Database Systems} {\em ({PODS} '83)}. 352--365.
\newblock


\bibitem[\protect\citeauthoryear{Ferm{\'{e}}, Krevneris, and Reis}{Ferm{\'{e}}
  et~al\mbox{.}}{2008}]%
        {ferme2008axiomatic}
{Eduardo Ferm{\'{e}}}, {Mart{\'{\i}}n Krevneris}, {and} {Maur{\'{\i}}cio Reis}.
  2008.
\newblock \showarticletitle{An axiomatic characterization of ensconcement-based
  contraction}.
\newblock {\em Journal of Logic and Computation\/} {18}, 5 (2008), 739--753.
\newblock


\bibitem[\protect\citeauthoryear{Fuhrmann}{Fuhrmann}{1991}]%
        {fuhrmann1991theory}
{Andr{\'e} Fuhrmann}. 1991.
\newblock \showarticletitle{Theory contraction through base contraction}.
\newblock {\em Journal of Philosophical Logic\/} {20}, 2 (1991), 175--203.
\newblock


\bibitem[\protect\citeauthoryear{Fuhrmann and Hansson}{Fuhrmann and
  Hansson}{1994}]%
        {fuhrmann1994survey}
{Andr{\'e} Fuhrmann} {and} {Sven~Ove Hansson}. 1994.
\newblock \showarticletitle{A survey of multiple contractions}.
\newblock {\em Journal of Logic, Language and Information\/} {3}, 1 (1994),
  39--75.
\newblock


\bibitem[\protect\citeauthoryear{G{\"a}rdenfors}{G{\"a}rdenfors}{1981}]%
        {gardenfors1981epistemic}
{Peter G{\"a}rdenfors}. 1981.
\newblock \showarticletitle{An epistemic approach to conditionals}.
\newblock {\em American Philosophical Quarterly\/} {18}, 3 (1981), 203--211.
\newblock


\bibitem[\protect\citeauthoryear{G{\"a}rdenfors}{G{\"a}rdenfors}{1988}]%
        {gardenfors1988knowledge}
{Peter G{\"a}rdenfors}. 1988.
\newblock {\em Knowledge in Flux: Modeling the Dynamics of Epistemic States}.
\newblock MIT Press.
\newblock


\bibitem[\protect\citeauthoryear{Gelfond and Lifschitz}{Gelfond and
  Lifschitz}{1988}]%
        {gelfond1988stable}
{Michael Gelfond} {and} {Vladimir Lifschitz}. 1988.
\newblock \showarticletitle{The stable model semantics for logic programming}.
  In {\em Proceedings of the Fifth International Conference on Logic
  Programming}. 1070--1080.
\newblock


\bibitem[\protect\citeauthoryear{Hansson}{Hansson}{1989}]%
        {hansson1989new}
{Sven~Ove Hansson}. 1989.
\newblock \showarticletitle{New operators for theory change}.
\newblock {\em Theoria\/} {55}, 2 (1989), 114--132.
\newblock
\showISSN{1755-2567}


\bibitem[\protect\citeauthoryear{Hansson}{Hansson}{1991}]%
        {hansson1991belief}
{Sven~Ove Hansson}. 1991.
\newblock \showarticletitle{Belief contraction without recovery}.
\newblock {\em Studia Logica\/} {50}, 2 (1991), 251--260.
\newblock
\showISSN{00393215, 15728730}


\bibitem[\protect\citeauthoryear{Hansson}{Hansson}{1993}]%
        {hansson1993reversing}
{Sven~Ove Hansson}. 1993.
\newblock \showarticletitle{Reversing the Levi identity}.
\newblock {\em Journal of Philosophical Logic\/} {22}, 6 (1993), 637--669.
\newblock


\bibitem[\protect\citeauthoryear{Hansson}{Hansson}{1997}]%
        {hansson1997semi}
{Sven~Ove Hansson}. 1997.
\newblock \showarticletitle{Semi-revision}.
\newblock {\em Journal of Applied Non-Classical Logics\/} {7}, 1--2 (1997),
  151--175.
\newblock


\bibitem[\protect\citeauthoryear{Hansson}{Hansson}{1999}]%
        {hansson1999textbook}
{Sven~Ove Hansson}. 1999.
\newblock {\em {A} {T}extbook of {B}elief {D}ynamics. {T}heory {C}hange and
  {D}atabase {U}pdating}.
\newblock {K}luwer.
\newblock


\bibitem[\protect\citeauthoryear{Hansson and Wassermann}{Hansson and
  Wassermann}{2002}]%
        {hansson2002local}
{Sven~Ove Hansson} {and} {Renata Wassermann}. 2002.
\newblock \showarticletitle{Local change}.
\newblock {\em Studia Logica\/} {70}, 1 (2002), 49--76.
\newblock


\bibitem[\protect\citeauthoryear{Harper}{Harper}{1976}]%
        {harper1976rational}
{William~L. Harper}. 1976.
\newblock \showarticletitle{Rational conceptual change}.
\newblock {\em PSA: Proceedings of the Biennial Meeting of the Philosophy of
  Science Association\/}  {Two: Symposia and Invited Papers} (1976), 462--494.
\newblock


\bibitem[\protect\citeauthoryear{Inoue and Sakama}{Inoue and Sakama}{2004}]%
        {inoue2004equivalence}
{Katsumi Inoue} {and} {Chiaki Sakama}. 2004.
\newblock \showarticletitle{Equivalence of logic programs under updates}.
\newblock In {\em Logics in Artificial Intelligence}. Lecture Notes in Computer
  Science, Vol. 3229. 174--186.
\newblock


\bibitem[\protect\citeauthoryear{Katsuno and Mendelzon}{Katsuno and
  Mendelzon}{1991}]%
        {katsuno1991propositional}
{Hirofumi Katsuno} {and} {Alberto~O. Mendelzon}. 1991.
\newblock \showarticletitle{Propositional knowledge base revision and minimal
  change}.
\newblock {\em Artificial Intelligence\/} {52}, 3 (1991), 263--294.
\newblock


\bibitem[\protect\citeauthoryear{Katsuno and Mendelzon}{Katsuno and
  Mendelzon}{1992}]%
        {katsuno1992difference}
{Hirofumi Katsuno} {and} {Alberto~O. Mendelzon}. 1992.
\newblock \showarticletitle{On the difference between updating a knowledge base
  and revising it}.
\newblock In {\em Belief revision}, {Peter G{\"a}rdenfors} (Ed.). Chapter~7,
  183--203.
\newblock


\bibitem[\protect\citeauthoryear{Kowalski}{Kowalski}{1974}]%
        {kowalski1974predicate}
{Robert Kowalski}. 1974.
\newblock \showarticletitle{Predicate logic as a programming language}. In {\em
  Proceedings of the IFIP Congress}. 569--574.
\newblock


\bibitem[\protect\citeauthoryear{Kr{\"u}mpelmann and
  Kern-Isberner}{Kr{\"u}mpelmann and Kern-Isberner}{2012}]%
        {krumpelmann2012belief}
{Patrick Kr{\"u}mpelmann} {and} {Gabriele Kern-Isberner}. 2012.
\newblock \showarticletitle{Belief base change operations for answer set
  programming}.
\newblock In {\em Logics in Artificial Intelligence}. Lecture Notes in Computer
  Science, Vol. 7519. 294--306.
\newblock


\bibitem[\protect\citeauthoryear{Leite and Pereira}{Leite and Pereira}{1998}]%
        {leite1998generalizing}
{Jo{\~a}o~Alexandre Leite} {and} {Lu{\'i}s~Moniz Pereira}. 1998.
\newblock \showarticletitle{Generalizing updates: From models to programs}.
\newblock In {\em Logic Programming and Knowledge Representation}. Lecture
  Notes in Computer Science, Vol. 1471. 224--246.
\newblock


\bibitem[\protect\citeauthoryear{Levi}{Levi}{1977}]%
        {levi1977subjunctives}
{Isaac Levi}. 1977.
\newblock \showarticletitle{Subjunctives, dispositions and chances}.
\newblock In {\em Dispositions}. Synthese Library, Vol. 113. 303--335.
\newblock


\bibitem[\protect\citeauthoryear{Levi}{Levi}{1980}]%
        {levi1980enterprise}
{Isaac Levi}. 1980.
\newblock {\em The Enterprise of Knowledge: An Essay on Knowledge, Credal
  Probability, and Chance}.
\newblock MIT Press.
\newblock


\bibitem[\protect\citeauthoryear{Lifschitz, Pearce, and Valverde}{Lifschitz
  et~al\mbox{.}}{2001}]%
        {lifschitz2001strongly}
{Vladimir Lifschitz}, {David Pearce}, {and} {Agust{\'i}n Valverde}. 2001.
\newblock \showarticletitle{Strongly equivalent logic programs}.
\newblock {\em ACM Transactions on Computational Logic\/} {2}, 4 (2001),
  526--541.
\newblock


\bibitem[\protect\citeauthoryear{Lloyd}{Lloyd}{1987}]%
        {lloyd1987foundations}
{John~W. Lloyd}. 1987.
\newblock {\em Foundations of Logic Programming}.
\newblock Springer-Verlag New York.
\newblock


\bibitem[\protect\citeauthoryear{Makinson}{Makinson}{1987}]%
        {makinson1987status}
{David Makinson}. 1987.
\newblock \showarticletitle{On the status of the postulate of recovery in the
  logic of theory change}.
\newblock {\em Journal of Philosophical Logic\/} {16}, 4 (1987), 383--394.
\newblock


\bibitem[\protect\citeauthoryear{Makinson}{Makinson}{1997}]%
        {makinson1997screened}
{David Makinson}. 1997.
\newblock \showarticletitle{Screened revision}.
\newblock {\em Theoria\/} {63}, 1--2 (1997), 14--23.
\newblock


\bibitem[\protect\citeauthoryear{McCarthy}{McCarthy}{1958}]%
        {mccarthy1958programs}
{John McCarthy}. 1958.
\newblock \showarticletitle{Programs with common sense}. In {\em Proceedings of
  the Symposium on Mechanisation of Thought Processes}. 77--84.
\newblock


\bibitem[\protect\citeauthoryear{Nayak}{Nayak}{1994}]%
        {nayak1994foundational}
{Abhaya~C. Nayak}. 1994.
\newblock \showarticletitle{Foundational belief change}.
\newblock {\em Journal of Philosophical Logic\/} {23}, 5 (1994), 495--533.
\newblock
\showISSN{1573-0433}


\bibitem[\protect\citeauthoryear{Nieder{\'e}e}{Nieder{\'e}e}{1991}]%
        {niederee1991multiple}
{Reinhard Nieder{\'e}e}. 1991.
\newblock \showarticletitle{Multiple contraction. A further case against
  {G}{\"a}rdenfors' principle of recovery}.
\newblock In {\em The Logic of Theory Change}. Lecture Notes in Computer
  Science, Vol. 465. 322--334.
\newblock


\bibitem[\protect\citeauthoryear{Osorio and Cuevas}{Osorio and Cuevas}{2007}]%
        {osorio2007updates}
{Mauricio Osorio} {and} {V{\'i}ctor Cuevas}. 2007.
\newblock \showarticletitle{Updates in answer set programming: An approach
  based on basic structural properties}.
\newblock {\em Theory and Practice of Logic Programming\/} {7}, 4 (2007),
  451--479.
\newblock


\bibitem[\protect\citeauthoryear{Parikh}{Parikh}{1999}]%
        {parikh1999beliefs}
{Rohit Parikh}. 1999.
\newblock \showarticletitle{Beliefs, Belief Revision, and Splitting Languages}.
\newblock {\em Logic, Language and Computation\/}  {2} (1999), 266--278.
\newblock


\bibitem[\protect\citeauthoryear{Rott}{Rott}{1992}]%
        {rott1992modellings}
{Hans Rott}. 1992.
\newblock \showarticletitle{Modellings for belief change: Base contraction,
  multiple contraction, and epistemic entrenchment (Preliminary report)}.
\newblock In {\em Logics in Artificial Intelligence}. Lecture Notes in Computer
  Science, Vol. 633. 139--153.
\newblock


\bibitem[\protect\citeauthoryear{Satoh}{Satoh}{1988}]%
        {satoh1988nonmonotonic}
{Ken Satoh}. 1988.
\newblock \showarticletitle{Nonmonotonic reasoning by minimal belief revision}.
  In {\em Proceedings of the International Conference on Fifth Generation
  Computer Systems}. 455--462.
\newblock


\bibitem[\protect\citeauthoryear{Schaub and Wang}{Schaub and Wang}{2003}]%
        {schaub2003semantic}
{Torsten Schaub} {and} {Kewen Wang}. 2003.
\newblock \showarticletitle{A semantic framework for preference handling in
  answer set programming}.
\newblock {\em Theory and Practice of Logic Programming\/} {3}, 4--5 (2003),
  569--607.
\newblock


\bibitem[\protect\citeauthoryear{Schwind and Inoue}{Schwind and Inoue}{2013}]%
        {schwind2013characterization}
{Nicolas Schwind} {and} {Katsumi Inoue}. 2013.
\newblock \showarticletitle{Characterization theorems for revision of logic
  programs}.
\newblock In {\em Logic Programming and Nonmonotonic Reasoning}. Lecture Notes
  in Computer Science, Vol. 8148. 485--498.
\newblock


\bibitem[\protect\citeauthoryear{Slota}{Slota}{2012}]%
        {slota2012updates}
{Martin Slota}. 2012.
\newblock {\em Updates of hybrid knowledge bases}.
\newblock Ph.D. Dissertation. Universidade Nova de Lisboa.
\newblock


\bibitem[\protect\citeauthoryear{Slota and Leite}{Slota and Leite}{2012}]%
        {slota2012robust}
{Martin Slota} {and} {Jo{\~a}o Leite}. 2012.
\newblock \showarticletitle{Robust equivalence models for semantic updates of
  answer-set programs}. In {\em Principles of Knowledge Representation and
  Reasoning: Proceedings of the Thirteenth International Conference, KR 2012}.
  158--168.
\newblock


\bibitem[\protect\citeauthoryear{Slota and Leite}{Slota and Leite}{2013}]%
        {slota2013rise}
{Martin Slota} {and} {Jo{\~a}o Leite}. 2013.
\newblock \showarticletitle{The rise and fall of semantic rule updates based on
  {SE}-models}.
\newblock {\em Theory and Practice of Logic Programming\/}  {FirstView} (2013),
  1--39.
\newblock
\showISSN{1475-3081}


\bibitem[\protect\citeauthoryear{Turner}{Turner}{2003}]%
        {turner2003strong}
{Hudson Turner}. 2003.
\newblock \showarticletitle{Strong equivalence made easy: Nested expressions
  and weight constraints}.
\newblock {\em Theory and Practice of Logic Programming\/} {3}, 4 (2003),
  609--622.
\newblock


\bibitem[\protect\citeauthoryear{Wassermann}{Wassermann}{2000}]%
        {wassermann2000resource}
{Renata Wassermann}. 2000.
\newblock {\em Resource-bounded belief revision}.
\newblock Ph.D. Dissertation. Universiteit van Amsterdam.
\newblock


\bibitem[\protect\citeauthoryear{Wassermann}{Wassermann}{2011}]%
        {wassermann2011agm}
{Renata Wassermann}. 2011.
\newblock \showarticletitle{On {AGM} for non-classical logics}.
\newblock {\em Journal of Philosophical Logic\/} {40}, 2 (2011), 271--294.
\newblock


\bibitem[\protect\citeauthoryear{Williams}{Williams}{1994}]%
        {williams1994logic}
{Mary-Anne Williams}. 1994.
\newblock \showarticletitle{On the logic of theory base change}.
\newblock In {\em Logics in Artificial Intelligence}. Lecture Notes in Computer
  Science, Vol. 838. 86--105.
\newblock


\bibitem[\protect\citeauthoryear{Wong}{Wong}{2008}]%
        {wong2008sound}
{Ka-Shu Wong}. 2008.
\newblock \showarticletitle{Sound and complete inference rules for
  {SE}-consequence}.
\newblock {\em Journal of Artificial Intelligence Research\/}  {31} (2008),
  205--216.
\newblock


\bibitem[\protect\citeauthoryear{Zhuang, Delgrande, Nayak, and Sattar}{Zhuang
  et~al\mbox{.}}{2016}]%
        {zhuang2016reconsidering}
{Zhiqiang Zhuang}, {James~P. Delgrande}, {Abhaya~C. Nayak}, {and} {Abdul
  Sattar}. 2016.
\newblock \showarticletitle{Reconsidering {AGM}-style belief revision in the
  context of logic programs}. In {\em Proceedings of the 22nd European
  Conference on Artificial Intelligence, ECAI 2016}. 671--679.
\newblock


\end{thebibliography}
                             
\received{Month 2000}{Month 2000}{Month 2000}
\end{document}